\documentclass[10pt,journal,compsoc]{IEEEtran}

\usepackage{multirow}
\usepackage{makecell}
\usepackage{amsfonts}
\usepackage{nicefrac}
\usepackage{booktabs}
\usepackage{xcolor}
\usepackage{url}            
\usepackage{microtype}      
\usepackage{xcolor}         
\usepackage{color}         
\usepackage{amsmath,amssymb,amsthm,mathrsfs,lineno}
\usepackage{subfigure,graphicx}
\usepackage{algorithmic}
\usepackage{algorithm}
\usepackage{tabularx}
\usepackage{wrapfig}
\usepackage{ragged2e}
\usepackage{dsfont}
\usepackage{bbm}
\usepackage{bm}
\usepackage{makecell}

\newtheorem{theorem}{Theorem}
\newtheorem{definition}{Definition}
\usepackage{cite}
\ifCLASSINFOpdf

\else

\fi

\begin{document}

\title{Theory-inspired Label Shift Adaptation via \\ Aligned Distribution Mixture}

\author{Ruidong Fan, Xiao Ouyang, Hong Tao, Yuhua Qian~\IEEEmembership{Member,~IEEE}, Chenping~Hou,~\IEEEmembership{Member,~IEEE}
\thanks{This work was partially supported by the NSF for Distinguished Young Scholars under Grant No. 62425607, the Key NSF of China under Grant No. 62136005, the NSF of China under Grant No. 61922087, 62476282 and 62006238, the NSF of Hunan Province under Grant 2023JJ20052, and the Research Innovation Project for Postgraduate Students under Grant No. CX20230012. (Corresponding author: Chenping Hou)}
\thanks{Ruidong Fan, Xiao Ouyang, Hong Tao and Chenping Hou are with the College of Science, National University of Defense Technology, Changsha, 410073, China. E-mail: fanruidong1996@hotmail.com, ouyangxiao98@hotmail.com, taohong.nudt@hotmail.com, hcpnudt@hotmail.com. Yuhua Qian is with the Institute of Big Data Science and Industry, Shanxi University, Taiyuan, 030006, China. E-mail: jinchengqyh@126.com}}

\markboth{Theory-inspired Label Shift Adaptation via  Aligned Distribution Mixture}%
{Fan \MakeLowercase{\textit{et al.}}: Theory-inspired Label Shift Adaptation via Aligned Distribution Mixture}

\IEEEtitleabstractindextext{
\begin{abstract}
\justifying{As a prominent challenge in addressing real-world issues within a dynamic environment, label shift, which refers to the learning setting where the source~(training) and target~(testing) label distributions do not match, has recently received increasing attention. Existing label shift methods solely use unlabeled target samples to estimate the target label distribution, and do not involve them during the classifier training, resulting in suboptimal utilization of available information. One common solution is to directly blend the source and target distributions during the training of the target classifier. However, we illustrate the theoretical deviation and limitations of the direct distribution mixture in the label shift setting. To tackle this crucial yet unexplored issue, we introduce the concept of aligned distribution mixture, showcasing its theoretical optimality and generalization error bounds. By incorporating insights from generalization theory, we propose an innovative label shift framework named as \textbf{A}ligned \textbf{D}istribution \textbf{M}ixture~(ADM). Within this framework, we enhance four typical label shift methods by introducing modifications to the classifier training process. Furthermore, we also propose a one-step approach that incorporates a pioneering coupling weight estimation strategy. Considering the distinctiveness of the proposed one-step approach, we develop an efficient bi-level optimization strategy. Experimental results demonstrate the effectiveness of our approaches, together with their effectiveness in COVID-19 diagnosis applications. }
\end{abstract}
}

\maketitle

\IEEEdisplaynontitleabstractindextext

\section{Introduction} \label{sec1}
The success of deep learning heavily relies on the classical assumption of supervised learning, which posits that the source~(training) and target~(testing) domain share the same distribution~\cite{murphy2012machine,bengio2021deep}. However, in practical applications, this assumption is often violated due to variations in data collection and labeling processes, leading to a notable decline in the testing performance of the trained models. This phenomenon is known as distribution shift, which poses a significant challenge in the deployment of deep learning models in real-world scenarios~\cite{quinonero2022dataset,sugiyama2012machine,zhou2022open,long2018transferable}. Label shift~\cite{Azizzadenesheli2022Importance,saerens2002adjusting,du2014semi,zhang2013domain,bai2022adapting,maity2022}, as a key branch of distribution shift, involves a change in the label distribution between the training and test sets~($P_t(Y)\ne P_s(Y)$), while maintaining unchanged conditional distributions~($P_t(X|Y) = P_s(X|Y)$). For instance, in the context of vehicle recognit1ion, the prevalent car models may differ across various cities, while the manufacturers of the same car models remain consistent, resulting in a uniform appearance. As another illustration, in the domain of bird identification, the migratory behavior of certain bird species yield fluctuations in their distribution within the same geographical area between spring and winter; however, their physical characteristics remain consistent throughout the seasons.
\begin{figure}[!t] \label{fig1}
	\centering
	\includegraphics[width=0.5\textwidth]{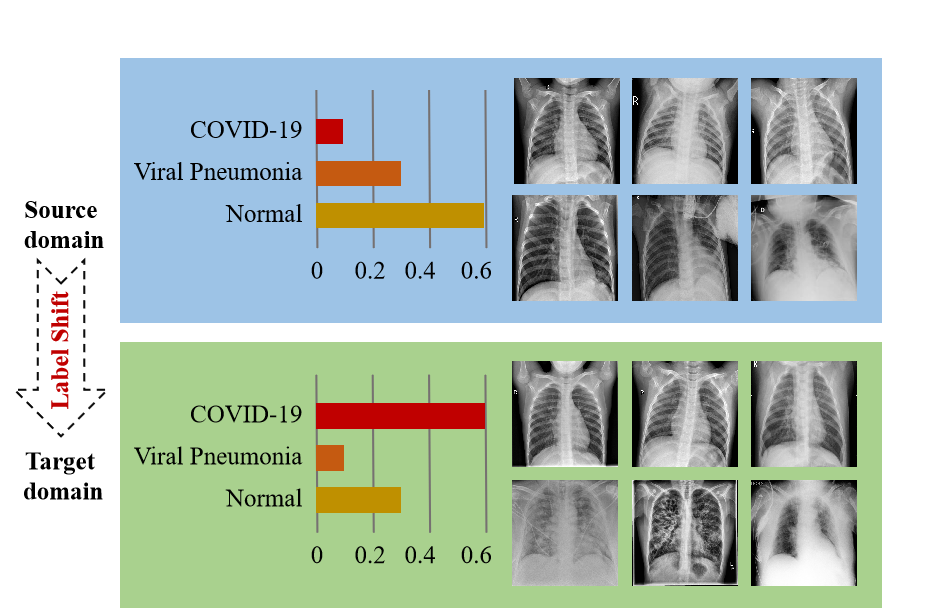}
	\caption{An applied illustration of the label shift scenario can be observed in the domain of COVID-19 diagnosis. Prior to and during the pandemic, the prevalent forms of pneumonia are viral pneumonia and COVID-19, respectively. Despite this shift, the computed tomography~(CT) manifestations of homopneumonia continue to exhibit consistency across the epidemiological periods.}
	\vskip -0.15in
\end{figure}

Motivated by the practical demand outlined above, it is crucial to adapt traditional approaches to augment the generalization capability of models and effectively tackle label shift scenarios~\cite{garg2023rlsbench,li2019target,JMLR/Dirk,wu2021online}. Most advanced label shift approaches~\cite{lipton2018detecting,iclr/Azizzadenesheli19,wacv/SipkaSM22,alexandari2020maximum,nips/GargWBL20} typically involve two primary stages: (1) estimating the target distribution; and (2) constructing an unbiased estimation of the target risk. In the initial stage, certain studies~\cite{lipton2018detecting,iclr/Azizzadenesheli19,wacv/SipkaSM22} infer the target distribution by leveraging the confusion matrix and predicted target labels derived from a pre-trained source classifier. Meanwhile, other works~\cite{alexandari2020maximum,nips/GargWBL20} estimate the target distribution through the utilization of Kullback-Leibler~(KL) divergence to minimize the dissimilarity between the weighted feature and target feature distributions. Subsequently, in the second stage, a substantial portion of existing research~\cite{JMLR/Afonso,uai/PodkopaevR21} employs labeled source samples and the importance-weighted empirical risk minimization~(ERM) framework to train an unbiased target classifier. Moreover, the label shift scenario postulates the concurrent utilization of both labeled and unlabeled data, aligning with the principles of semi-supervised learning~(SSL). SSL~\cite{yang2022survey,mey2022improved,EngelenH20,ChenMZA24}, where the training data is drawn from a direct amalgamation of labeled source and unlabeled target distributions, has exhibited success across various domains, such as image classification~\cite{sohn2020fixmatch}, medical image analysis~\cite{CheplyginaBP19}, and other applications~\cite{jeong2019consistency,HuHLXSST23,JiaoZDXZCJ24}. Nevertheless, when labeled and unlabeled training data stem from inconsistent distributions, SSL may encounter significant performance deterioration~\cite{abs220300190,Fang00023}. Current SSL approaches under inconsistent distributions commonly assume a covariate shift between labeled and unlabeled training data~\cite{ryan2015semi,aminian2022information}, which diverges from the label shift setting. 

Despite the advancements achieved by these approaches, there remain at least two challenges that necessitate attention in addressing the label shift predicament~(as depicted in Fig. 1). Firstly, traditional label shift approaches tend to rely exclusively on labeled source samples for training a re-weighted classifier, overlooking the potential advantages of integrating unlabeled target samples utilized in estimating the target distribution. This limited utilization of available information may result in suboptimal model performance. Secondly, a viable solution to above challenge is found in SSL, which entails the direct integration of source and target distributions during the classifier training. However, traditional SSL methods have never explicitly focused on the label shift scenario. Therefore, the rationale and efficacy of attempting to employ existing SSL techniques to enhance performance in label shift setting remain uncertain. 

\begin{figure*}[!t] 
	\label{fig3}
	\centering
	\includegraphics[width=1\textwidth]{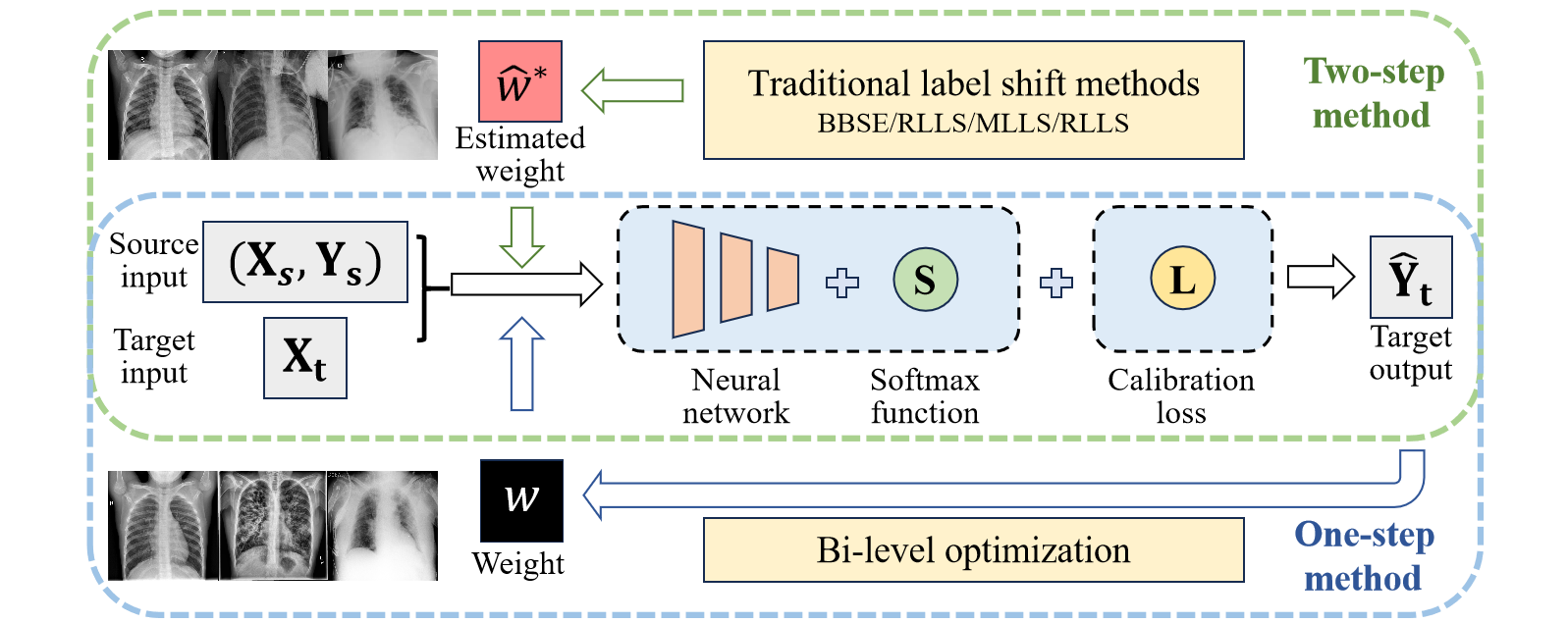}
	\caption{Visualization of the ADM framework. ADM framework, comprising a neural network, softmax function and calibration loss, encompasses four two-step approaches based on traditional weight estimation, alongside a one-step approach based on novel coupling weight estimation. The two-step approaches employs the traditional label shift methods to estimate the alignment weight, subsequently integrating it into the proposed framework. In contrast, the one-step approach creatively unifies the weight estimation and classifier training into a cohesive loss function, employing a bi-level optimization strategy for resolution.}
\end{figure*}
In order to address the aforementioned challenges with this novel yet underexplored learning scenario, we first verify the theoretical deviation of the direct distribution mixture in label shift scenario, revealing its unsuitability. Building upon this analysis, we introduce the concept of aligned distribution mixture~(ADM), where a weighted parameter is employed to align the source and target distributions, alongside a tradeoff parameter to blend the aligned distributions. We theoretically establish the optimality and upper bound of generalization error for ADM framework. By ignoring the Rademacher complexity inherent in generalization theory and ensuring classifier calibration, we propose a theory-inspired label shift framework incorporating a calibration loss, featuring both two-step and one-step approaches~(as depicted in Fig. 2). The two-step approaches leverage the traditional label shift techniques to estimate the alignment weight, which is then incorporated into the proposed framework. Moreover, the one-step approach creatively integrates weight estimation and classifier training into a unified loss function, employing a bi-level optimization strategy for resolution. Lastly, extensive experiment results are presented to validate the efficacy of ADM framework. In summary, the contributions of our study are delineated as follows:
\begin{itemize}
	\item We introduce the ADM framework to tackle this pivotal yet underexplored issue prevalent across diverse application domains. This study represents an innovative endeavor in formulating a label shift framework that adeptly incorporates unlabeled target samples into the training of re-weighted classifiers.
	\item We introduce the concept of ADM and validate its theoretical optimality in the label shift scenario. Building upon this foundation, we investigate the upper bound of generalization error and formulate the experience loss framework by disregarding the Rademacher complexity.
	\item With the ADM framework, we introduce four two-step approaches rooted in conventional weight estimation, alongside a one-step approach founded on innovative coupling weight estimation and bi-level optimization strategy. To the best of our knowledge, this marks the inaugural endeavor towards resolving the label shift quandary in a single step.
	\item We demonstrate the effectiveness of our approaches across a diverse array of datasets. The experimental results consistently indicate the superior performance of our approaches in comparison to alternative methods across the majority of cases. Furthermore, the application of our framework in the realm of COVID-19 diagnosis underscores its potential versatility and utility in practical settings.
\end{itemize}

The subsequent sections of the manuscript are organized as follows: Section 2 presents an overview of related works. The delineation of the problem setting is expounded in Section 3. Theoretical findings are demonstrated in Section 4, followed by the exposition of the ADM framework encompassing both two-step and one-step approaches in Section 5. Empirical results on various benchmark datasets are showcased in Section 6. Ultimately, the conclusions drawn from this study are summarized in Section 7.

\section{Relate Work}
\subsection{Label Shift} 
Label shift represents a notable form of distribution shift that has garnered considerable attention in recent years~\cite{KirchmeyerRBG22,Shui0LGL021,LiMMDC19,YuLH23,YeTPB24,FanOLHH23,ParkYCY23,SunMED23}. It occurs when the source and target domains exhibit different label distributions while maintaining identical conditional distributions~\cite{RakotomamonjyFG22,HwangLKOK22}. Existing label shift methods primarily focus on precise target distribution estimation. Previous studies~\cite{lipton2018detecting,iclr/Azizzadenesheli19} leverage arbitrary black box predictors to estimate the target distribution through the utilization of confusion matrices and predicted target labels, thereby offering robust statistical assurances. Certain works~\cite{wacv/SipkaSM22,alexandari2020maximum,nips/GargWBL20} theoretically and empirically affirm the necessity of classifier calibration for handling label shift, and employ KL divergence to minimize the distribution disparities in target distribution estimation. Moreover, label shift pervades a multitude of learning tasks, including online learning~\cite{bai2022adapting,wu2021online,QianBZZZ23,BabyGYBL023}, federated learning~\cite{XuH23,PlassierMRMP23}, active learning~\cite{0005KXSP0L023,HwangLKOK22,ZhaoLAY21}, multi-task learning~\cite{Shui0LGL021,FanOLHH23,KimTSFGCH22}, and more. Nevertheless, existing label shift methods solely rely on unlabeled target samples for estimating the target distribution, neglecting its direct utilization in training the target classifier, thereby leading to suboptimal exploitation of sample information.
\subsection{Semi-supervised Learning} 
Semi-supervised learning~(SSL), which serves as a typical learning scenario, combines labeled and unlabeled data during the training phase. Two fundamental approaches in SSL are entropy regularization and pseudo-label learning~\cite{yang2022survey,mey2022improved}. In entropy minimization, the main empirical risk loss is augmented with an entropy regularization of the predicted conditional distribution from unlabeled samples, penalizing uncertainties in label predictions~\cite{GrandvaletB04}. The efficacy of the entropy minimization algorithm relies on certain assumptions, such as the manifold assumption~\cite{iscen2019label}, positing that both labeled and unlabeled data are drawn from a shared manifold, or the cluster assumptions~\cite{chapelle2002cluster}, which presume that samples with similar features exhibit similar labels. In contrast, pseudo-label learning involves training the model in a supervised manner using labeled samples, and subsequently assigning pseudo-labels to unlabeled samples with high confidence~\cite{ZouYLKW19,MinBL24}. However, when labeled and unlabeled training data originate from inconsistent distributions, SSL encounters significant performance deterioration. Existing SSL methods with inconsistent distributions mainly focus on covariate shift, assuming a discrepancy in the feature distribution between labeled and unlabeled samples~\cite{GuSX24,aminian2022information,JiaG0SXL23}. However, semi-supervised learning under label shift remains an unexplored territory of vital importance, warranting further investigation and resolution.
\section{Problem Setting}
In this section, we commence by presenting key definitions outlined in the paper. Subsequently, we provide a succinct overview of prevalent label shift methods and the traditional semi-supervised learning framework grounded in direct distribution mixture.
\subsection{Notations} 
Let $\mathcal{X} \subseteq \mathbb{R}^{d}$ denote the feature space, and $\mathcal{Y} \subseteq [K]$ represent the discrete label space, where $d$ is the dimensionality of the features, and $K$ represents the total number of classes. We define the source and target distributions, denoted as $P_s$ and $P_t$ respectively. Building upon this foundation, we define the weighted source distribution $P_s^w$ characterized by important weight $w\in \mathbb{R}^{K}$ as follows: $P_s^w(Y) = w(Y)P_s(Y)$ and $P_s^w(X|Y) = P_s(X|Y)$. We have access to labeled data $(X_s,Y_s) =\{x_i,y_i\}_{i=1}^{n} $, which is independently and identically drawn from the source distribution $P_s$. Furthermore, we have unlabeled data $X_t =\{x_i\}_{i=n+1}^{n+m}$, which is independently and identically drawn from the target distribution $P_t$. We assume access to an arbitrary classifier $h \in \mathcal{H}$: $\mathcal{X} \mapsto {\Delta _{K - 1}}$, where $\Delta _{K - 1}$ denotes the standard $K$-dimensional probability simplex. We give a loss function $l(\cdot,\cdot) : \mathcal{Y} \times \mathcal{Y} \to \mathbb{R}^{1}$ and define the expected loss, the weighted expected loss, the empirical loss and the weighted empirical loss respectively as follows:
\begin{equation}\label{eq1}
	\begin{aligned}
	&\mathbb{E}(h,{P_s}) = \int_{\cal X} {\int_{\cal Y} {{P_s}(x,y)l(h(x),y){\rm{d}}x{\rm{d}}y} } ,
	\\&\mathbb{E}(h,P_s^w) = \int_{\cal X} {\int_{\cal Y} {w(y){P_s}(x,y)l(h(x),y){\rm{d}}x{\rm{d}}y} },
	\\&	\mathbb{\hat E}(h,{P_s}) = \frac{1}{n}\sum\limits_{i = 1}^n {l(h({x_i}),{y_i})},
	\\& \mathbb{\hat E}(h,P_s^w) = \frac{1}{n}\sum\limits_{i = 1}^n {w({y_i})l(h({x_i}),{y_i})} .
	\end{aligned}
\end{equation} 
We denote the direct mixture of two distributions $P_s$ and $P_t$ with proportion $\beta$ as ${P}_m = {{\rm{MIX}}_{\beta}}({P}_s,{P}_t)$. For enhanced precision, we redefine the direct distribution mixture~(DDM) by discerning the label and conditional distributions as separate entities:
\begin{equation}
	\begin{aligned}
	&\qquad{P}_m(Y) = \beta{P}_s(Y) + (1-\beta){P}_t(Y),\\&
	{P}_m(X|Y) = \beta{P}_s(X|Y) + (1-\beta){P}_t(X|Y).
	\end{aligned}
\end{equation} 
Furthermore, for the ease of theoretical expression, we present the definitions of Rademacher Complexity, Total Variation (TV) distance, canonical calibration and bounded loss function.
\begin{definition} {\bf{(Rademacher Complexity)}}
Define the function set $\mathcal{G}(l,\mathcal{H}) = \{l(h(x),y),\,h\in\mathcal{H}\}$, where $\{x_i,y_i\}_{i=1}^{n}$ represents a set containing $n$ source samples. Subsequently, the Rademacher complexity of $\mathcal{G}$ in relation to the sample set is explicated as:  
\begin{equation}
	\begin{aligned}
	{\rm{Rad}}_{n}(\mathcal{G}) := \mathbb{E}_{\bm{\sigma}}\left[\mathop {\sup }\limits_{h\in\mathcal{H}}\frac{1}{n}\sum\limits_{i = 1}^n {{\sigma _i}l(h({x_i}),{y_i})} \right],
	\end{aligned}
\end{equation}
where $\bm{\sigma} = \{\sigma_i\}_{i=1}^{n}$ and $\sigma_i$ are i.i.d. uniform random variables taking values in $\{+1,-1\}$.
\end{definition}

\begin{definition}{\bf{(TV Distance)}}
	The TV distance $d_{\rm{TV}}(\cdot,\cdot)$ between two probability measures $P_s(Y)$ and $P_t(Y)$, is defined as
\begin{equation}	
	d_{\rm{TV}}(P_s(Y),P_t(Y)) =  \frac{1}{2}\int_\mathcal{Y} \left|{P_s}(Y=y) - {P_t}(Y=y)\right|{\rm{d}}y.
\end{equation}
\end{definition} 

\begin{definition}{\bf{(Canonical Calibration)}}
A classifier $h \in \mathcal{H}$ is canonically calibrated on the source domain if for any $j\in[K]$ and $(x,y)\in (X_s,Y_s)$,
	\begin{equation}	
		P_s(y=j|h(x)) = h_j(x),
	\end{equation}
where $h_j(x)$ represents the $j$-th element of the output $h(x)$.
\end{definition}

\begin{definition}{\bf{(Bounded Loss Function)}}
	A function $l(\cdot,\cdot) : \mathcal{Y} \times \mathcal{Y} \to \mathbb{R}^{1}$ is called bounded loss if there exists a constant $C$ such that for any $(y_0,y_1) \in  \mathcal{Y} \times \mathcal{Y}$,
\begin{equation}
	\left| {l(y_0,y_1)} \right| \le C.
\end{equation}
\end{definition}

\subsection{Label Shift Technologies} 

\textbf{Reweighted Strategy.} Under the label shift assumption, i.e., the same conditional distribution $P_s(X|Y) = P_t(X|Y)$ and the different label distribution $P_s(Y) \ne P_t(Y)$, the goal of label shift setting is to find an optimal target classifier $h_t$ which minimizes the following reweighted strategy:
\begin{equation}\label{eq2}
	\begin{aligned}
		\mathcal{L}(h_t) &= \mathbb{E}(h_t,P_t)  = \mathbb{E}_{(X,Y) \sim {P_t}}[l(h_t(X),Y)] \\& = \mathbb{E}_{(X,Y) \sim {P_s}}\left[w(Y)l(h_t(X),Y)\right] = \mathbb{E}(h_t,P_s^{w}),
	\end{aligned}
\end{equation} 
where $w$ represents the importance weight that measures the difference in label distributions between the source and target domains, and $w(Y) = {{{{P}_t}{(Y)}} \mathord{\left/{\vphantom {{{{P}_t}{(Y)}} {{{P}_s}{\rm{(Y)}}}}} \right.\kern-\nulldelimiterspace} {{{P}_s}{(Y)}}}$. 

In practical applications, the true weight parameter $w$ is frequently unknown. In response, traditional label shift methods strive to approximate the true weight by obtaining the estimated weight $\hat w$ based on finite samples. Subsequently, we provide a concise overview of the existing techniques for importance weight estimation.

\noindent\textbf{Importance Weight Estimation.} Given  $C_h$ as the confusion matrix with $[C_h]_{ij} = P_s(h(X_s)=i,Y_s=j)$ and $q_h$ as the vector which represents the probability mass function of $h(X_t)$, the estimated $\hat{C}_h$ and $\hat{q}_h$ in the typical finite sample setting are derived from the labeled source and unlabeled target samples. BBSE~\cite{lipton2018detecting} computes the importance weight by $\hat w = {\hat{C}_h}^{-1} \hat{q}_h$. RLLS~\cite{iclr/Azizzadenesheli19} evaluates the weight shift degree $\hat{\theta} = \mathop {\arg\min_\theta} || \hat{C}_h{\theta} - (\hat{q}_h- \hat{C}_h\mathbf{1}) ||_2 + {\Delta _C}\left\|\theta\right\|_2$, and deduces the importance weight by $\hat{w} = \gamma \hat{\theta} + \mathbf{1}$, where both  $\Delta _C$ and $\gamma$ serve as tradeoff parameters.
Moreover, assuming $h_s$ as a well-calibrated source classifier, MLLS~\cite{nips/GargWBL20} estimates the importance weight utilizing the equation: $\hat w = \arg {\max _w}\frac{1}{m}\sum\nolimits_{i = n + 1}^{n + m} {\log {h_s}{{({x_i})}^{\rm T}}(w\odot P_s(Y))} $. Furthermore, to mitigate the impact of negative weight, SCML~\cite{wacv/SipkaSM22} determines the importance weight by maximizing the log-likelihood: $\hat w = \arg {\max _{w \in \Delta _{K - 1}}}\sum\nolimits_{k = 1}^K {{m_k}\log {w(k)[C_{h_s}]_{k:}}}$, where $m_k$ denotes the numbers of decisions made by the source classifier for $k$-th class on the unlabeled data $X_t$. 

\subsection{SSL Framework via Direct Distribution Mixture} 
The current label shift methods solely rely on unlabeled target data for importance weight estimation, which is insufficient to fully leverage the unlabeled information. Utilizing the unlabeled data for classifier training has the potential to enhance model performance. SSL via DDM emerges as the requisite learning framework for this purpose, enabling the joint utilization of labeled and unlabeled data to acquire a hypothesis that facilitates predictions for new unlabeled instances. Let us introduce the definition of the unsupervised expected loss function:
\begin{equation}\label{eq3}
	\begin{split} 
		\mathbb{E}(h,{P_t}) &\buildrel \Delta \over = \int_{\mathcal X} {\int_{\mathcal Y} {{P_t}(x,y)l(h(x),y)\text{d}x\text{d}y} } 
		\\&=  \int_{\mathcal X} {{P_t}(x)\underbrace {\left( {\int_{\mathcal Y} {{P_t}(y|x)l(h(x),y)\text{d}y} } \right)}_{{l_u}(h(x);{P_t})}} \text{d}x
		\\& =  \int_{\mathcal X} {{P_t}(x){l_u}(h(x);{P_t})}\text{d}x,
	\end{split}
\end{equation}
where ${l_u}(h(x);{P_t})$ denotes the unsupervised loss function that quantifies the loss linked to unlabeled target data. On this basis, the SSL framework via DDM can be defined as:
\begin{equation}\label{eq4}
	\begin{split} 
		\mathbb{E}(h,{\rm{MIX}}({P_s},{P_t})) \buildrel \Delta \over = \beta\mathbb{E}(h,{P_s}) + (1-\beta)\mathbb{E}(h,{P_t}),
	\end{split}
\end{equation}
and the corresponding total loss function is as follows:
\begin{equation}\label{eq5}
	\begin{split} 
		\mathcal{L}_{SSL} = \frac{\beta}{n}\sum\limits_{i = 1}^n {l(h({x_i}),{y_i})} + \frac{1-\beta}{m}\sum\limits_{j = n+1}^{n+m} {l_u(h({x_j});{P_t})},
	\end{split}
\end{equation}
where $\beta$ balances between the supervised and unsupervised risk, and $0 \le \beta \le 1$. If we choose $\beta=0$, the above problem reduces to an unsupervised learning setting and if we choose $\beta=1$, the above problem reduces to a supervised learning scenario.

\section{Theoretical Results} \label{sec2}

The success of the SSL framework relies on the assumption that all samples are drawn from a common distribution~\cite{WangJWWM23,tsmc/HuangD22}. This leads us to question the feasibility of achieving optimal performance by straightforwardly employing the traditional SSL framework in the label shift scenario. Therefore, in this section, this section delves into an examination of the generalization capabilities of the traditional SSL framework in label shift scenario and introduces a novel theoretical framework to tackle this challenge.

\subsection{Impossibility of SSL Framework via DDM to Correct Label Shift}
We contemplate the capacity of the SSL framework via DDM to acquire the optimal target classifier $h_t$, aimed at minimizing the expected target loss. To delve into this inquiry, we introduce the ensuing theorem:
\begin{theorem}\label{th1} {\bf{(The Irrationality of Direct Distribution Mixture)}}
    Assume that ${P}_m = {{\rm{MIX}}_{\beta}}({P}_s,{P}_t)$ and the sets of label conditional distribution $\{P_t(X|Y=y), y\in[K]\}$ are linearly independent, we can prove that for any hypothesis $h\in \mathcal{H}$, $\mathbb{E}(h,{P}_m)+ \mathbb{E}(h,{P}_t)>0$.
\end{theorem}
\begin{proof}
In the initial phase, let us revisit the formulation of the expected loss function:
\begin{equation}\label{seq1}
\begin{split} 
\mathbb{E}(h,{P}_m) 
&= \mathbb{E}_{P_m}[l(h(X),Y)]
\\& = \int_\mathcal{X} {\int_\mathcal{Y} {{P_m}(x,y)l(h(x),y)\text{d}x\text{d}y}} 
\\& = \int_\mathcal{X} P_m(x)\left(\int_\mathcal{Y} {{P_m}(y|x)l(h(x),y)\text{d}y}\right)\text{d}x,
\end{split}
\end{equation}
Note that the condition $\mathbb{E}(h,{P}_m) = 0$ is achievable when all datasets following the distribution ${P}_m$ are separable completely, and the label $Y$ for a given $X$ is determinate. Correspondingly, a labeling rule $h_m$ exists such that $y=h_m(x)$ for every $(x,y)$ sampled from ${P}_m$, with the conditional distribution adhering to the Dirac distribution, i.e.,
\begin{equation}\label{seq2}
\begin{split} 
{P_m}{(y|x)}=\left\{ {\begin{array}{*{20}{c}}
	{{\delta},\quad \text{if}\;y = {h_m}(x),}\\
	{0,\;\;\;  \text{else}.\quad \quad \quad \;\;\;\;}
	\end{array}} \right.
\end{split}
\end{equation}
Similarly, the condition $\mathbb{E}(h,{P}_t)=0$ holds true only if there exists an $h_t$ that satisfies the aforementioned criteria. Note that $\mathbb{E}(h,{P}_m) + \mathbb{E}(h,{P}_t) = 0$ if and only if $\mathbb{E}(h,{P}_m) = \mathbb{E}(h,{P}_t) = 0$. This indicates that all distributions attain zero errors with a unified label rule $h^* = h_m = h_t$. Assuming  $\mathbb{E}(h_m,{P}_m)=0$, we can derive
\begin{equation}\label{seq3}
\begin{split} 
\mathbb{E}(h_m,{P}_t) 
&= \mathbb{E}_{P_t}[l(h_{m}(X),Y)]
\\& = \int_\mathcal{X} P_t(x)\left(\int_\mathcal{Y} {{P_t}(y|x)l(h_{m}(x),y)\text{d}y}\right)\text{d}x.
\end{split}
\end{equation}
By employing the Bayes formula, it becomes evident that
\begin{equation}\label{seq4}
\begin{split} 
{P_t}(Y |X) = \frac{{P_m}(X)}{{P_t}(X)}{{P}_m}{(Y|X)}.
\end{split}
\end{equation}
In addition, we have
\begin{equation}\label{seq5}
\begin{split} 
{P_m}(X)-{P_t}(X)& =  \int_\mathcal{Y} {{P_m}(x,y)\text{d}y}-\int_\mathcal{Y} {{P_t}(x,y)\text{d}y}
\\& \mathop=\limits^{\textcircled{\small{1}}} \int_\mathcal{Y} {{P_t}(x|y)({P_m}(y)-{P_t}(y))\text{d}y}
\\&= \int_\mathcal{Y} {\beta{P_t}( x|y)({P_s}(y)-{P_t}(y))\text{d}y}.
\end{split}
\end{equation}
Where the validity of equality $\textcircled{\small{1}}$ arises from the underlying conditional distribution presumption within the label shift scenario. Given that the label conditional distributions $\{P_t(X|Y=y), y\in[K]\}$ are linearly independent and the label distributions differ~(${P_s}(Y)\ne{P_t}(Y)$), we can establish
\begin{equation}\label{seq6}
\begin{split} 
{P_m}(X)-{P_t}(X) \ne 0.
\end{split}
\end{equation}
Combining with Eq. (\ref{seq4}), we have
\begin{equation}\label{seq7}
\begin{split} 
{P_t}(Y |X) \ne {P_m}(Y|X).
\end{split}
\end{equation}
Given that ${P_m}(Y|X)$ conforms to the Dirac distribution, there exists $y_0 \ne h_m(x)$ corresponding to $l(h_m(x),y_0)>0$, satisfying ${P_t}( y_0|x)\ne 0$. Revisiting Eq. (\ref{seq3}), we find 
\begin{equation}\label{seq8}
\begin{split} 
\mathbb{E}(h_m,{P}_t)
& = \int_\mathcal{X} P_t(x)\left(\int_\mathcal{Y} {{P_t}(y|x)l(h_{m}(x),y)\text{d}y}\right)\text{d}x
\\& \ge  \int_\mathcal{X} P_t(x){P_t}(y_0|x)l(h_{m}(x),y_0)\text{d}x
\\& >0.
\end{split}
\end{equation}
Thus, it proves that for any hypothesis $h_m\in\mathcal{H}$ satisfying $\mathbb{E}(h_m,{P}_m) = 0$, $\mathbb{E}(h_m,{P}_t)>0$ always holds, and vice versa. Finally, we have for any hypothesis $h\in\mathcal{H}$, 
\begin{equation}\label{seq10}
\begin{split} 
\mathbb{E}(h,{P}_m) + \mathbb{E}(h,{P}_t)>0.
\end{split}
\end{equation}
\end{proof}
Theorem \ref{th1} explicitly shows the limitations of the SSL framework via DDM, demonstrating the inability of any hypothesis to attain theoretical optimality even in scenarios where datasets are linearly separable. Subsequently, we establish upper bounds on the expected error disparity between $P_m$ and $P_t$, aiming to delve into the underlying causes for the ineffectiveness of the SSL framework via DDM.

\begin{theorem}\label{th2} {\bf{(Upper Bound on The Expected Error Disparity)}}
	Assume that the loss function $l$ is upper bounded by a constant $C$, then for any hypothesis $h\in\mathcal{H}$, we have
	\begin{equation}\label{seq11}
	\begin{split} 
	\left|\mathbb{E}(h,{P}_m) - \mathbb{E}(h,{P}_t)\right|\le{2\beta C}d_{\rm{TV}}(P_s(Y),P_t(Y)).
	\end{split}
	\end{equation}
\end{theorem}
\begin{proof}
	Firstly, drawing from the definition of $P_m$ and the label shift assumption $P_s(X|Y) = P_t(X|Y)$, we deduce
\begin{equation}
\begin{split} 
	{P}_m(X|Y) = \beta{P}_s(X|Y) + (1-\beta){P}_t(X|Y) = {P}_t(X|Y).
\end{split}
\end{equation}
Subsequently, we formulate the expected error disparity in integral form and analyze its upper bound as follows:
\begin{equation}\label{seq12}
\begin{split} 
&\left|\mathbb{E}(h,{P}_m) - \mathbb{E}(h,{P}_t)\right|
\\& \le \int_\mathcal{X} \int_\mathcal{Y} l(h(x),y)\left|{P_m}(x,y)- {P_t}(x, y)\right|\text{d}x\text{d}y
\\& \le C\int_\mathcal{X} \int_\mathcal{Y} \left|{P_m}(x|y){P_m}(y) - {P_t}(x|y){P_t}(y)\right|\text{d}x\text{d}y
\\& = C\int_\mathcal{Y} |{P_m}(y) - {P_t}(y)| \left(\int_\mathcal{X} {P_t}(x|y)\text{d}x\right)\text{d}y
\\& \mathop=\limits^{\textcircled{\small{2}}} C \int_\mathcal{Y} \left|\beta{P_s}(y) + \left(1-\beta-1\right){P_t}(y)\right|\text{d}y
\\& = {2\beta C}d_{\rm{TV}}(P_s(Y),P_t(Y)),
\end{split}
\end{equation}
where the equality ${\textcircled{\small{2}}}$ is valid due to the fact that $\int_\mathcal{X} {P_t}(x|y)\text{d}x = 1$ holds for any $y\in \mathcal{Y}$. 
\end{proof}
Theorem \ref{th2} reveals that the inefficacy of SSL framework via DDM in label shift scenarios stems from the discrepancy in label distributions. This observation motivates us to align the label distributions of $P_m$ and $P_t$, while upholding the equivalence of their conditional distributions. Therefore, in the ensuing subsection, we provide theoretical validation and technical elucidation for the alignment of label distributions.
\subsection{ADM Framework to Correct Label Shift}
We propose a novel distribution mixture technique to align label distributions. This method involves assigning weights to the source distribution and then amalgamating it with the target distribution. The efficacy of this approach is supported by theoretical guarantees, which are expounded upon as follows:
\begin{theorem}\label{th3} {\bf{(Sufficiency of ADM Correction)}}
	Assume that ${P}^w_m = {{\rm{MIX}}_{\beta}}({P}^w_s,{P}_t)$, where ${P}^w_s(Y) = w(Y){P}_s(Y)$ and ${P}^w_s(X|Y) = {P}_s(X|Y)$. Then for any $h\in\mathcal{H}$, there exists a unique $w^*(Y) = {{P}_t(Y) \mathord{\left/{\vphantom {{P}_t(Y) {P}_s(Y)}} \right. \kern-\nulldelimiterspace} {P}_s(Y)} $, such that
	\begin{equation}\label{seq13}
	\begin{split} 
	\mathbb{E}(h,{P}^{w^*}_m) = \mathbb{E}(h,{P}_t).
	\end{split}
	\end{equation}
\end{theorem}

\begin{proof}
	We provide distinct proofs to establish both the existence and uniqueness of the theorem.
	
	\noindent{\bf{[Existence]}}
	By the definition of $P^w_m$, we have
	\begin{equation}\label{seq14}
	\begin{split} 
	&{P}^{w^*}_m(X|Y) = \beta{P}^{w^*}_s(X|Y) + (1-\beta){P}_t(X|Y)=  {P}_t(X|Y),\\&
	{P}^{w^*}_m(Y) = \beta{P}^{w^*}_s(Y) + (1-\beta){P}_t(Y)=  {P}_t(Y).
	\end{split}
	\end{equation}
    Thus, the following formula holds:   
	\begin{equation}\label{seq15}
	\begin{split} 
	\mathbb{E}(h,{P}^{w^*}_m) 
	& = \int_\mathcal{X} \int_\mathcal{Y} l(h(x),y){P^{w^*}_m}(x,y)\text{d}x\text{d}y
	\\& = \int_\mathcal{X} \int_\mathcal{Y} l(h(x),y){P^{w^*}_m}(x|y){P^{w^*}_m}(y)\text{d}x\text{d}y
	\\& = \int_\mathcal{X} \int_\mathcal{Y} l(h(x),y){P_t}(x|y){P_t}(y)\text{d}x\text{d}y
	\\& = \mathbb{E}(h,{P}_t).
	\end{split}
	\end{equation}
	{\bf{[Uniqueness]}} Suppose there exists $w_2 \ne w^*$ for which $\mathbb{E}(h,{P}^{w_2}_m) = \mathbb{E}(h,{P}^{w^*}_m) = \mathbb{E}(h,{P}_t)$, then we derive
	\begin{equation}\label{seq16}
	\begin{split} 
	&\mathbb{E}(h,{P}^{w^*}_m) - \mathbb{E}(h,{P}^{w_2}_m)
	\\& = \int_\mathcal{X} \int_\mathcal{Y} l(h(x),y)\left({P^{w^*}_m}(x,y) - {P^{w_2}_m}(x,y)\right)\text{d}x\text{d}y
	\\& = \int_\mathcal{X} \int_\mathcal{Y} l(h(x),y){P_t}(x|y)\left({P^{w^*}_m}(y)- {P^{w_2}_m}(y)\right)\text{d}x\text{d}y
	\\& = \int_\mathcal{X} \int_\mathcal{Y} \beta l(h(x),y){P_t}(x|y)\left({P_t}(y)-{P^{w_2}_s}(y)\right)\text{d}x\text{d}y
	\\& = 0.
	\end{split}
    \end{equation}
    Since $\beta > 0$, for any $y\in [K]$, we have
    \begin{equation}\label{seq17}
    \begin{split}
    l(h(X),y){P_t}(X|y)\left({P_t}(y)-{P^{w_2}_s}(y)\right)=0.
    \end{split}
    \end{equation} 
    If $l(h(X),y)>0$ holds for every $y\in [K]$, it is apparent that ${P_t}(y) = {P^{w_2}_s}(y)$ and $w_2 = w^*$, thereby contradicting the initial assumption. Let us consider the scenario where there exists $y = K_0$ such that $l(h(X),K_0) = 0$ and $l(h(X),y) > 0$ for other cases, resulting in
    \begin{equation}\label{seq18}
    \begin{split}
    \sum\nolimits_{y\backslash \{ {K_0}\}} {P_t}(X|y)\left({P_t}(y)-{P^{w_2}_s}(y)\right) = 0.
    \end{split}
    \end{equation} 
     Since the label conditional distributions $\{P_t(X|Y=y), y\in[K]\}$ are linearly independent, ${P_t}(y) = {P^{w_2}_s}(y)$ holds, leading to $w_2 = w^*$ and contradicting the initial assumption. In conclusion, we can deduce that $w^*$ stands as the unique optimal solution.   
\end{proof}

Theorem \ref{th3} introduces a novel ADM framework, ensuring a zero upper bound on the expected error disparity. Subsequently, to validate the model's efficacy, we delve into the generalization performance of the ADM framework. Since the true $w$ is unknown, we introduce an estimated weight $\hat w$, followed by an analysis of the generalization error bound within the ADM framework.

\begin{theorem}\label{th4} {\bf{(Generalization Error Bound of ADM Framework)}}
Given $n$ samples drawn from the source distribution $P_s$, $m$ samples drawn from the target distribution $P_t$ and a bound loss function $l$ with a constant $C$, the following generalization bound holds with probability at least $1-2\delta$
\begin{equation}
	\begin{split}
		\mathbb{E}(h,{P}^{w^*}_m) &\le \beta\mathbb{\hat E}(h,P_s^{\hat w}) +(1 - \beta ) \mathbb{\hat E}(h,{P_t}) + \beta C\left\| w^*-\hat w \right\|_2
		\\& \qquad\,\, + \beta\left( 2 {\rm{Rad}}_{n}(\mathcal{G}_1)+ C\left\| w^*\right\|_{\infty}\sqrt {\frac{\log{\frac{2}{\delta}}}{n}} \right)
		\\& \qquad\,\, + (1-\beta)\left( 2 {\rm{Rad}}_{m}(\mathcal{G}) + C\sqrt {\frac{\log{\frac{2}{\delta}}}{m}} \right).     	
	\end{split}
\end{equation}
\end{theorem}

\begin{proof}
We partition the expectation concerning mixed distribution into the sum of expectations related to the weighted source and target distributions:
\begin{equation}\label{seq19}
\begin{split} 
\mathbb{E}(h,{P}^{w^*}_m)= \beta \mathbb{E}(h,P_s^{w^*}) + (1 - \beta )\mathbb{E}(h,{P_t}).
\end{split}
\end{equation}	
[1] For the expectation of the weighted source distribution, employing principles of addition and subtraction, we obtain
\begin{equation}\label{aeq1}
\begin{split} 
\mathbb{E}(h,P_s^{w^*}) &=\mathbb{\hat E}(h,P_s^{\hat w}) - \mathbb{\hat E}(h,P_s^{\hat w}) + \mathbb{\hat E}(h,P_s^{w^*}) 
\\& \qquad\qquad\quad\,\,\, -\mathbb{\hat E}(h,P_s^{w^*}) +  \mathbb{E}(h,P_s^{w^*})
\\& \le \mathbb{\hat E}(h,P_s^{\hat w}) + \left|\mathbb{\hat E}(h,P_s^{w^*})- \mathbb{\hat E}(h,P_s^{\hat w}) \right|
\\& \qquad\qquad\quad\,\,\,  + \left|\mathbb{E}(h,P_s^{w^*}) -\mathbb{\hat E}(h,P_s^{w^*}) \right|.
\end{split}
\end{equation}
Let us define $\tilde l \in \mathbb{R}^K$ and $\tilde l(j) =\sum\nolimits_{i=1}^{n} \mathds{1}_{y_i=j}l(h(x_i),y_i)$, and it is obvious that $||\tilde l||_1 \le n$ and $||\tilde l||_{\infty} \le n$. Then by Hoelder’s inequality, $||\tilde l||_{2} \le n$ always holds. For the empirical error disparity, we have
\begin{equation}\label{seq20}
\begin{split} 
&\left|\mathbb{\hat E}(h,P_s^{w^*})- \mathbb{\hat E}(h,P_s^{\hat w})\right| 
\\&\qquad\quad\quad = \left|\frac{1}{n}\sum\nolimits_{i=1}^{n} {(w^*(y_i)-{\hat w}(y_i))l(h(x_i),y_i)} \right|
\\&\qquad\quad\quad \le \left|\frac{1}{n}\sum\nolimits_{j=1}^{K}{(w^*(j)-{\hat w}(j)){\tilde l}(j)} \right|
\\&\qquad\quad\quad \le \frac{1}{n}\left\| w^*-{\hat w} \right\|_2 \left\|{\tilde l}(j)\right\|_2
\\&\qquad\quad\quad \le C\left\| w^*-{\hat w} \right\|_2.
\end{split}
\end{equation}	
In accordance with the illustration in Statistical Learning Theory and Chapter 4~\cite{wainwright2019high}, we have
\begin{equation}\label{aeq2}
\begin{split} 
&\left|\mathbb{E}(h,P_s^{w^*}) -\mathbb{\hat E}(h,P_s^{w^*}) \right|
\\&\qquad\qquad\qquad \le 2{\rm{Rad}}_{n}(\mathcal{G}_1) + C\left\| w^*\right\|_{\infty}\sqrt {\frac{\log{\frac{2}{\delta}}}{n}}
\end{split}
\end{equation}	
with probability at least $1-\delta$. Substituting Eq. (\ref{seq20}) and Eq. (\ref{aeq2}) into Eq. (\ref{aeq1}), the following equation holds:
\begin{equation}\label{seq21}
\begin{split} 
\mathbb{E}(h,P_s^{w^*}) \le&  \mathbb{\hat E}(h,P_s^{\hat w}) + C\left\| w^*-\hat w \right\|_2
\\& + 2{\rm{Rad}}_{n}(\mathcal{G}_1) + C\left\| w^*\right\|_{\infty}\sqrt {\frac{\log{\frac{2}{\delta}}}{n}}.
\end{split}
\end{equation}
[2] For the expectation of target distribution, according to the illustration in Statistical Learning Theory and Chapter 4~\cite{wainwright2019high}, we have 
\begin{equation}\label{aeq3}
	\begin{split} 
		&\mathbb{E}(h,P_t) \le \mathbb{\hat E}(h,P_t) + 2{\rm{Rad}}_{m}(\mathcal{G}) + C\sqrt {\frac{\log{\frac{2}{\delta}}}{m}}.
	\end{split}
\end{equation}	
with probability at least $1-\delta$. Substituting Eq. (\ref{seq21}) and Eq. (\ref{aeq3}) into Eq. (\ref{seq19}), the following generalization bound holds with probability at least $1-2\delta$
\begin{equation}
	\begin{split}
		\mathbb{E}(h,{P}^{w^*}_m) &  = \beta \mathbb{E}(h,P_s^{w^*}) + (1 - \beta )\mathbb{E}(h,{P_t})
		\\&\le \beta\mathbb{\hat E}(h,P_s^{\hat w}) +(1 - \beta ) \mathbb{\hat E}(h,{P_t}) + \beta C\left\| w^*-\hat w \right\|_2
		\\& \qquad\,\, + \beta\left( 2 {\rm{Rad}}_{n}(\mathcal{G}_1)+ C\left\| w^*\right\|_{\infty}\sqrt {\frac{\log{\frac{2}{\delta}}}{n}} \right)
		\\& \qquad\,\, + (1-\beta)\left( 2 {\rm{Rad}}_{m}(\mathcal{G}) + C\sqrt {\frac{\log{\frac{2}{\delta}}}{m}} \right).     	
	\end{split}
\end{equation}
\end{proof}

\section{The proposed experience framework} 
Our objective is to determine the optimal target classifier, which is equal to minimize the expectation of the aligned distribution mixture $\mathbb{E}(h,{P}^{w^*}_m)$. By leveraging the theoretical findings outlined earlier, we can establish the optimization goal without Rademacher complexity:
\begin{equation}\label{seq23}
\begin{split} 
\mathop {\min }\limits_{h \in \mathcal{H}} {\beta\mathbb{\hat E}(h,P_s^{\hat w}) +(1 - \beta ) \mathbb{\hat E}(h,{P_t}) + \beta C\left\| w^*-\hat w \right\|_2},
\end{split}
\end{equation}
where $0<\beta<1$ is the tradeoff parameter. As illustrated in Eq. (\ref{eq3}), the unsupervised empirical risk $ \mathbb{\hat E}(h,{P_t})$ is impacted by the unsupervised loss  ${l_u}(h(x);{P_t})$, which is determined by the conditional distribution $P_t(Y|X)$. However, under typical circumstances, the exact value of $P_t(Y|X)$ is unknown. Traditional SSL approaches approximate the true conditional distributions by considering the output of softmax layer in a deep neural network, i.e., $h(x) \approx  P_t(Y|X=x)$. 

Recent studies~\cite{icml/GuoPSW17,tmlr/BohdalYH23} have indicated that contemporary neural networks exhibit calibration issues, leading to mismatches between confidence and accuracy. SSL work~\cite{iclr/SchmutzHM23} shows that overly confident outputs, denoted by $h_i(X)>>{\max}_iP_t(Y_i|X)$, can impede the model's ability to assimilate new knowledge from unlabeled data, potentially resulting in overfitting. Furthermore, overly confident predictions can compromise the integrity of estimated weights with label shift methods. A prevalent technology involves the utilization of post-processing calibration techniques to rectify the model's output. However, these calibration methods typically entail a two-step process, rendering them unsuitable for one-pass semi-supervised learning paradigms. In our study, we employ the $\gamma$-loss for classifier calibration, where the supervised $\gamma$-loss function is articulated as follows:
\begin{equation}\label{seq31}
	\begin{split} 
		l^{\gamma}\left(h(x),y\right) = \frac{\gamma}{\gamma-1}\left({1-{h_y}(x)}^{1-\frac{1}{\gamma}}\right),\,\gamma \in (0,\infty),
	\end{split}
\end{equation}
where $h_y(x)$ represents the $y$-th element of the output $h(x)$. Accordingly, the formulation of the unsupervised loss function employing the $\gamma$-loss can be expressed as:
\begin{equation}\label{seq32}
	\begin{split} 
		l^{\gamma}_u(h(x);P_t) = \frac{\gamma}{\gamma-1}\sum\limits_{j = 1}^{K}{{h}_j(x)\left({1-{h}_j(x)}^{1-\frac{1}{\gamma}}\right)}.
	\end{split}
\end{equation}
We delineate several properties of the $\gamma$-loss and provide a theoretical justification for its efficacy in calibrating the classifier's output.
\begin{theorem}\label{th5} 
Under the definition of $\gamma$-loss, for $\gamma \in (0,\infty)$, we have
\begin{equation}\label{seq33}
	\begin{split} 
		l^{\gamma}\left(h(x),y\right) = \left\{ {\begin{array}{*{20}{c}}
				-\log\left( {h_y}(x)\right),\; \gamma = 1;\qquad\quad\;\;\; \\
				\frac{\gamma}{\gamma-1}\left({1-{h_y}(x) }^{1-\frac{1}{\gamma}}\right),\;\text{others}.
		\end{array}} \right.
	\end{split}
\end{equation}
And the expected $\gamma$-loss an be delineated as follows:
\begin{equation}\label{seq34}
	\begin{split} 
	\mathbb{E}\left[l^{\gamma}\left(h(x),y\right)\right] = \left\{ {\begin{array}{*{20}{c}}
			\mathbb{E}_{XY}\left[-\log(h_y(x))\right],\; \gamma = 1;\qquad \qquad\qquad\\
			\mathbb{E}_{XY}\left[\frac{\gamma}{\gamma-1}\left(1-{h_y(x)}^{1-\frac{1}{\gamma}}\right)\right],\;\text{others}.
	\end{array}} \right.
	\end{split}
\end{equation}
By minimizing the expected $\gamma$-loss, we derive the optimal output $h^{*}(x)$ as depicted below:
\begin{equation}\label{seq35}
	\begin{split} 
	h_y^{*}(x) = \frac{{P(y|x)}^\gamma}{\sum\nolimits_{i} {{P(i|x)}^\gamma} },\quad \forall y\in[K].
	\end{split}
\end{equation}
\end{theorem}
\begin{proof}
	
	[1] Let us examine the scenarios in which $\gamma=1$:
	\begin{equation}\label{seq36}
		\begin{split} 
		l^{1}(h(x),y) &=\mathop {\lim }\limits_{\gamma  \to 1}\frac{\gamma}{\gamma-1}\left({1-h_y(x)}^{1-\frac{1}{\gamma}}\right)
		\\& = -\mathop {\lim }\limits_{\hat \gamma  \to 0} \frac{ h_y(x)^{\hat \gamma}-1}{\hat \gamma}
		 = -\log  h_y(x).
		\end{split}
	\end{equation}
	Subsequently, we aim to minimize the expected 1-loss, leading to the following optimization problem:
	\begin{equation}\label{seq37}
		\begin{split} 
			\left\{ {\begin{array}{*{20}{c}}
					\min -\sum\nolimits_{x,y} {P(x,y)}\log (h_y(x)),\qquad\qquad\quad\;\\
					{s.t.\;\sum\nolimits_y {h_y(x)}  = 1,\;h_y(x) \ge 0,\;\forall y\in[K].}
			\end{array}} \right.
		\end{split}
	\end{equation}
	The convex nature of the optimization problem in Equation (\ref{seq37}) is evident. Thus, the global optimum can be efficiently determined by employing Lagrangian multipliers $\eta$ and $\mu$, defined as follows:
	\begin{equation}\label{seq38}
		\begin{split} 
			\mathcal{L}(h(x),\eta,\mu) &=  -\sum\nolimits_{x,y} {P(x,y)}\log (h_y(x)) 
			\\&\qquad+ \eta(\sum\nolimits_y {h_y(x)}-1) + \sum\nolimits_y \mu h_y(x).
		\end{split}
	\end{equation}
	Upon equating the gradients to zero, we have
	\begin{equation}
		\begin{split} 
			\frac{{\partial	\mathcal{L}(h(x),\eta,\mu)}}{{\partial h_y(x)}} = -\frac{P(x,y)}{{h_y(x)}} + \eta + \mu = 0,
		\end{split}
	\end{equation}
    Upon solving the above equation, we derive the corresponding optimal value as:
	\begin{equation}
		\begin{split} 
			h^*_y(x) = \frac{P(x,y)}{ \eta + \mu}.
		\end{split}
	\end{equation}
	Since $h^*_y(x)$ must satisfy the condition $\sum\nolimits_y {h^*_y(x)}  = 1$, we have $ \eta + \mu = P(x)$, and  the final global optimal value $h^*(x)$ is
	\begin{equation}\label{seq39}
		\begin{split} 
		h_y^{*}(x) = P(y|x),\quad \forall y\in[K].
		\end{split}
	\end{equation}
    [2] For $\gamma \in(0,1)\cup (1,\infty)$, the minimal $\gamma$-loss is
    \begin{equation}\label{seq41}
    	\begin{split} 
    		\left\{ {\begin{array}{*{20}{c}}
    				{\min \mathbb{E}\left[ {{l^\gamma }(h(x),y)} \right] = \min -\sum\nolimits_{x,y}  \frac{\gamma}{\gamma-1}{P(x,y)} h_y(x)^{1 - \frac{1}{\gamma }},}\\
    				{s.t.\;\sum\nolimits_y {h_y(x)}  = 1,\;h_y(x) \ge 0,\;\forall y\in[K].}\qquad\quad\quad\quad\;\;\;\;
    		\end{array}} \right.
    	\end{split}
    \end{equation}
    Similarly, in tackling the aforementioned issue, we construct the Lagrange function with two Lagrange multipliers $\eta$ and $\mu$:
    \begin{equation}
    	\begin{split} 
    		\mathcal{L}(h(x),\eta,\mu) &=  -\sum\nolimits_{x,y}  \frac{\gamma}{\gamma-1}{P(x,y)} h_y(x)^{1 - \frac{1}{\gamma }} 
    		\\&\qquad+ \eta(\sum\nolimits_y {h_y(x)}-1) + \sum\nolimits_y \mu h_y(x),
    	\end{split}
    \end{equation}
    Upon computation, we determine the corresponding optimal value to be: 
    \begin{equation}
    	\begin{split} 
    		h^*_y(x) = \left(\frac{P(x,y)}{ \eta + \mu}\right)^\gamma.
    	\end{split}
    \end{equation}
    By applying the Karush-Kuhn-Tucker~(KKT) conditions, we ascertain the optimal values of the Lagrange multipliers as
    \begin{equation}\label{seq42}
    	\begin{split} 
         \eta + \mu = \sum\nolimits_{i} {{P(x,y_i)}^\gamma} 
    	\end{split}
    \end{equation}
   with the optimal solution as 
   	\begin{equation}\label{seq43}
   		\begin{split} 
        	h_y^{*}(x) = \frac{{P(y|x)}^\gamma}{\sum\nolimits_{i} {{P(i|x)}^\gamma} },\quad \forall y\in[K].
   		\end{split}
   	\end{equation}
   \end{proof}
   Theorem \ref{th5} elucidates that  the parameter $\gamma$ has the ability to 'soften' the outputs, thereby increasing the output entropy with the range of $0< \gamma <1$. As $\gamma \to \infty$, the output $h(x)$ converges towards a singular point mass configuration~($h_y(x)=1$ and $h_i(x)=0$, $\forall i \ne y$). When $\gamma =1$, the original probability distribution is restored. In the limit as $\gamma \to 0$, the output $h(x)$ gravitates towards $1/K$, indicating maximum uncertainty. Notably, the $\gamma$-loss exhibits similarities to the post-processing calibration technique known as temperature scaling~\cite{icml/GuoPSW17}, which has demonstrated remarkable efficacy in enhancing prediction calibration. Consequently, the $\gamma$-loss serves as an effective mechanism for calibrating the outputs of classifiers.
   
   Upon establishing the loss function, we revisit the theoretical ADM framework elucidated in Eq. (\ref{seq23}). In practical applications, the target label distribution is often unknown, leading to ambiguity regarding the true weight $w^*$. Within the ADM framework, we employ two distinct strategies for estimating the true weight. The first strategy involves the two-step approaches that integrate traditional techniques, whereas the second strategy involves a one-step approach grounded in bi-level optimization principles.

   \subsection{Two-step Approaches}
   In addressing the aforementioned issue, we employ established label shift methods~(such as BBSE, RLLS, MLLS and SCML) to approximate the true weight, denoted as ${\hat w}^*$. Consequently, the theoretical loss function Eq.~(\ref{seq23}) undergoes a transformation, yielding the subsequent expression:
   \begin{equation}\label{seq44}
   	\begin{split}
   				 \mathop {\min }\limits_{h} \beta\mathbb{\hat E}(h,P_s^{{\hat w}^*}) +(1 - \beta ) \mathbb{\hat E}(h,{P_t}),
   	\end{split}
   \end{equation}
   which is equal to 
   \begin{equation}\label{seq45}
   	\begin{split}
   		\mathop {\min }\limits_{h} \frac{\beta}{n}\sum\limits_{i = 1}^n {{\hat w}^*({y_i})l^{\gamma}(h({x_i}),{y_i})} + \frac{1 - \beta }{m}\sum\limits_{j = n+1}^{n+m}	l^{\gamma}_u(h(x_j);P_t).
   	\end{split}
   \end{equation}
   Upon deriving the optimal target classifier, predictions can be generated for every dataset adhering to the label distribution. The primary steps of our proposed two-step approaches are succinctly outlined in Algorithm \ref{alg1}.
   \begin{algorithm}[ht]
   	\caption{Procedure of two-step approaches}
   	\label{alg1}
   	\begin{algorithmic}
   		\STATE \textbf{Input}: The labeled source samples $\{x_i,y_i\}_{i=1}^{n}$, unlabeled target samples $\{x_i\}_{i=n+1}^{n+m}$ and test target set $X_{te}$, tradeoff parameter $\beta$ and 'soften' parameter $\gamma$.
   		\STATE \textbf{Initialize}: Initialize network parameters $h_0$.
   		\STATE \qquad Use traditional label shift method (such as BBSE, RLLS, MLLS and SCML) to estimate importance weight ${\hat w}^*$;
   		\STATE \textbf{while} not converged \textbf{do} 		
   		\STATE \qquad Update the network parameters $h$ via stochastic gradient descent by solving Eq. (\ref{seq45});
   		\STATE \textbf{end}
   		\STATE \qquad Use the ultimate classifier to make predictions $Y_{te}$ on the test set $X_{te}$;
   		\STATE \textbf{Output}: The ultimate classifier and the predicted labels $Y_{te}$.		
   		\STATE \textbf{End procedure}
   	\end{algorithmic}
   \end{algorithm}
   
   Algorithm \ref{th1} illustrates that the two-step approaches separate the crucial weight estimation on mixed distributions from the  target classifier training. This segregation not only guarantees the robustness of weight estimation, but also raises a pertinent query: the weight estimation process hinges on the classifier predictions for the target data, and conversely, the estimated weights exert influence on the classifier training. Is it plausible to concurrently address both these facets to enhance the model's efficacy? Consequently, in the subsequent subsection, we introduce a one-step approach to tackle the aforementioned challenge.
   \subsection{One-step Approach}
   \subsubsection{Algorithm}
   Traditional label shift methods do not incorporate the target samples during classifier training; however, they necessitate the predictions of the target samples to estimate the importance weights. Therefore, they inherently operate as two-step procedures. In our theoretical framework, the inclusion of target samples in the training phase presents an opportunity to train the target classifier using a single loss function.
   
   Let us consider the case where $\gamma = 1$. According to Theorem \ref{th3}, upon acquiring the optimal weight $w^*$, the optimal target classifier $h_t$ can be derived by minimizing Eq. (\ref{seq23}). Subsequently, in accordance with Theorem \ref{th5}, the predictions for target samples yield $h_t(x) = P_t(Y|x), \forall x \sim  P_t$. Concurrently, guided by the definition of $w^*$ in Theorem \ref{th3}, we have
   \begin{equation}
   \begin{split} 
   	{w^*} = \frac{{{P_t}(Y)}}{{{P_s}(Y)}} = \frac{{\sum\nolimits_{i = n + 1}^{n + m} {{P_t}(Y|{x_i})} }}{{m{P_s}(Y)}} = \frac{{\sum\nolimits_{i = n + 1}^{n + m} {{h_t}({x_i})} }}{{m{P_s}(Y)}}.
   \end{split}   
   \end{equation}
   In conclusion, based on the theoretical loss function Eq. (\ref{seq23}), we derive the following one-step loss formulation: 
    \begin{equation}\label{seq46}
    	\begin{split} 
    	\left\{ {\begin{array}{*{20}{c}}
    				{	\mathop {\min }\limits_{h} \mathcal{L}(h,\hat w) = \frac{\beta}{n}\sum\limits_{i = 1}^n {{\hat w}({y_i})l^{\gamma}(h({x_i}),{y_i})}} \qquad\qquad\qquad\qquad\qquad
    					\\{\qquad\qquad\qquad\qquad+ \frac{1 - \beta }{m}\sum\limits_{j = n+1}^{n+m}	l^{\gamma}_u(h(x_j);P_t)},\\
    				{s.t.\;\hat w = \arg {{\min }_{w\succeq0, w^{\rm T}{P_s}(Y)=1}}\left\|w - \frac{{\sum\nolimits_{i = n + 1}^{n + m} {{h}({x_i})} }}{{m{P_s}(Y)}}\right\|_F^2,\;\;\;\;}
    		\end{array}} \right.
    	\end{split}
    \end{equation}
    where $P_s(Y)$ is the source label distribution and determined by source samples, i.e. $P_s(Y=j) = \frac{1}{n}\mathbbm{1}\sum\nolimits_{i = 1}^{n}\{y_i=j\}, \forall j\in[K]$. 
    
    The proposed one-step approach manifests as a comprehensive hierarchical optimization conundrum, wherein model training and weight estimation represent distinct optimization levels. Diverging from the two-step strategies, the one-step approach embodies a notable advantage. Specifically, it facilitates explicit modeling and optimization of the interplay between the estimated weights $\hat w$ and the model parameters $h$ via an implicit gradient~(IG)-based optimization framework~\cite{pami/LiuGZML22}. Here, IG denotes the gradient of lower-level solution $\hat w$ in relation to the upper-level variable $h$. The integration of IG sets the one-step approach apart from the two-step approaches as the former jointly optimizes the upper-level and lower-level variables, in contrast to the latter that optimizes them independently~(minimizes final loss by fixing the weight). However, the reliance on IG makes the weight estimation more dependent on the classifier performance, potentially yielding less stable outcomes compared to the two-step methods.Subsequently, we delineate the methodology for solving the optimization problem specified in Eq. (\ref{seq46}).
    
    \subsubsection{Optimization}  In the context of gradient descent, the gradient of the objective function in Eq. (\ref{seq46}) can be expressed as: 
    \begin{equation}
    	\begin{split} 
    		\nabla\mathcal{L}(h,\hat w) =  {\nabla _h}\mathcal{L}(h,\hat w) + \underbrace{\frac{{\text{d}\hat w^{\rm T}}}{{\text{d}h}}}_{\text{IG}}{\nabla _{\hat w}}\mathcal{L}(h,\hat w),
    	\end{split}
    \end{equation}
    where $\nabla _h$ and $\nabla _{\hat w}$ denote the partial derivatives of the bi-variate function. Traditional bi-level methods derive the IG formula through the rigorous implicit function theory. Nevertheless, the calculation of IG poses inherent challenges, primarily stemming from the complexities associated with matrix inversions, second-order partial derivatives, and constraints. In the present study, it becomes apparent that the optimal solution $\tilde w$ of the unconstrained lower-level loss can be succinctly expressed as:
    \begin{equation}
    	\begin{split} 
    	  {\tilde w} = \frac{{\sum\nolimits_{i = n + 1}^{n + m} {{h}({x_i})} }}{{m{P_s}(Y)}}.
    	\end{split}
    \end{equation}
     Given that for every $x\in\mathcal{X}$, the output $h(x)$ satisfies the stipulations: $h(x)\succeq 0$ and $\sum\nolimits_{j =1}^{K} h_j(x)=1$, alongside $P_s(Y) \succ 0$, we can deduce
\begin{equation}
     	\begin{split} 
     \left\{ \begin{array}{l}
     	{\tilde w} \succeq 0,\\
     	{\tilde w}^{\rm T}{P_s}(Y) = \frac{1}{m}\sum\nolimits_{j =1}^{K}{\sum\nolimits_{i = n + 1}^{n + m} {{h_j}({x_i})} }=1.
     \end{array} \right.
    \end{split}
\end{equation}
Hence, we can derive the optimal solution $\hat w$ for the constrained lower-level loss:
\begin{equation}\label{seq63}
	\begin{split} 
		\hat w = {\tilde w} = \frac{{\sum\nolimits_{i = n + 1}^{n + m} {{h}({x_i})} }}{{m{P_s}(Y)}}.
	\end{split}
\end{equation}
We summarize the main procedure of our proposed two-step approaches in Algorithm \ref{alg2}.
   \begin{algorithm}[ht]
	\caption{Procedure of one-step approach}
	\label{alg2}
	\begin{algorithmic}
		\STATE \textbf{Input}: The labeled source samples $\{x_i,y_i\}_{i=1}^{n}$, unlabeled target samples $\{x_i\}_{i=n+1}^{n+m}$ and test target set $X_{te}$, tradeoff parameter $\beta$, upper-level learning rate $\upsilon $ and 'soften' parameter $\gamma$.
		\STATE \textbf{Initialize}: Initialize network parameters $h_0$ and weight $\hat w_0=\mathbf{1}$.
		\STATE \qquad Use the source and target samples to train a classifier $\hat h_0$ at the fixed weight $\hat w_0$; 
		\STATE \textbf{for}  Iteration $j = 1, . . . ,T$ do	
		\STATE \qquad \textbf{Lower-level:} Update the weight $\hat w_j$ by optimizing:
		\[\hat w_j =  \frac{{\sum\nolimits_{i = n + 1}^{n + m} {{\hat h}_{j-1}({x_i})} }}{{m{P_s}(Y)}};\]	
		\STATE \qquad \textbf{Upper-level:} Update the network parameters $\hat h_j$ via stochastic gradient descent calling:
		\[ \left. {\hat h_j = \hat h_{j-1} -\upsilon\left(  {\nabla _h}\mathcal{L}(h,\hat w) + {\frac{{\text{d}\hat w^{\rm T}}}{{\text{d}h}}}{\nabla _{\hat w}}\mathcal{L}(h,\hat w)  \right)} \right|_{\scriptstyle {\hat w=\hat w_j}\hfill\atop
			\scriptstyle {h=\hat h_{j-1}}\hfill};\]
		\STATE \textbf{end for}
		\STATE \qquad Use the updated classifier $\hat h_T$ to make predictions $Y_{te}$ on the test set $X_{te}$;
		\STATE \textbf{Output}: The ultimate classifier $\hat h_T$ and the predicted labels $Y_{te}$.		
		\STATE \textbf{End procedure}
	\end{algorithmic}
\end{algorithm}

\section{Experiments} 
In this section, we evaluate the performance of our ADM framework in comparison to other closely related methods in different aspects. There are totally three groups of experiments. Firstly, we compare the traditional label shift methods with their two-step extended counterparts and one-step approach under various shift scenarios. Subsequently, the second set of experiments delves into comprehensive analyses of ADM, including assessments of the deviation of importance weight estimation, examination of the impact of labeled and unlabeled data volumes, exploration of various parameter configurations and visualization of convergence curves. To culminate, we apply our approach to the application of COVID-19 diagnosis. Before delving into the specifics, we provide an overview of the used datasets, the comparative methodologies, network architecture details, and the parameter configurations.
\subsection{Configuration}
\textbf{Datasets.} We conduct a comprehensive evaluation of the performance and efficacy of ADM framework utilizing diverse datasets with numerous artificial shifts, including MNIST~\cite{lecun1998gradient}, Fasion MNIST~\cite{xiao2017fashion}, USPS~\cite{291440}, CIFAR10~\cite{krizhevsky2009learning} and CIFAR100~\cite{krizhevsky2009learning}. These datasets encompass varying numbers of categories, ranging from 10 to 100. In our experimental setup, each dataset is randomly divided into two balanced parts, representing the source and target domains correspondingly. Subsequently, we introduce two distinct types of shifts in our experiments: (1) Tweak-One shift, which makes the probability of a certain source class change to $\rho$, and the probability of other source classes keep the same ratio. (2) Dirichlet shift, which generates a Dirichlet distribution by the concentration parameter $\alpha$ and makes the source label scale consistent with it. We then sample the source and target data based on the aforementioned proportions, thereby establishing the source and target sets respectively. It is worth noting that to ensure the effectiveness of training process, each class is represented by a minimum of 30 data points within the shift set. The subsequent part provides a detailed exposition on the distinctive characteristics of the datasets utilized in our study.
\begin{itemize}
	\item MNIST\footnote{http://yann.lecun.com/exdb/mnist/}: A dataset consisting of 10 variations~('0' to '9') of handwritten Arabic numerals. The images are grayscale pictures with a resolution of 28x28 pixels. For our analysis, 1000, 10000 and 5000 samples are used as source, target and test sets respectively. 
	\item Fasion MNIST\footnote{https://www.kaggle.com/zalando-research/fashionmnist}: A dataset consisting of 10 types of fashion items: shirt, T-shirt, pullover, dress, coat, trouser, bag, sandals, sneaker, and ankle boots. For our analysis, 2000, 20000 and 5000 samples are used as source, target and test sets respectively. 
	\item USPS\footnote{https://www.csie.ntu.edu.tw/~cjlin/libsvmtools/datasets}:  A dataset consisting of 10 numeric categories, which is similar to MNIST. The images are grayscale pictures with a resolution of 16x16 pixels. For our analysis, 1000, 3000 and 1000 samples are used as source, target and test sets respectively. 
	\item CIFAR10\footnote{https://www.cs.toronto.edu/~kriz/cifar.html}: A dataset consisting of colored images of 10 items: airplane, cat, deer, dog, automobile, frog, bird, ship, horse and truck. For our analysis, 2000, 20000 and 5000 samples are used as source, target and test sets respectively. 
	\item CIFAR100\footnote{https://www.cs.toronto.edu/~kriz/cifar.html}: A dataset consisting of colored images of 100 items. These images have a resolution of 64x64 pixels. For our analysis, 5000, 20000 and 5000 samples are used as source, target and test sets respectively. 
\end{itemize}

\noindent\textbf{Methods.} In the primary experimental section, we demonstrate the efficacy of ADM framework through a comparative analysis involving multiple methods and their respective variants. 
\begin{itemize}
	\item WW illustrates the performance of the base classifier in the absence of the estimated importance weights.
	\item CSSL~(2022)~\cite{aminian2022information}  is a semi-supervised method adept at mitigating the challenges posed by covariate shift scenario.
	\item  BBSE~(2018)~\cite{lipton2018detecting} and RLLS~(2019)~\cite{iclr/Azizzadenesheli19} are two label shift methods founded on hard and soft confusion matrices, respectively. On the other hand, MLLS~(2020)~\cite{nips/GargWBL20} and SCML~(2022)~\cite{wacv/SipkaSM22} are  two additional label shift methods that incorporate classifier calibration. 
	\item ADM-BBSE, ADM-RLLS, ADM-MLLS and ADM-SCML represent four two-step variations aimed at showcasing the efficacy of the ADM framework, where the base methods align with the nomenclature suffix. Furthermore, ADM-OS is the one-step approach of our ADM framework. 
	\item AvgImp denotes the average enhancement in performance achieved by both the two-step and one-step approaches across all label shift methods. 
\end{itemize}
 
\noindent\textbf{Network architecture and evaluation indicators.} All methods exhibit the flexibility to utilize any classifier as the foundational model for weight estimation. Moreover, under unchanged conditions, a more precise base~(source) classifier results in more accurate weight estimations and subsequent weighted classifier. In our experimental setup, we utilize a two-layer fully connected neural network for the MNIST, Fasion MNIST and USPS datasets, while employing ResNet18 for the CIFAR10 and CIFAR100 datasets. Each distribution parameter is randomly sampled 10 times for every shift type to assess the Accuracy~(Acc) and Mean Square Error~(MSE) of the estimated weights. Specifically, MSE is defined as:
\begin{equation}\label{eq39}
	\begin{split}
		\text{MSE}(\hat w) = \frac{1}{\text{len}(\hat w)}\left\|\frac{\hat w}{ p_s(Y)} - \frac{p_t(Y)}{p_s(Y)}\right\|^2.
	\end{split}
\end{equation}

\noindent\textbf{Parameter setting.} For the comparison methods, all parameters are selected in accordance with the strategies outlined in their respective references. In our study, the balance parameter, denoted as ${(1-\beta) \mathord{\left/{\vphantom {(1-\beta) \beta}} \right.\kern-\nulldelimiterspace} \beta}$, is chosen from the discrete set $[0.01,0.05,0.1,0.5,1]$, while the calibration parameter $\lambda$ is selected from the discrete set $[0.8,0.9,1,1.5,2]$. Moreover, we choose the values of shift parameter $\alpha$ from the set $[0.1,0.5,1,5]$(smaller values of $\alpha$ result in more extreme label shift), and select the values of shift parameter $\rho$ from the set $[0.3,0.5,0.7,0.9]$(larger values of $\rho$ result in more extreme label shift).

\begin{table*}[!htbp]
	\caption{Acc performance(mean(std)) comparison on Dirichlet and Tweak-One shift datasets. Improvements of ADM framework are boldfaced.}
	\label{Tab1}
	\centering
	\setlength{\tabcolsep}{3.5pt}
	\renewcommand\arraystretch{1.1}
	\begin{tabular}{c|c|cccc|cccc}
		\toprule[1.5pt]
		\midrule[0.75pt]
		Dataset &Methods &{$\alpha = 0.1$}&{$\alpha = 0.5$} &{$\alpha = 1$} &{$\alpha = 5$} &{$\rho = 0.3$}&{$\rho = 0.5$} &{$\rho = 0.7$} &{$\rho = 0.9$}	\\
		\midrule[0.75pt]
		\multirow{12}{*}{MNIST}
		&WW       & 0.7621(.0079) & 0.7897(.0068) & 0.8277(.0089) & 0.8592(.0038) & 0.8604(.0021) & 0.8402(.0065) & 0.8127(.0040) & 0.6436(.0299) \\
		&CSSL     & 0.7990(.0075) & 0.8222(.0062) & 0.8492(.0070) & 0.8709(.0037) & 0.8714(.0013) & 0.8578(.0061) & 0.8381(.0035) & 0.7453(.0158) \\
		&BBSE     & 0.8212(.0046) & 0.8429(.0023) & 0.8530(.0037) & 0.8688(.0041) & 0.8628(.0013) & 0.8565(.0043) & 0.8455(.0043) & 0.7195(.0256) \\
		&ADM-BBSE & 0.8306(.0052) & 0.8514(.0045) & 0.8654(.0048) & 0.8800(.0038) & 0.8720(.0014) & 0.8660(.0046) & 0.8557(.0036) & 0.7382(.0221) \\
		&RLLS     & 0.8237(.0056) & 0.8418(.0037) & 0.8567(.0046) & 0.8663(.0037) & 0.8645(.0022) & 0.8541(.0053) & 0.8478(.0029) & 0.7494(.0031) \\
		&ADM-RLLS & 0.8317(.0058) & 0.8525(.0026) & 0.8654(.0046) & 0.8789(.0042) & 0.8733(.0015) & 0.8665(.0052) & 0.8545(.0033) & 0.7808(.0077) \\
		&MLLS     & 0.8278(.0066) & 0.8480(.0025) & 0.8618(.0016) & 0.8689(.0032) & 0.8648(.0022) & 0.8542(.0053) & 0.8491(.0036) & 0.7391(.0223) \\
		&ADM-MLLS & 0.8369(.0054) & 0.8578(.0032) & 0.8693(.0026) & 0.8792(.0028) & 0.8733(.0017) & 0.8673(.0047) & 0.8559(.0036) & 0.7550(.0243) \\
		&SCML     & 0.8141(.0075) & 0.8449(.0048) & 0.8614(.0029) & 0.8678(.0034) & 0.8651(.0022) & 0.8548(.0055) & 0.8416(.0023) & 0.7003(.0396) \\
	    &ADM-SCML & 0.8299(.0066) & 0.8535(.0033) & 0.8680(.0032) & 0.8821(.0034) & 0.8759(.0016) & 0.8636(.0057) & 0.8509(.0030) & 0.7414(.0393) \\
	    &ADM-OS   & 0.8346(.0059) & 0.8544(.0029) & 0.8673(.0033) & 0.8794(.0037) & 0.8748(.0020) & 0.8640(.0051) & 0.8532(.0039) & 0.7862(.0067) \\
	    &AvgImp   & \textbf{.0106}/\textbf{.0129} & \textbf{.0094}/\textbf{.0120}  & \textbf{.0088}/\textbf{.0091} &  \textbf{.0141}/\textbf{.0114} 
	              & \textbf{.0093}/\textbf{.0105} & \textbf{.0110}/\textbf{.0091}  & \textbf{.0083}/\textbf{.0072} &  \textbf{.0267}/\textbf{.0590} \\
		\midrule[0.75pt]
		\multirow{12}{*}{\makecell[c]{Fasion\\MNIST}}
		&WW       & 0.6436(.0297) & 0.7088(.0156) & 0.7632(.0167) & 0.8036(.0039) & 0.8051(.0028) & 0.7885(.0055) & 0.7655(.0084) & 0.6367(.0093) \\
		&CSSL     & 0.7285(.0175) & 0.7482(.0145) & 0.7754(.0154) & 0.8082(.0026) & 0.8137(.0025) & 0.8063(.0038) & 0.7936(.0034) & 0.7039(.0144) \\
		&BBSE     & 0.7073(.0294) & 0.7337(.0219) & 0.7842(.0191) & 0.8073(.0033) & 0.8056(.0016) & 0.8018(.0064) & 0.7869(.0055) & 0.7167(.0164) \\
		&ADM-BBSE & 0.7146(.0309) & 0.7422(.0229) & 0.7896(.0185) & 0.8104(.0032) & 0.8105(.0028) & 0.8086(.0043) & 0.7998(.0048) & 0.7481(.0227) \\
		&RLLS     & 0.7277(.0147) & 0.7354(.0223) & 0.7836(.0192) & 0.8087(.0033) & 0.8069(.0018) & 0.7993(.0053) & 0.7893(.0041) & 0.7322(.0192) \\
		&ADM-RLLS & 0.7562(.0105) & 0.7670(.0232) & 0.7935(.0129) & 0.8120(.0037) & 0.8117(.0016) & 0.8075(.0043) & 0.7961(.0033) & 0.7752(.0191) \\
		&MLLS     & 0.7418(.0081) & 0.7659(.0077) & 0.7868(.0083) & 0.8081(.0038) & 0.8069(.0022) & 0.8025(.0052) & 0.7938(.0031) & 0.7541(.0072) \\
		&ADM-MLLS & 0.7525(.0080) & 0.7824(.0156) & 0.8006(.0077) & 0.8155(.0031) & 0.8112(.0020) & 0.8095(.0047) & 0.7995(.0034) & 0.7629(.0083) \\
		&SCML     & 0.7086(.0294) & 0.7535(.0184) & 0.7725(.0181) & 0.8041(.0039) & 0.8070(.0023) & 0.7950(.0183) & 0.7735(.0024) & 0.7263(.0034) \\
		&ADM-SCML & 0.7162(.0298) & 0.7673(.0140) & 0.7861(.0179) & 0.8101(.0031) & 0.8107(.0022) & 0.8089(.0200) & 0.7859(.0114) & 0.7669(.0039) \\
		&ADM-OS   & 0.7618(.0116) & 0.7884(.0388) & 0.7938(.0242) & 0.8128(.0063) & 0.8147(.0024) & 0.8158(.0371) & 0.8005(.0028) & 0.7736(.0068) \\
		&AvgImp   & \textbf{.0135}/\textbf{.0405} & \textbf{.0176}/\textbf{.0413}  & \textbf{.0107}/\textbf{.0120} &  \textbf{.0050}/\textbf{.0058}       
		          & \textbf{.0044}/\textbf{.0081} & \textbf{.0090}/\textbf{.0162}  & \textbf{.0095}/\textbf{.0146} &  \textbf{.0310}/\textbf{.0413} \\
		\midrule[0.75pt]
		\multirow{12}{*}{USPS}
		&WW       & 0.7436(.0293) & 0.7790(.0116) & 0.8028(.0112) & 0.8831(.0057) & 0.8914(.0040) & 0.8316(.0059) & 0.7001(.0125) & 0.1564(.0019) \\
		&CSSL     & 0.8223(.0264) & 0.8328(.0023) & 0.8657(.0104) & 0.9079(.0031) & 0.9087(.0043) & 0.8700(.0032) & 0.8035(.0077) & 0.1206(.0435) \\
		&BBSE     & 0.7946(.0239) & 0.8378(.0233) & 0.8641(.0274) & 0.9036(.0031) & 0.9014(.0028) & 0.8751(.0018) & 0.7211(.0424) & 0.1798(.0256) \\
		&ADM-BBSE & 0.8365(.0197) & 0.8589(.0249) & 0.8729(.0287) & 0.9095(.0019) & 0.9090(.0025) & 0.8892(.0025) & 0.7663(.0345) & 0.2089(.0292) \\
		&RLLS     & 0.8147(.0232) & 0.8373(.0254) & 0.8841(.0113) & 0.9028(.0025) & 0.9031(.0031) & 0.8788(.0038) & 0.7339(.0471) & 0.2249(.0470) \\
		&ADM-RLLS & 0.8357(.0282) & 0.8584(.0241) & 0.8961(.0095) & 0.9096(.0024) & 0.9095(.0028) & 0.8881(.0028) & 0.7848(.0529) & 0.2606(.0536) \\
		&MLLS     & 0.7778(.0255) & 0.8038(.0516) & 0.8925(.0088) & 0.9014(.0027) & 0.9014(.0036) & 0.8659(.0043) & 0.6682(.0545) & 0.2786(.0349) \\
		&ADM-MLLS & 0.8212(.0402) & 0.8563(.0138) & 0.9011(.0081) & 0.9090(.0019) & 0.9086(.0029) & 0.8822(.0022) & 0.7229(.0759) & 0.3386(.0521) \\
		&SCML     & 0.7657(.0376) & 0.8118(.0360) & 0.8968(.0054) & 0.9041(.0031) & 0.9005(.0034) & 0.8703(.0028) & 0.6532(.0276) & 0.3617(.0585) \\
		&ADM-SCML & 0.8026(.0534) & 0.8463(.0256) & 0.9043(.0055) & 0.9111(.0029) & 0.9093(.0030) & 0.8854(.0015) & 0.6918(.0207) & 0.4007(.0443) \\
		&ADM-OS   & 0.8747(.0064) & 0.8950(.0031) & 0.9123(.0038) & 0.9168(.0026) & 0.9178(.0025) & 0.9004(.0022) & 0.8753(.0019) & 0.7437(.0109) \\
		&AvgImp   & \textbf{.0358}/\textbf{.0865} & \textbf{.0282}/\textbf{.0723}  & \textbf{.0092}/\textbf{.0279} &  \textbf{.0068}/\textbf{.0148}       
		          & \textbf{.0075}/\textbf{.0162} & \textbf{.0137}/\textbf{.0279}  & \textbf{.0474}/\textbf{.1812} &  \textbf{.0410}/\textbf{.4825} \\
		\midrule[0.75pt]
		\multirow{12}{*}{CIFAR10}
		&WW       & 0.3502(.0399) & 0.4908(.0293) & 0.5265(.0323) & 0.5596(.0176) & 0.5671(.0059) & 0.5482(.0192) & 0.4763(.0017) & 0.2414(.0598) \\
		&CSSL     & 0.4785(.0338) & 0.5749(.0264) & 0.5936(.0204) & 0.5982(.0075) & 0.6113(.0046) & 0.5824(.0141) & 0.5137(.0301) & 0.3033(.0316) \\
		&BBSE     & 0.3729(.0520) & 0.5391(.0339) & 0.5789(.0146) & 0.6021(.0068) & 0.5916(.0059) & 0.5680(.0160) & 0.5099(.0285) & 0.3143(.0413) \\
		&ADM-BBSE & 0.4327(.0476) & 0.5593(.0197) & 0.5994(.0139) & 0.6314(.0035) & 0.6120(.0048) & 0.5939(.0145) & 0.5449(.0214) & 0.3313(.0343) \\
		&RLLS     & 0.3759(.0456) & 0.5537(.0232) & 0.5861(.0135) & 0.6150(.0058) & 0.5767(.0079) & 0.5742(.0209) & 0.5038(.0246) & 0.3250(.0163) \\
		&ADM-RLLS & 0.4189(.0355) & 0.5836(.0282) & 0.6041(.0066) & 0.6463(.0048) & 0.6143(.0049) & 0.5904(.0153) & 0.5286(.0191) & 0.3769(.0154) \\
		&MLLS     & 0.3907(.0280) & 0.5567(.0255) & 0.5887(.0091) & 0.5903(.0069) & 0.5973(.0076) & 0.5540(.0204) & 0.4975(.0354) & 0.2949(.0446) \\
		&ADM-MLLS & 0.4218(.0172) & 0.5821(.0125) & 0.6023(.0040) & 0.6114(.0035) & 0.6052(.0065) & 0.5871(.0178) & 0.5201(.0381) & 0.3266(.0486) \\
		&SCML     & 0.3564(.0462) & 0.5477(.0176) & 0.5755(.0083) & 0.5884(.0058) & 0.5931(.0074) & 0.5842(.0207) & 0.4449(.0573) & 0.2384(.0592) \\
		&ADM-SCML & 0.3955(.0490) & 0.5626(.0234) & 0.5989(.0078) & 0.6101(.0057) & 0.6207(.0049) & 0.6101(.0160) & 0.4828(.0324) & 0.2761(.0487) \\
		&ADM-OS   & 0.5081(.0388) & 0.5793(.0248) & 0.6124(.0053) & 0.6237(.0046) & 0.6065(.0043) & 0.5949(.0141) & 0.5294(.0208) & 0.3582(.0367) \\
		&AvgImp   & \textbf{.0433}/\textbf{.1341} & \textbf{.0226}/\textbf{.0300}  & \textbf{.0189}/\textbf{.0301} &  \textbf{.0235}/\textbf{.0248}       
        & \textbf{.0234}/\textbf{.0168} & \textbf{.0252}/\textbf{.0248}  & \textbf{.0301}/\textbf{.0404} &  \textbf{.0346}/\textbf{.0651} \\
		\midrule[0.75pt]
		\multirow{12}{*}{CIFAR100}
		&WW       & 0.1631(.0223) & 0.1845(.0174) & 0.2143(.0211) & 0.3335(.0114) & 0.3094(.0191) & 0.2001(.0173) & 0.1876(.0207) & 0.1849(.0185) \\
		&CSSL     & 0.1948(.0176) & 0.2071(.0164) & 0.2265(.0146) & 0.3620(.0108) & 0.3318(.0226) & 0.2301(.0158) & 0.2052(.0152) & 0.2073(.0115) \\
		&BBSE     & 0.1871(.0246) & 0.2032(.0143) & 0.2665(.0078) & 0.3595(.0101) & 0.3383(.0093) & 0.2175(.0151) & 0.1978(.0147) & 0.1582(.0244) \\
		&ADM-BBSE & 0.1984(.0218) & 0.2172(.0155) & 0.2853(.0108) & 0.3794(.0125) & 0.3504(.0102) & 0.2333(.0138) & 0.2103(.0144) & 0.1687(.0211) \\
		&RLLS     & 0.1898(.0201) & 0.1937(.0108) & 0.2713(.0097) & 0.3737(.0068) & 0.3476(.0101) & 0.2028(.0154) & 0.1974(.0201) & 0.1884(.0207) \\
		&ADM-RLLS & 0.1992(.0180) & 0.2149(.0258) & 0.2937(.0087) & 0.3817(.0087) & 0.3530(.0097) & 0.2265(.0171) & 0.2219(.0121) & 0.2043(.0141) \\
		&MLLS     & 0.1776(.0140) & 0.2092(.0266) & 0.2660(.0049) & 0.3732(.0053) & 0.3515(.0114) & 0.2243(.0123) & 0.2089(.0075) & 0.1931(.0108) \\
		&ADM-MLLS & 0.1920(.0144) & 0.2224(.0189) & 0.2848(.0064) & 0.3833(.0065) & 0.3631(.0087) & 0.2602(.0116) & 0.2383(.0101) & 0.2131(.0096) \\
		&SCML     & 0.1581(.0174) & 0.1982(.0202) & 0.2558(.0086) & 0.3732(.0062) & 0.3567(.0127) & 0.1989(.0155) & 0.2110(.0168) & 0.2082(.0191) \\
		&ADM-SCML & 0.1704(.0136) & 0.2098(.0154) & 0.2715(.0074) & 0.3848(.0085) & 0.3726(.0111) & 0.2096(.0145) & 0.2204(.0155) & 0.2176(.0184) \\
		&ADM-OS   & 0.0649(.0080) & 0.1233(.0133) & 0.1810(.0172) & 0.3116(.0169) & 0.2423(.0292) & 0.1288(.0230) & 0.0734(.0118) & 0.0632(.0109) \\
		&AvgImp   & \textbf{.0133}/\textbf{-.1118} & \textbf{.0150}/\textbf{-.0778}  & \textbf{.0189}/\textbf{-.0839} &  \textbf{.0124}/\textbf{-.0583}       
        & \textbf{.0112}/\textbf{-.1062} & \textbf{.0215}/\textbf{-.0821}  & \textbf{.0189}/\textbf{-.1304} &  \textbf{.0390}/\textbf{-.1237} \\
		\midrule[0.75pt]
		\bottomrule[1.5pt]
	\end{tabular}
\end{table*}

\begin{figure*}[tbp] \label{fig4}
	\centering
	\subfigure[MNIST($\alpha$)]{
		\includegraphics[width=0.24\textwidth]{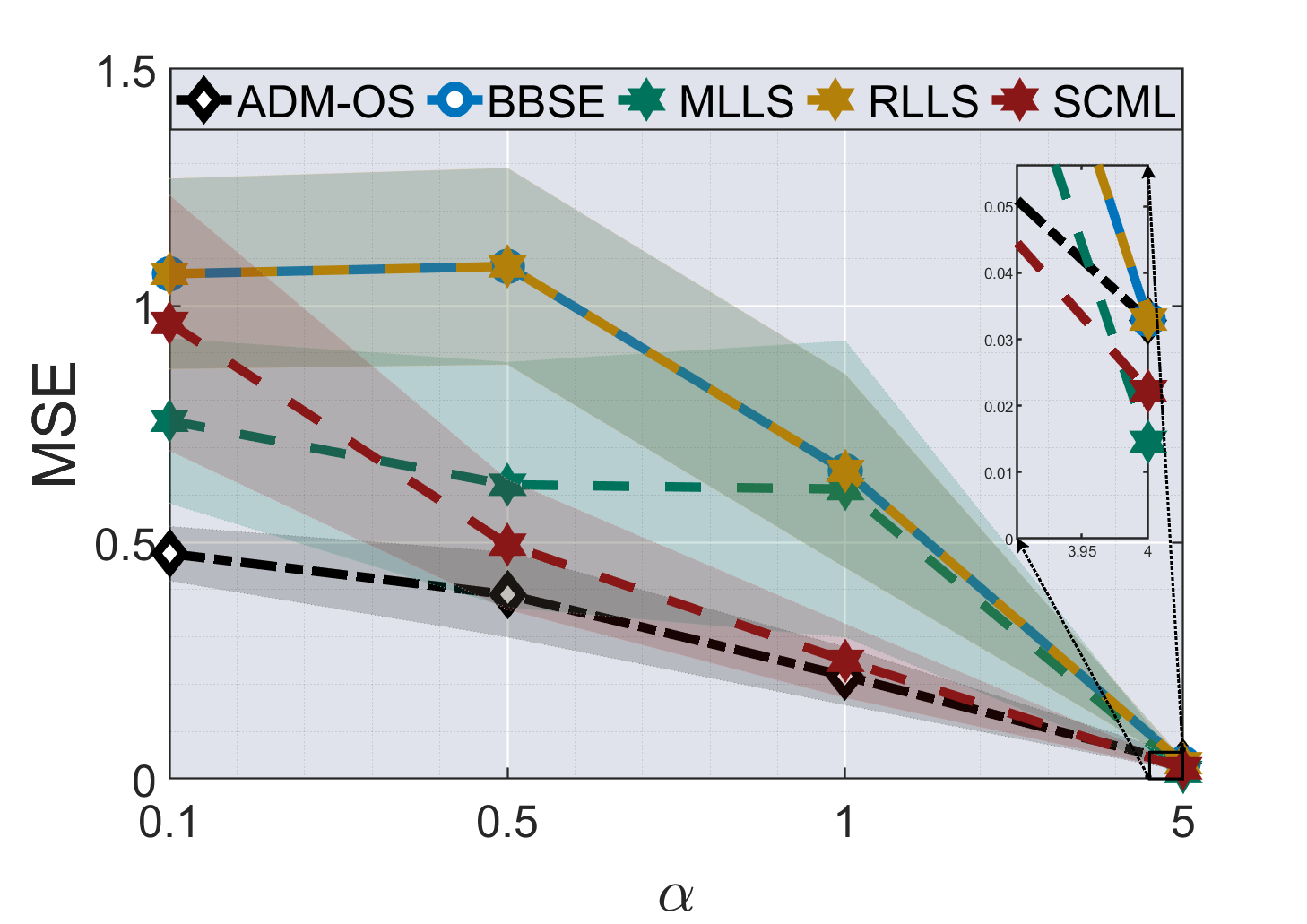}}
	\subfigure[MNIST($\rho$)]{
		\includegraphics[width=0.24\textwidth]{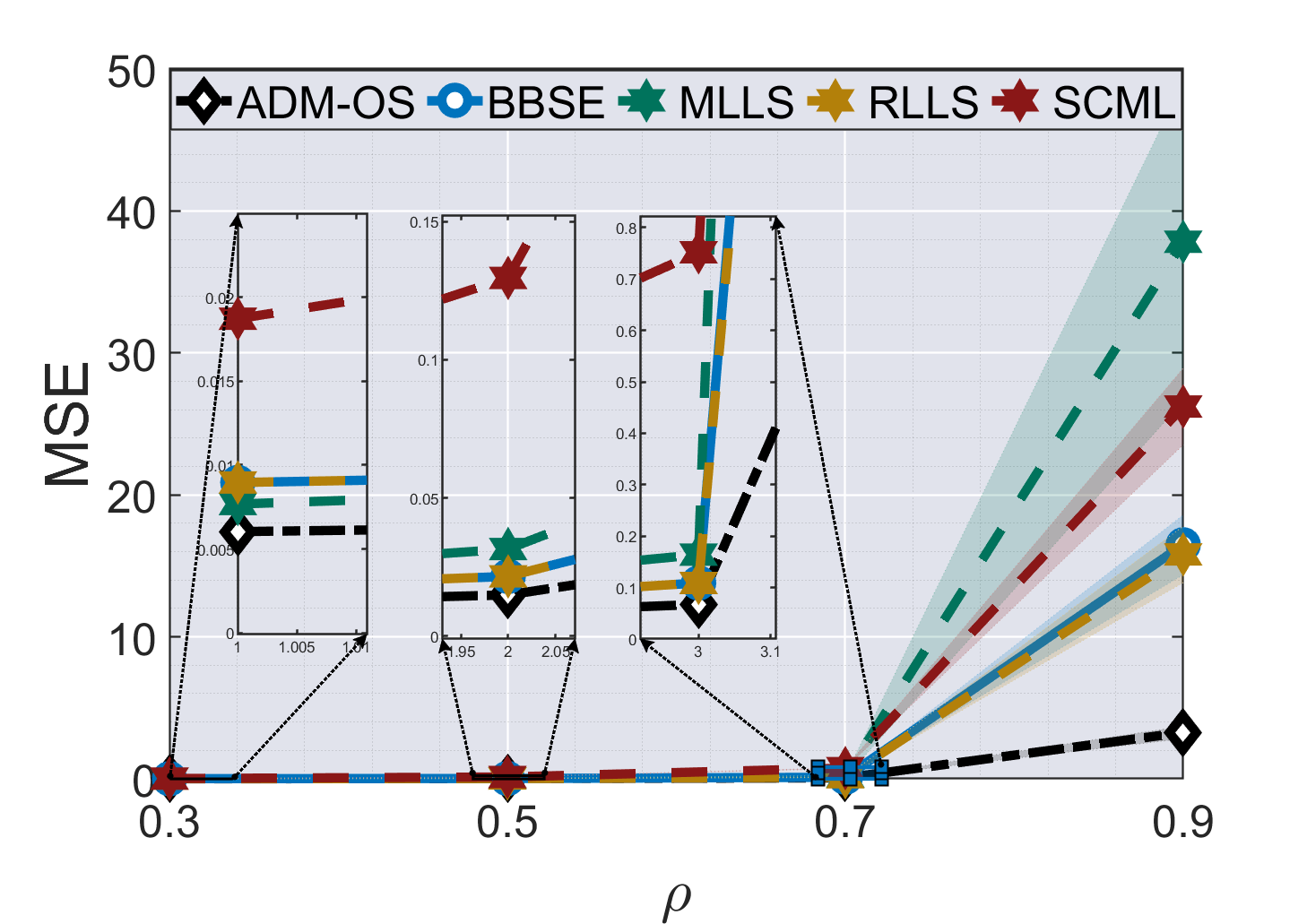}}
	\subfigure[FMNIST($\alpha$)]{
		\includegraphics[width=0.24\textwidth]{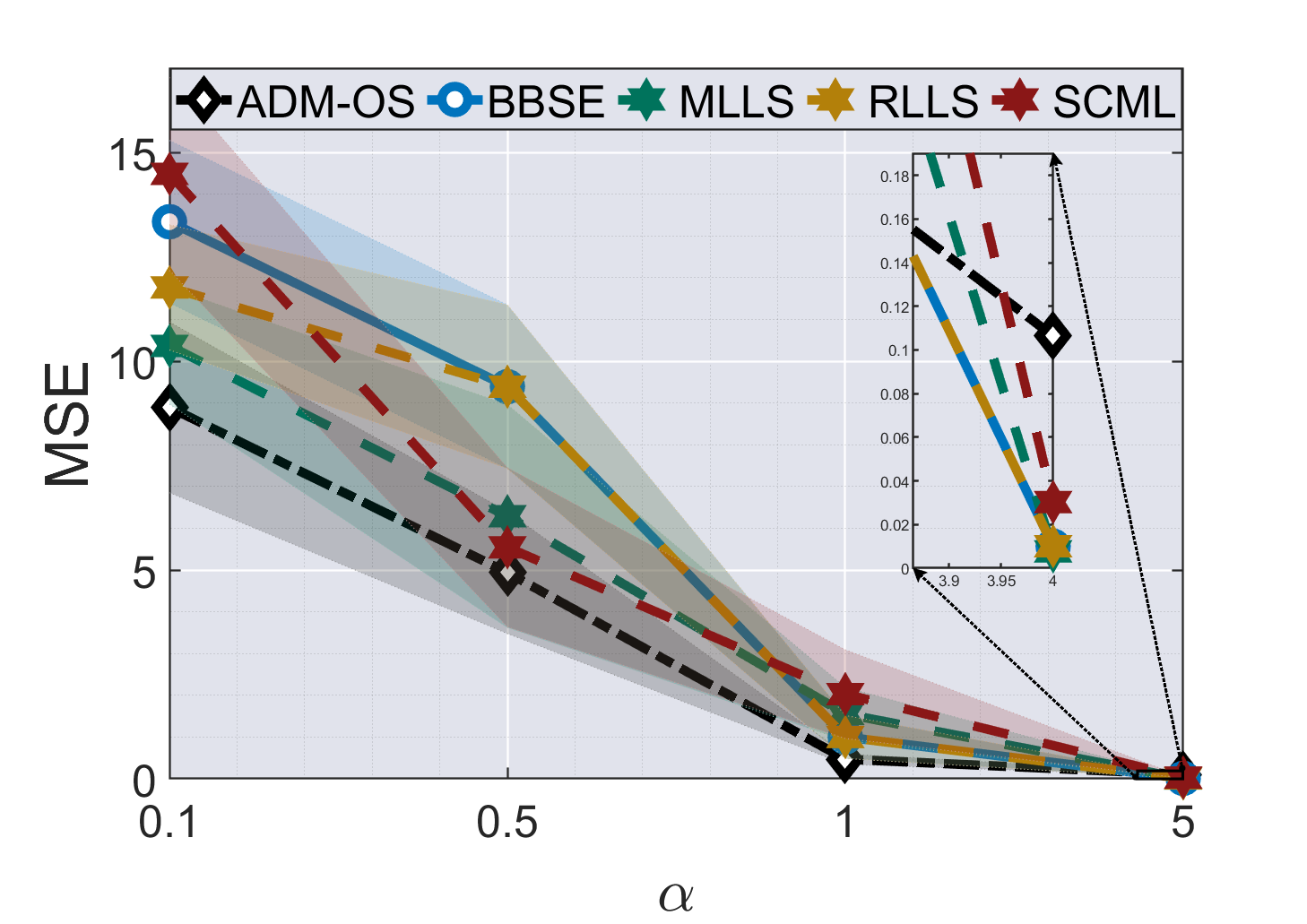}}
	\subfigure[FMNIST($\rho$)]{
		\includegraphics[width=0.24\textwidth]{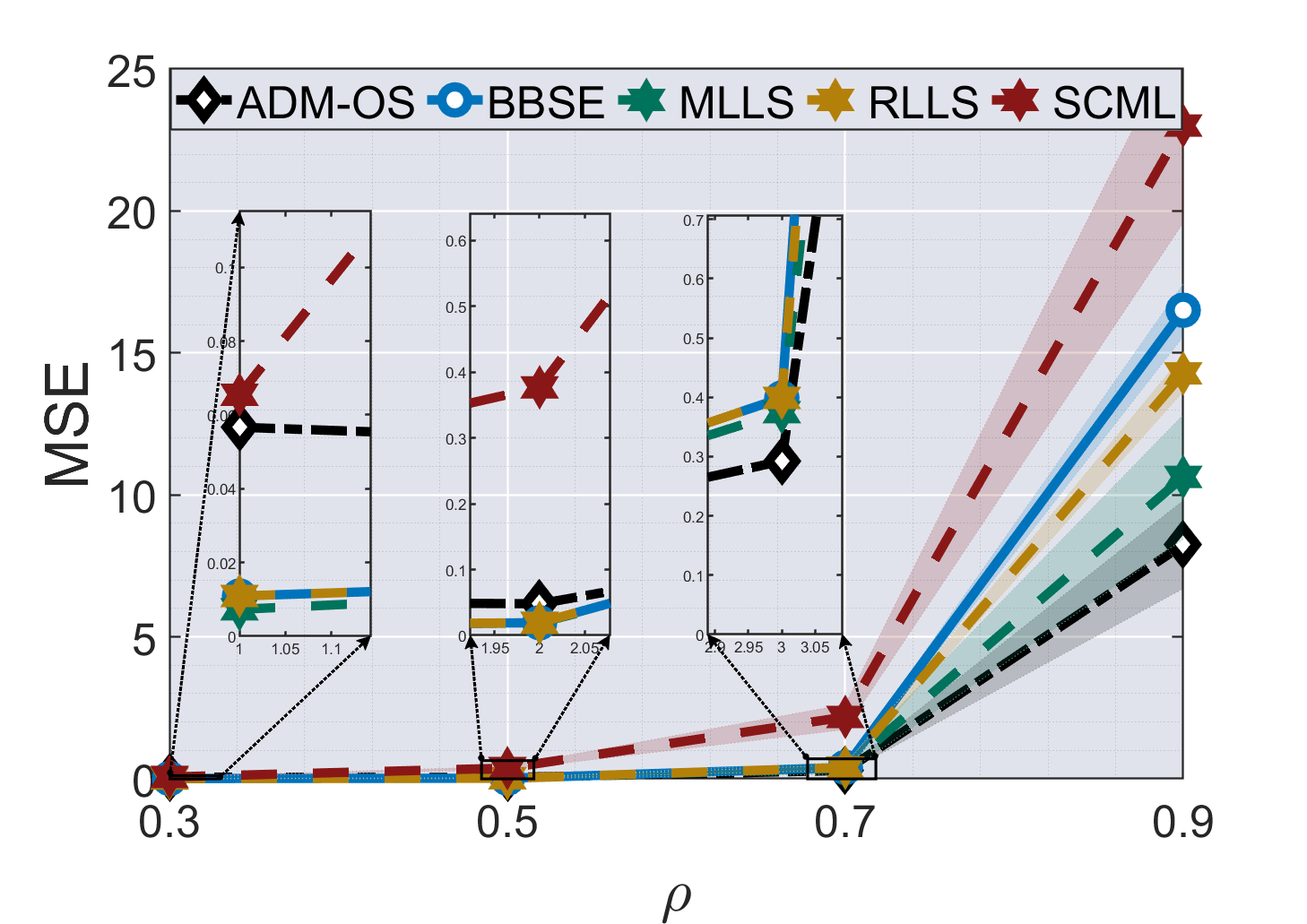}}
	\subfigure[USPS($\alpha$)]{
		\includegraphics[width=0.24\textwidth]{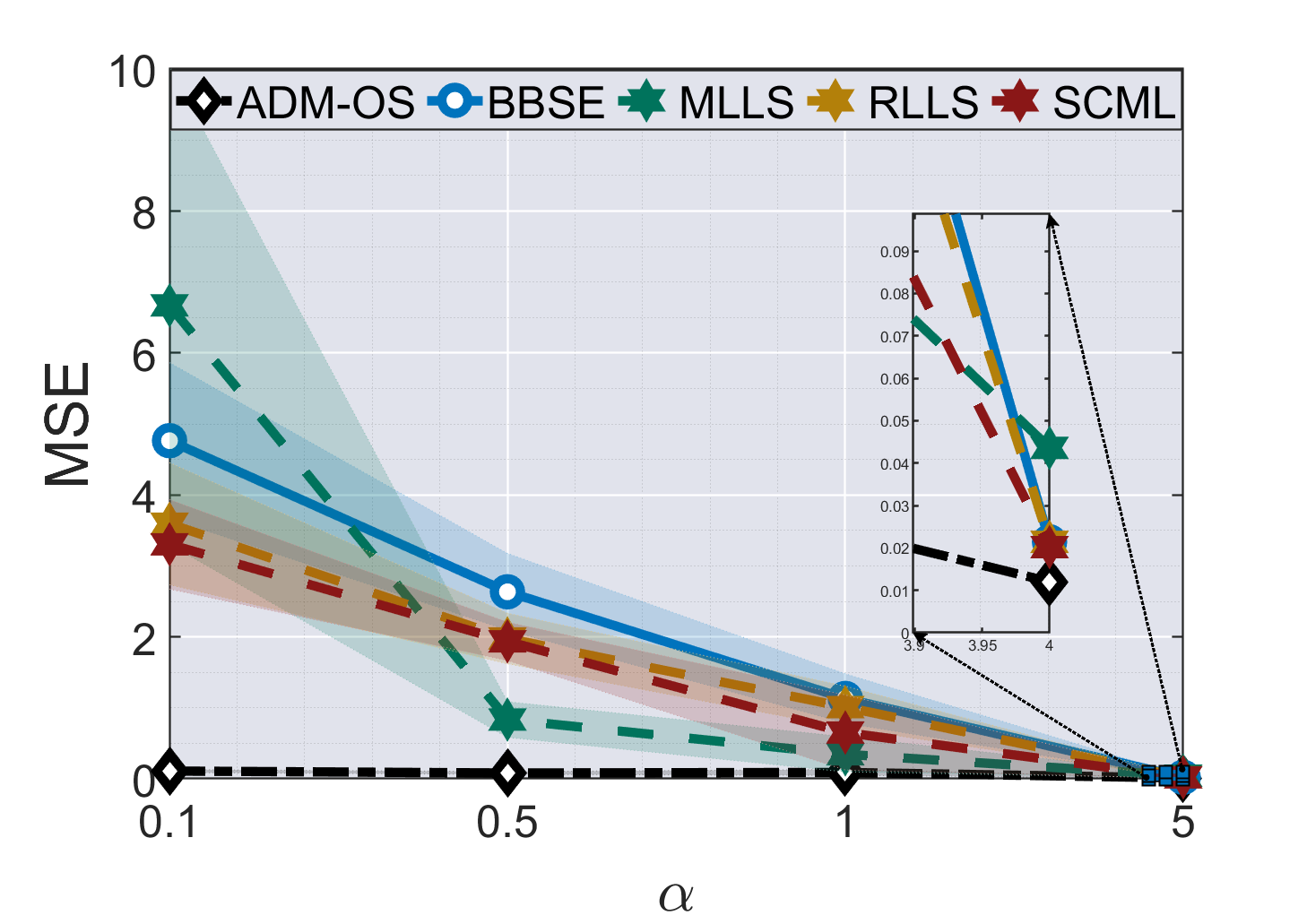}}
	\subfigure[USPS($\rho$)]{
		\includegraphics[width=0.24\textwidth]{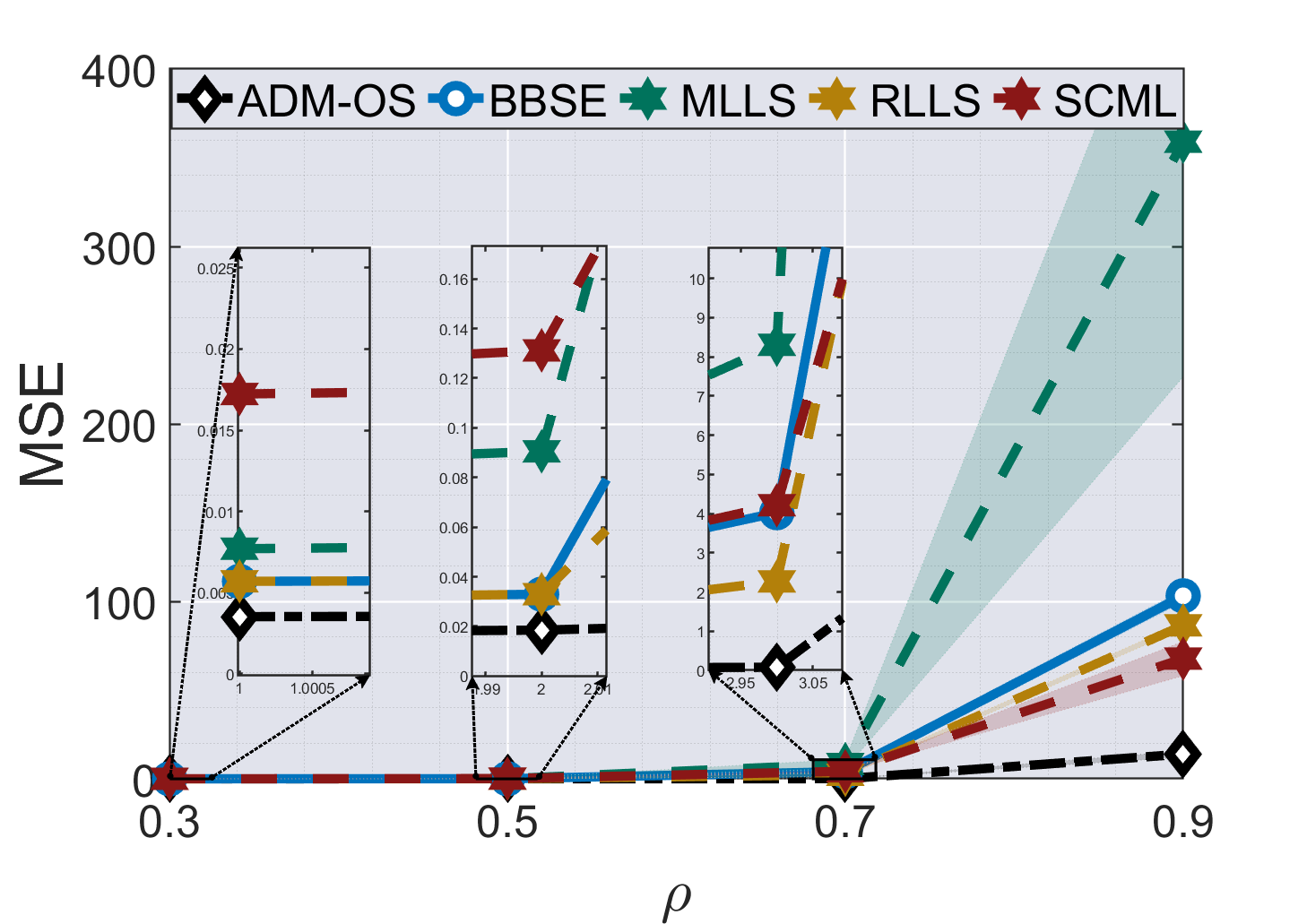}}
	\subfigure[CIFAR10($\alpha$)]{
		\includegraphics[width=0.24\textwidth]{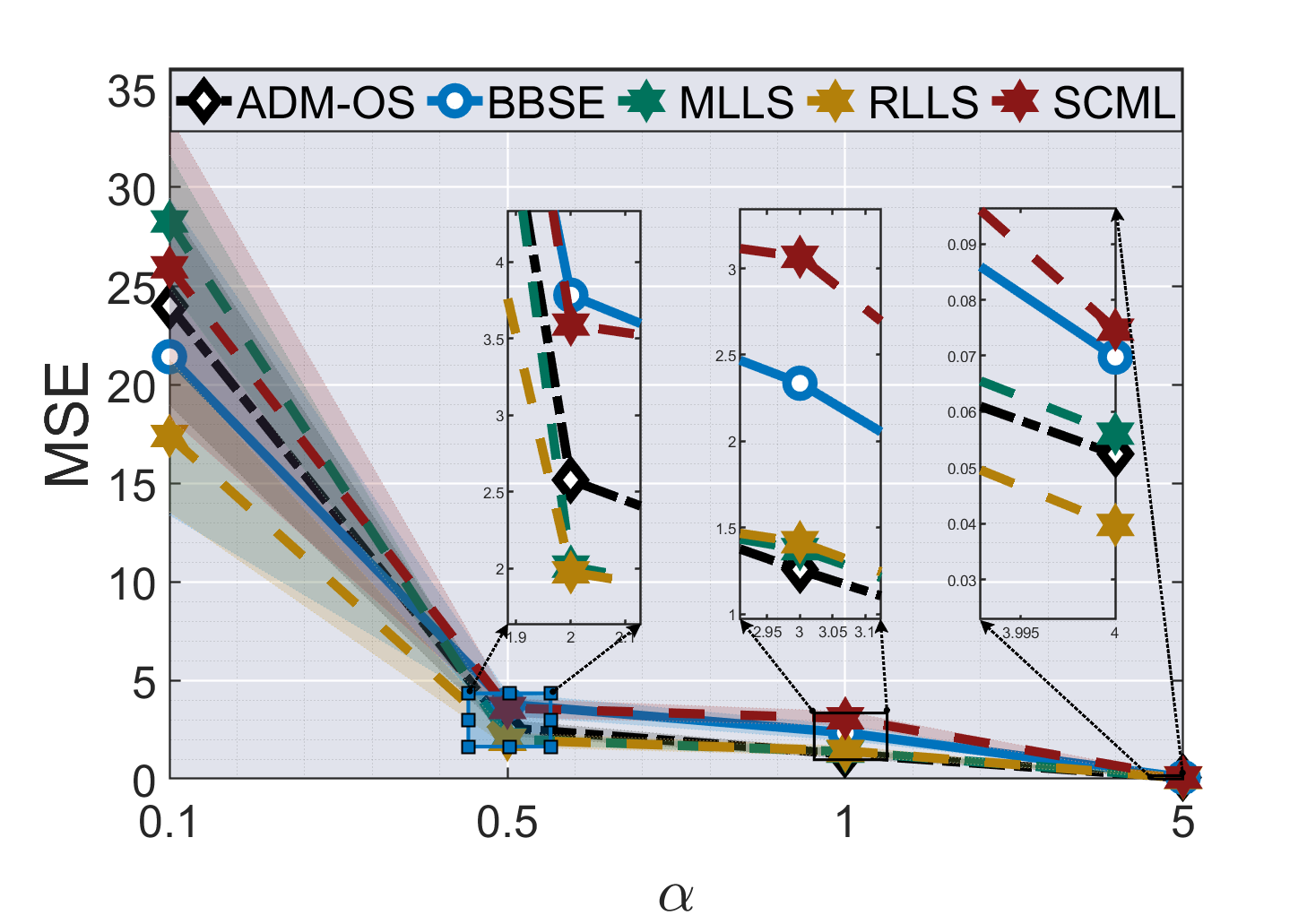}}
	\subfigure[CIFAR10($\rho$)]{
		\includegraphics[width=0.24\textwidth]{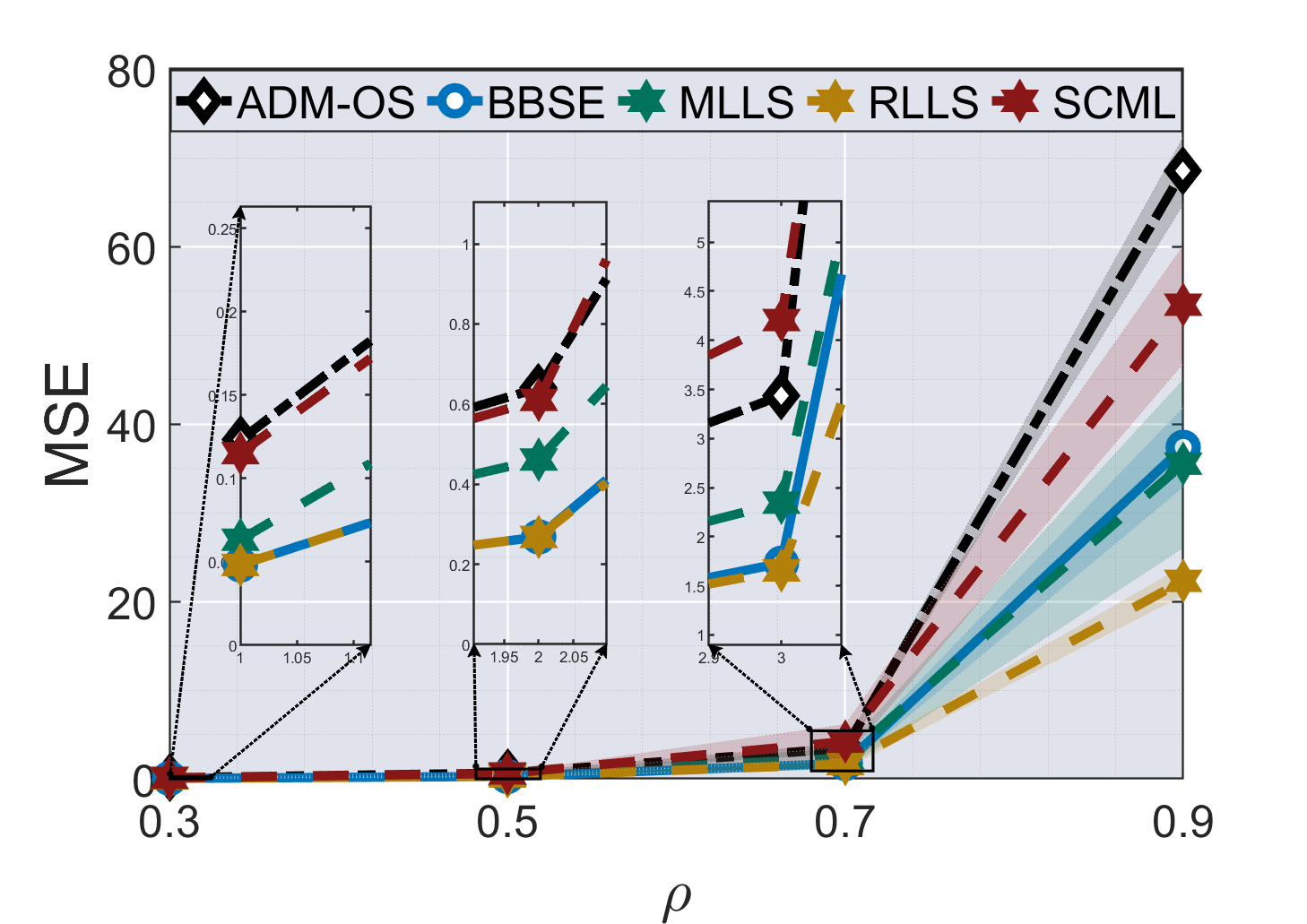}}
	\subfigure[CIFAR100($\alpha$)]{
		\includegraphics[width=0.24\textwidth]{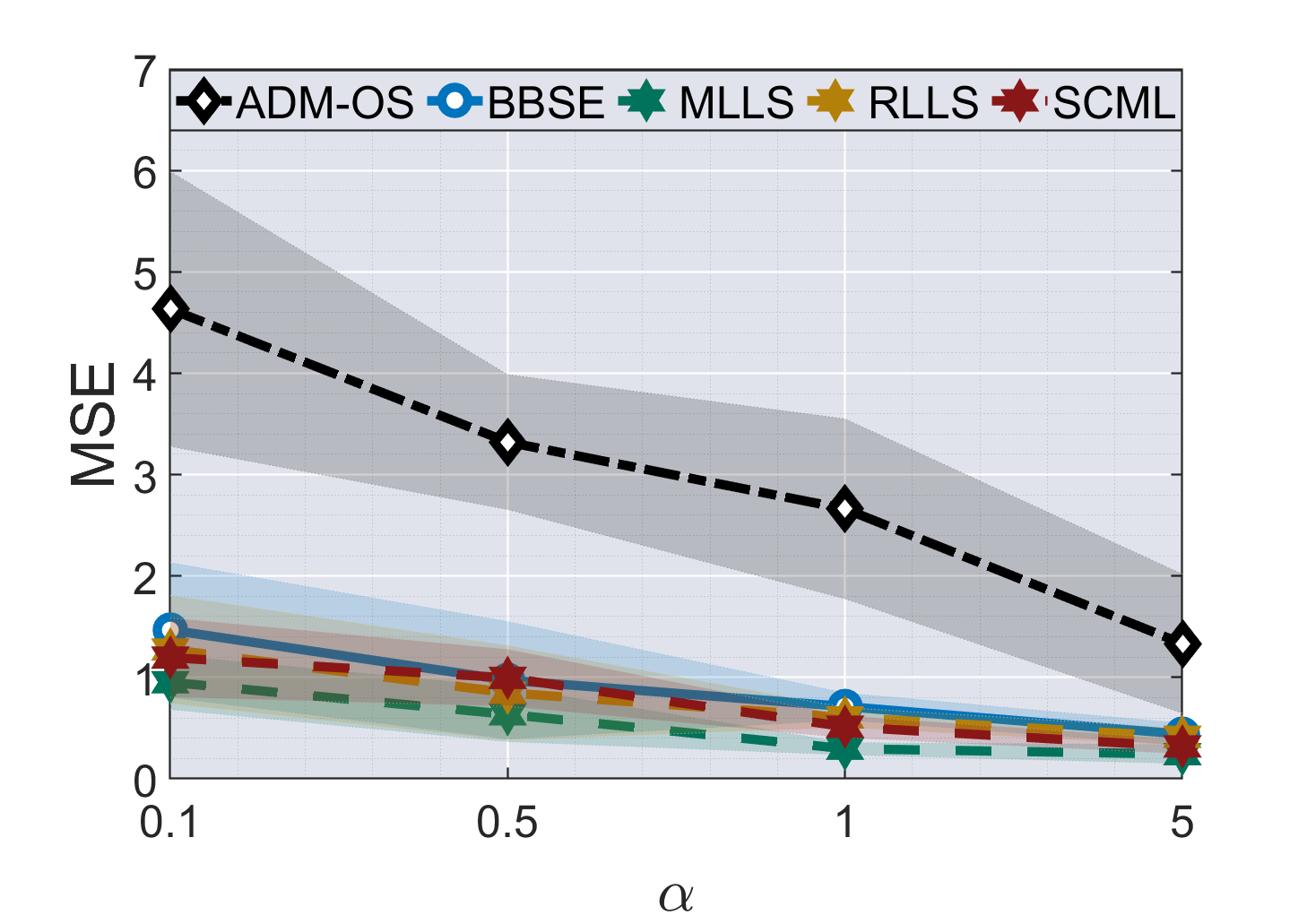}}
	\subfigure[CIFAR100($\rho$)]{
		\includegraphics[width=0.24\textwidth]{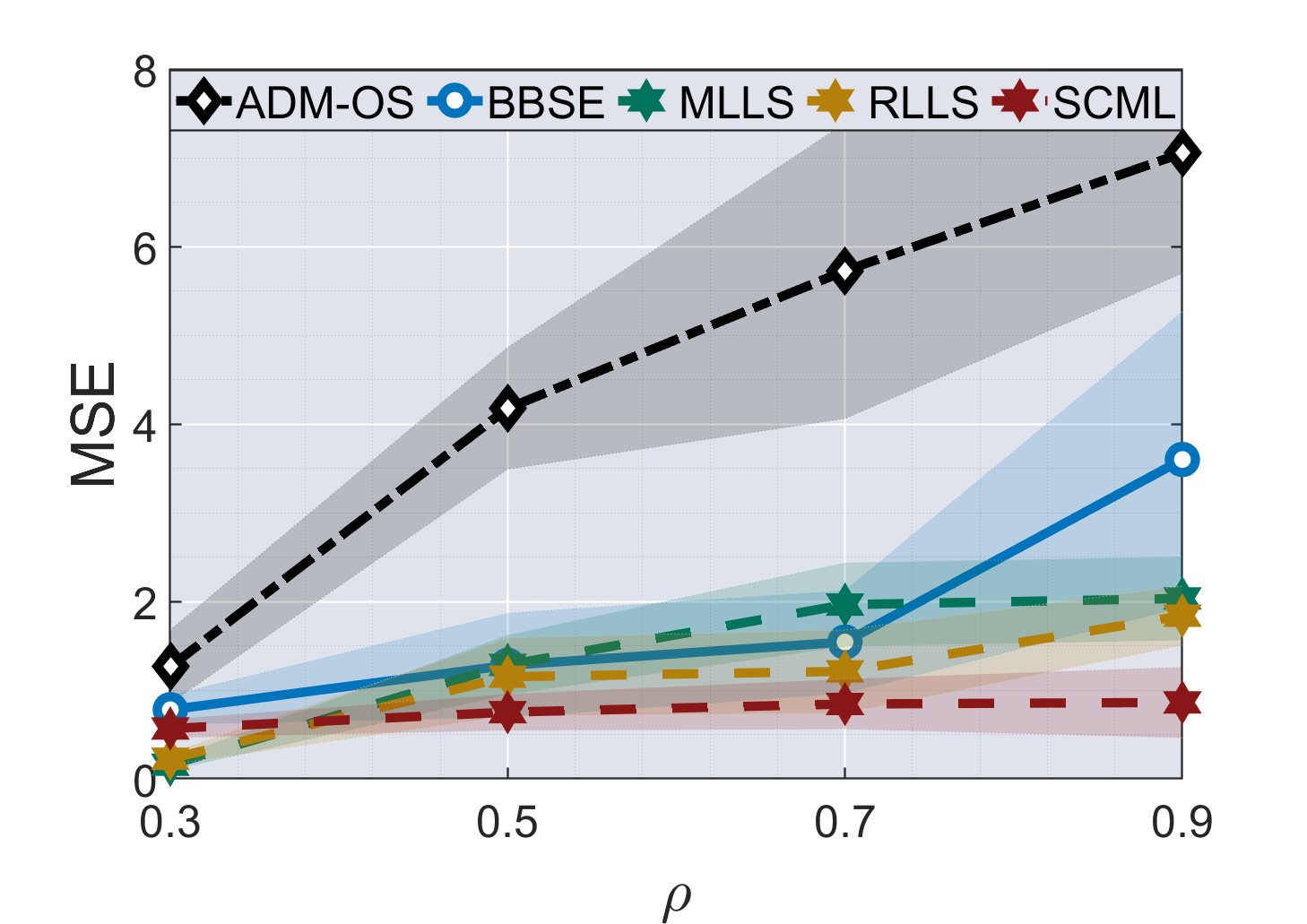}}
	\caption{The MSE comparison of ADM-OS(black line) with other label shift methods under different datasets.}
\end{figure*}
\begin{figure*}[htbp] \label{fig5}
	\centering
	\subfigure[MNIST($\alpha=1$)]{
		\includegraphics[width=1\textwidth]{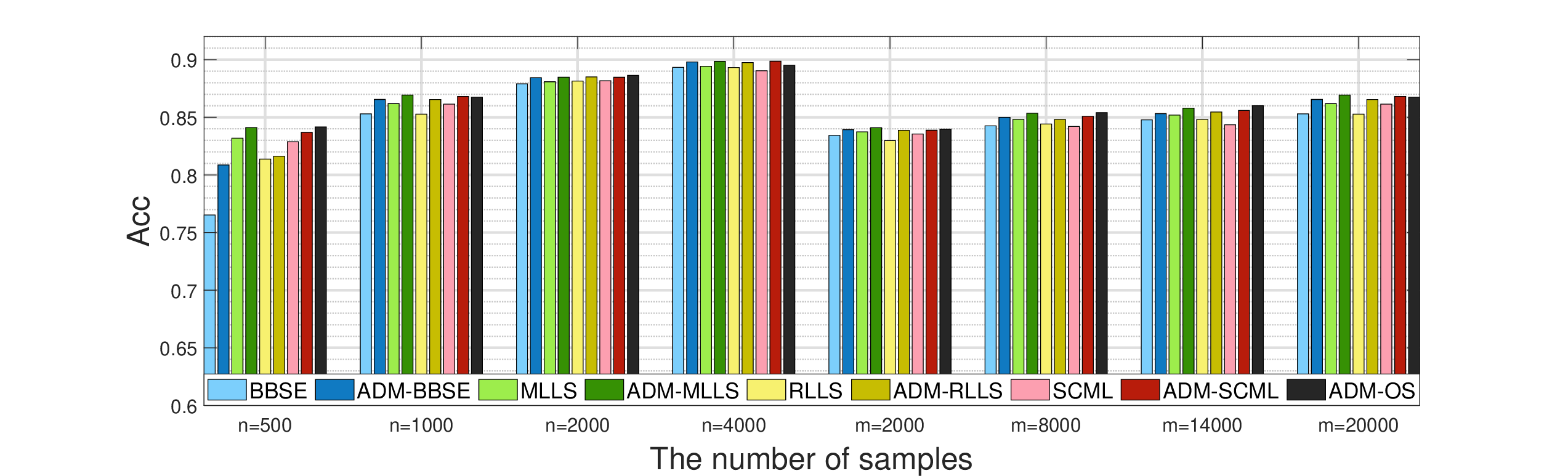}}
	\subfigure[Fasion MNIST($\alpha=1$)]{
		\includegraphics[width=1\textwidth]{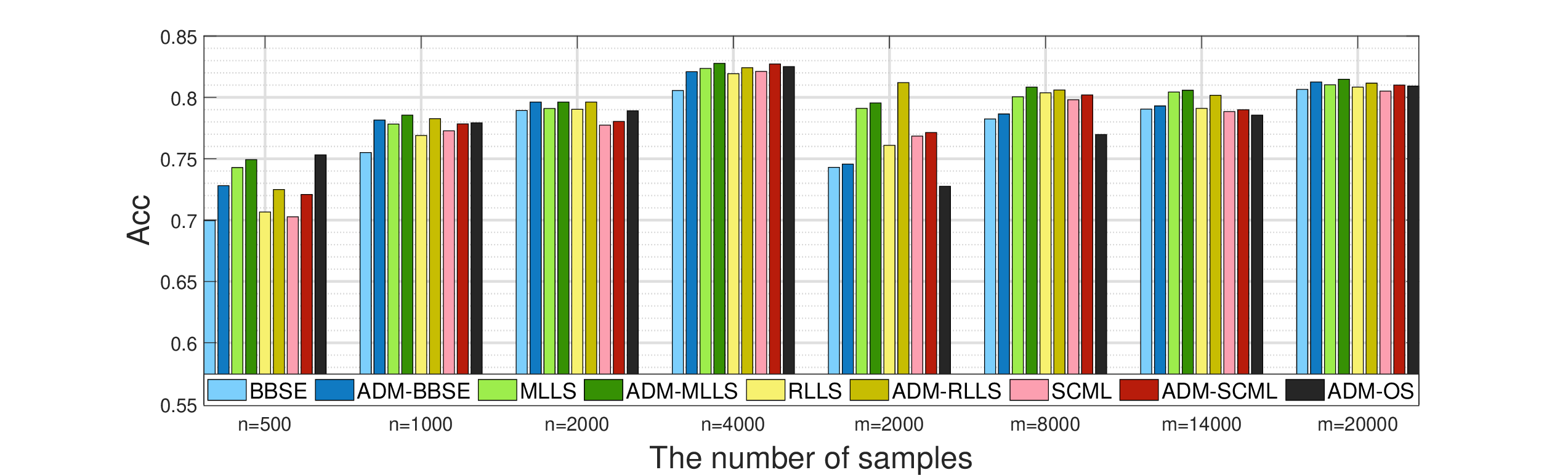}}
	\caption{The influence of the sample quantity from both the source and target domain on MNIST and Fasion MNIST datasets with $\alpha=1$.}
\end{figure*}
\subsection{Accuracy Comparison}
In order to verify the efficacy  of the proposed ADM framework, a comparative analysis is conducted against advanced label shift and SSL methods. The evaluation focuses on assessing the Acc on datasets with Dirichlet and Tweak-One label shifts. Notably, in consideration of the substantial number of classes in the Tweak-One shift CIFAR100 dataset, a normalization procedure is implemented to ensure that the proportions of 10 classes sum up to $\rho$. The experimental results are displayed in Table \ref{Tab1} and we have the following observations.
\begin{enumerate}
	\item While the performance of different comparative methods varies across different datasets, our two-step ADM approaches consistently enhance the performance of existing label shift methods in all cases, as evidenced by the AvgImp metric. Notably, on the Tweak-One shift USPS dataset with $\rho =0.9$ and the Dirichlet shift CIFAR10 dataset with $\alpha =0.1$, the two-step approaches showcase an average enhancement exceeding 4\%.	These findings serve as compelling evidence of the stability and validity of the ADM framework.
	\item In comparison to the two-step approaches, ADM-OS exhibits superior enhancements across most scenarios. To illustrate, on the Dirichlet shift CIFAR10 dataset with $\alpha =0.1$, ADM-OS showcases an average improvement exceeding 10\%, while on the Tweak-One shift USPS dataset with $\rho =0.9$, ADM-OS displays an average enhancement nearing 50\%. Nonetheless, ADM-OS encounters challenges on the CIFAR100 dataset. This situation is possibly attributed to the dataset's extensive category count and the model's diminished accuracy, impeding the effective learning of importance weights. Despite this, ADM-OS maintains commendable performance with datasets featuring a modest category count.
	\item The ADM framework exhibits a propensity for achieving greater enhancements in scenarios characterized by substantial label shift compared to those with minor label shift. For instance, on the Tweak-One shift Fasion MNIST dataset with $\rho =0.9$, the ADM framework attains an average improvement exceeding 3\%. In contrast, on the Tweak-One shift Fasion MNIST dataset with $\rho =0.3$, the ADM framework yields a modest average increase of less than 1\%. This disparity can be attributed to the inherent challenges associated with learning precise weights in instances of large label shift. Therefore, the integration of target samples into the training process is an effective strategy for mitigating the disparity between the training and test distributions.
	\item In contrast to WW, ADM framework enhances the predictive capabilities, especially in large label shift. This observation stands as robust evidence supporting the effectiveness of the aligned distribution mixture. Furthermore, when compared with the semi-supervised approach CSSL, the ADM framework outperforms, underscoring the superior performance achieved through distribution alignment strategies.
	\item It is evident that the two-step ADM approaches consistently enhance the efficacy of existing methods and the performance of ADM-OS is contingent upon the alignment between model's output and conditional distribution. Optimal alignment leads to substantial performance enhancements, whereas misalignment may result in performance degradation. Hence, in practical applications, the selection of an appropriate strategy can be guided by evaluating the model's output on the training dataset.
\end{enumerate}
\subsection{Performance Analysis}
\subsubsection{MSE comparison}
In assessing the precision of weight estimation within ADM-OS, a comparative analysis is conducted against existing label shift methods utilizing the MSE metric. The results are depicted in Fig. 3, from which several key observations have been derived.
\begin{enumerate}
	\item In the majority of scenarios, ADM-OS excels in achieving precise weight estimation, showcasing the efficacy of incorporating weight estimation within classifier training processes. Notably, on the CIFAR10 dataset, ADM-OS marginally trails behind partial
	label shift methods. Nevertheless, the accuracy performance of ADM-OS, as indicated in Table \ref{Tab1}, surpasses that of its comparisons. This discrepancy may stem from the fact that the weights in ADM-OS are derived through aligned distribution mixture and have been subject to iterative refinements by the classifier, rendering them better suited for the target data.
	\item On the CIFAR100 dataset, ADM-OS exhibits suboptimal performance. This outcome is likely attributed to the dataset's extensive category count and the model's diminished accuracy, posing challenges for the model to learn accurate weights effectively. On other datasets, the estimated weights derived by ADM-OS demonstrate greater accuracy and reduced variance. In general, ADM-OS demonstrates commendable performance in scenarios featuring small class sizes or good base classifiers.
\end{enumerate}

\begin{figure*}[!htbp] \label{fig7}
	\centering
	\subfigure[MNIST(ADM-MLLS)]{
		\includegraphics[width=0.24\textwidth]{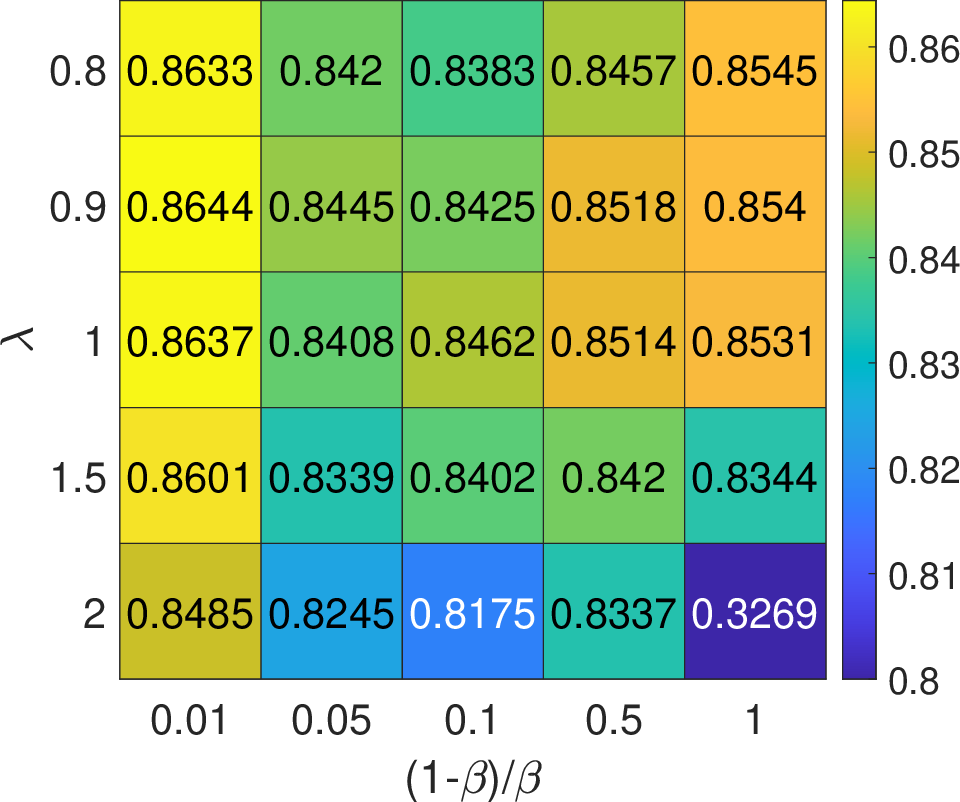}}
	\subfigure[MNIST(ADM-RLLS)]{
		\includegraphics[width=0.24\textwidth]{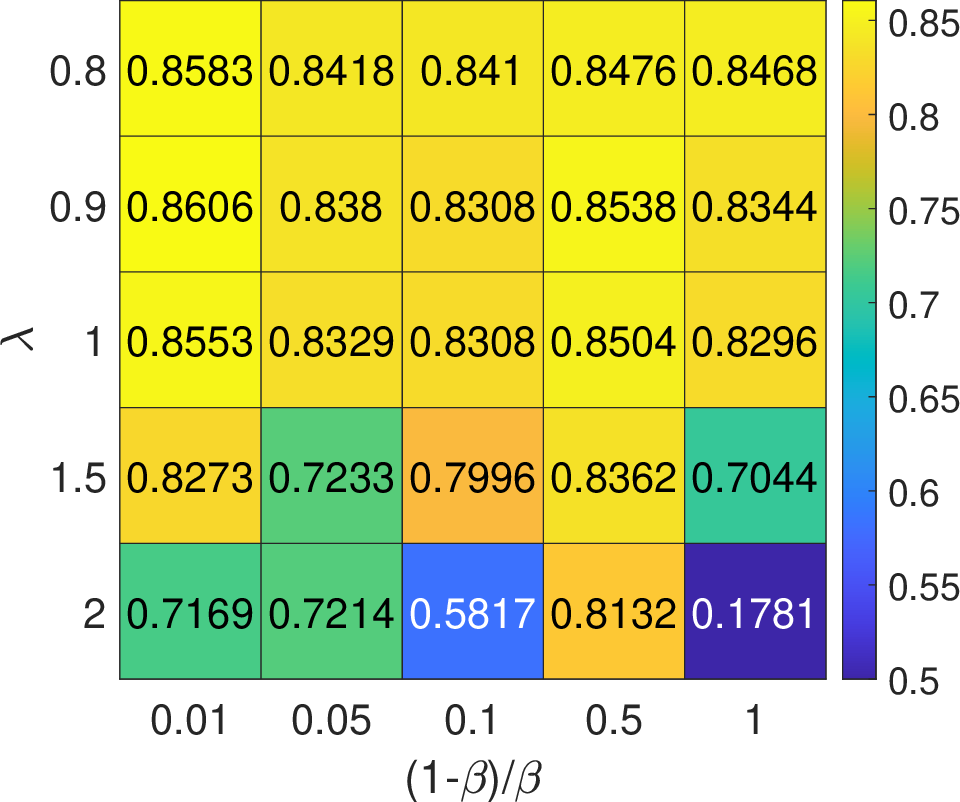}}	
	\subfigure[MNIST(ADM-OS)]{
		\includegraphics[width=0.24\textwidth]{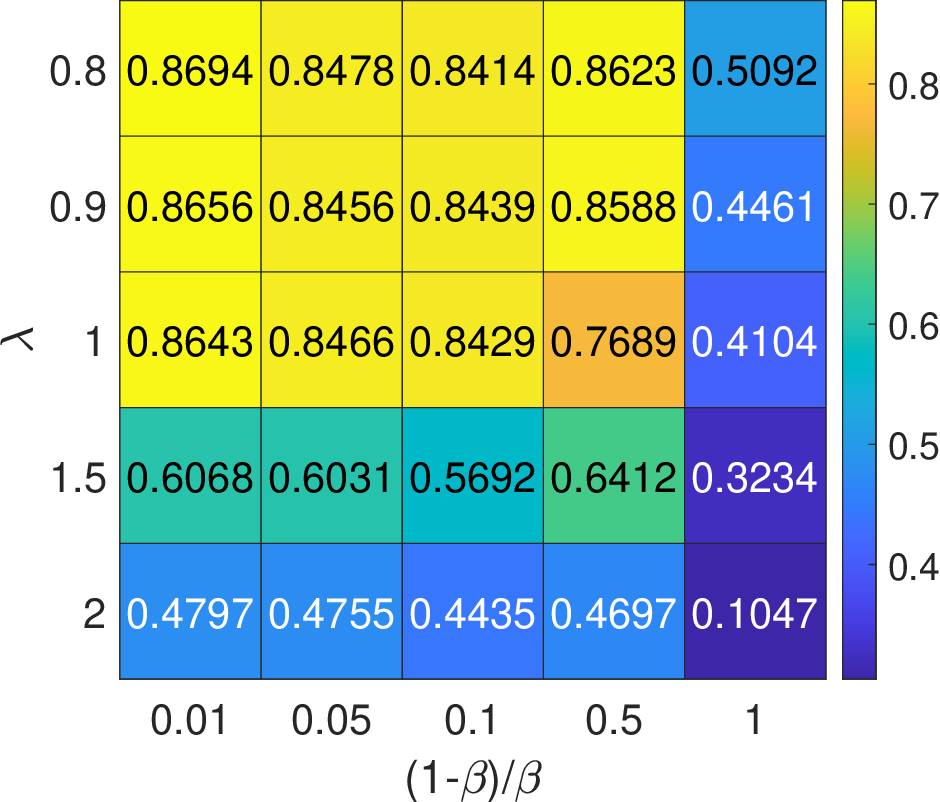}}
	\subfigure[FMNIST(ADM-RLLS)]{
		\includegraphics[width=0.24\textwidth]{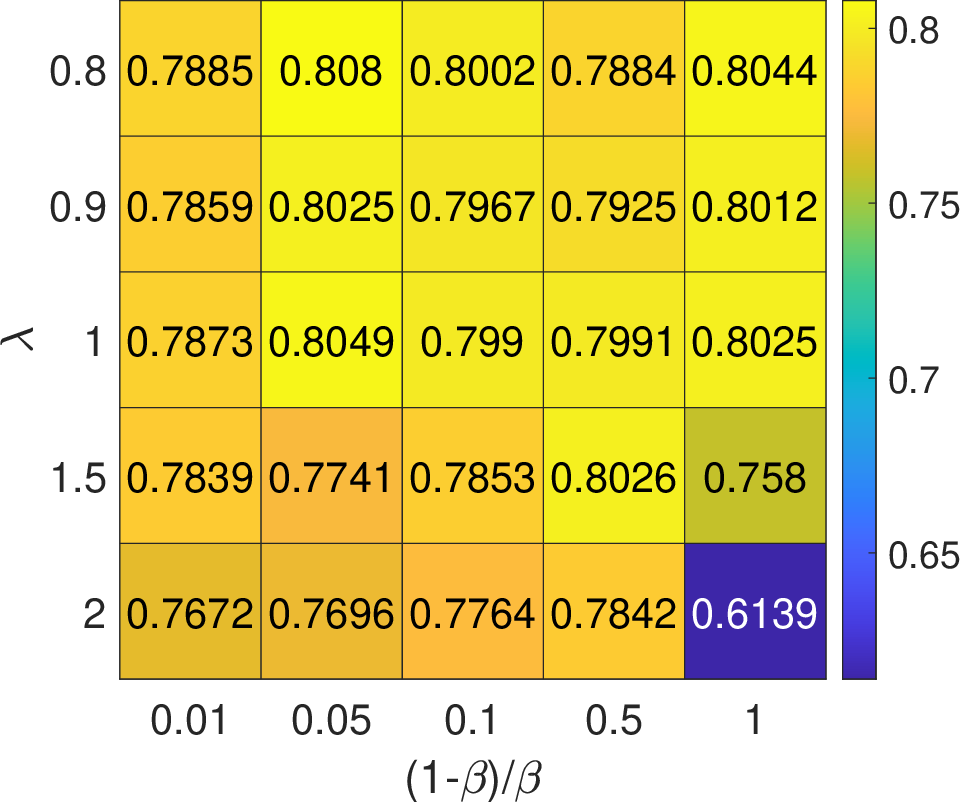}}
	\subfigure[FMNIST(ADM-OS)]{
		\includegraphics[width=0.24\textwidth]{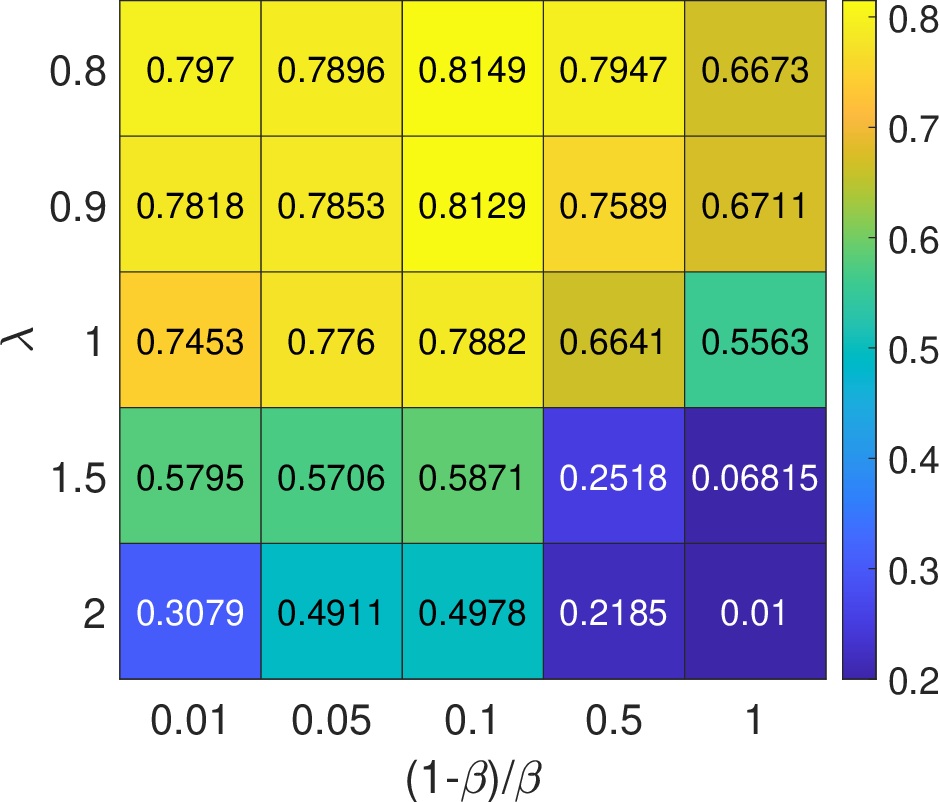}}
	\subfigure[USPS(ADM-MLLS)]{
		\includegraphics[width=0.24\textwidth]{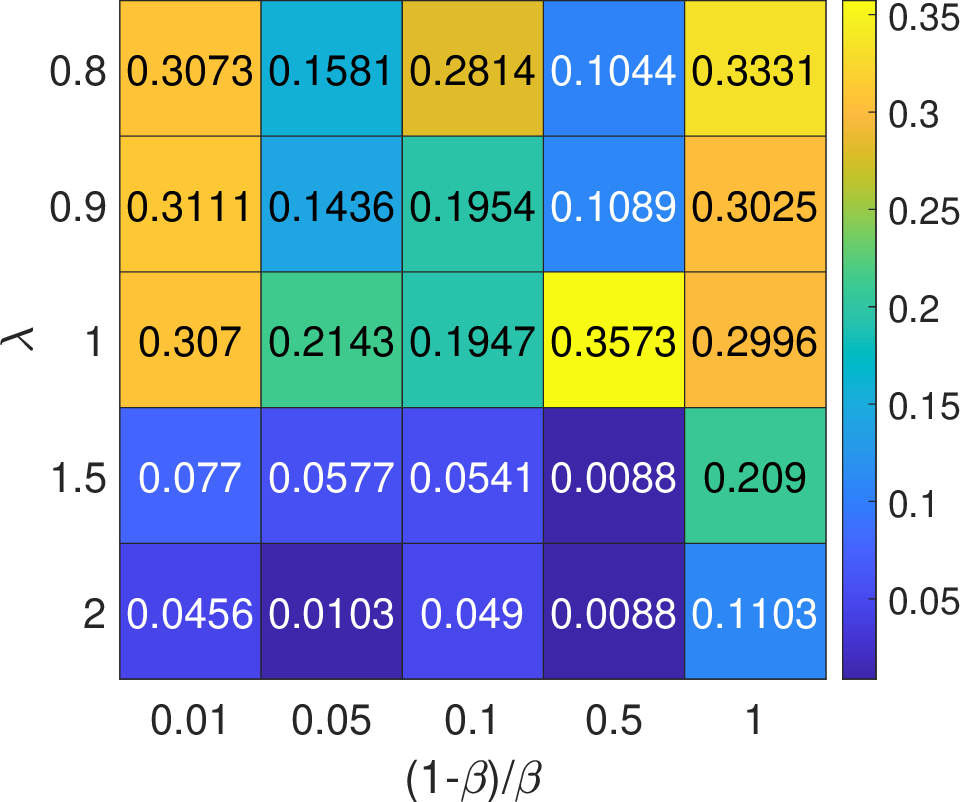}}
	\subfigure[USPS(ADM-RLLS)]{
		\includegraphics[width=0.24\textwidth]{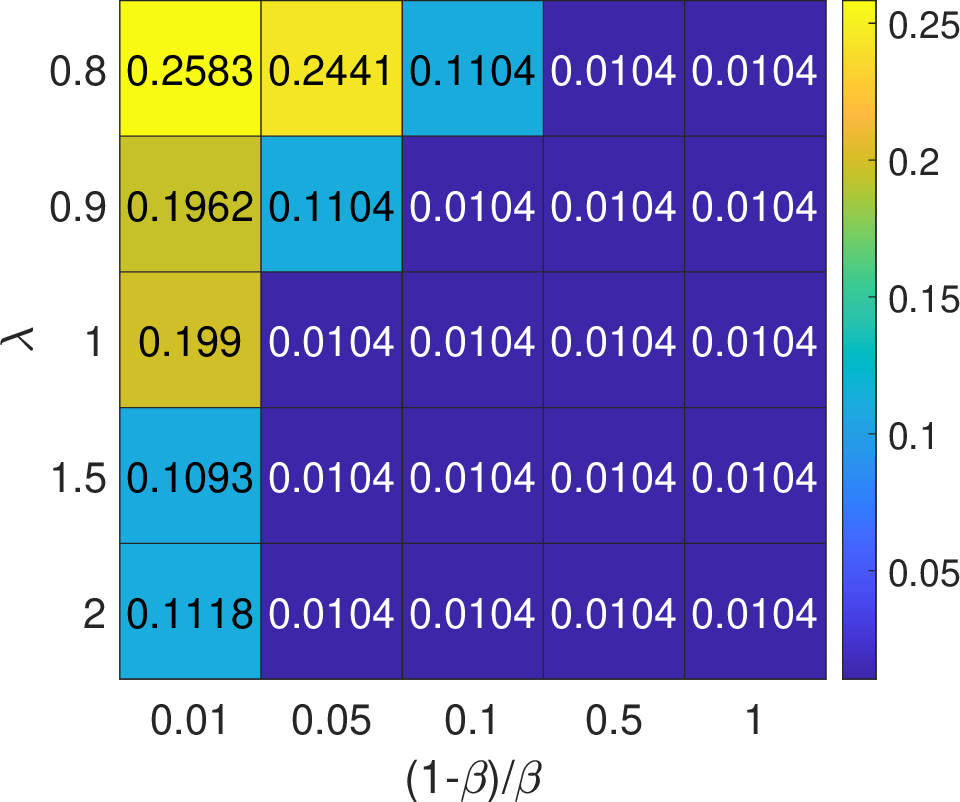}}	
	\subfigure[USPS(ADM-OS)]{
		\includegraphics[width=0.24\textwidth]{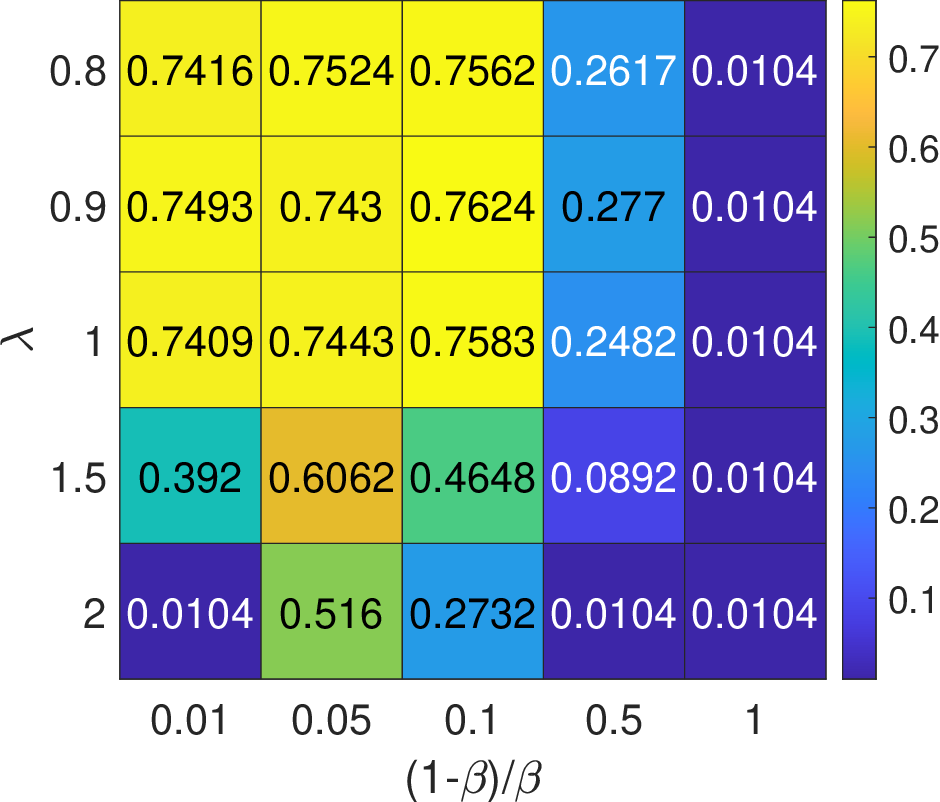}}
	\caption{Sensitivity analysis with different $\lambda$ and $\beta$ for multiple ADM approaches on MNIST, Fasion MNIST and USPS datasets with $\alpha=1$ and  $\rho=0.5$.}
\end{figure*}

\begin{figure*}[!htbp] \label{fig8}
	\centering
	\subfigure[MNIST($\alpha=1$)]{
		\includegraphics[width=0.24\textwidth]{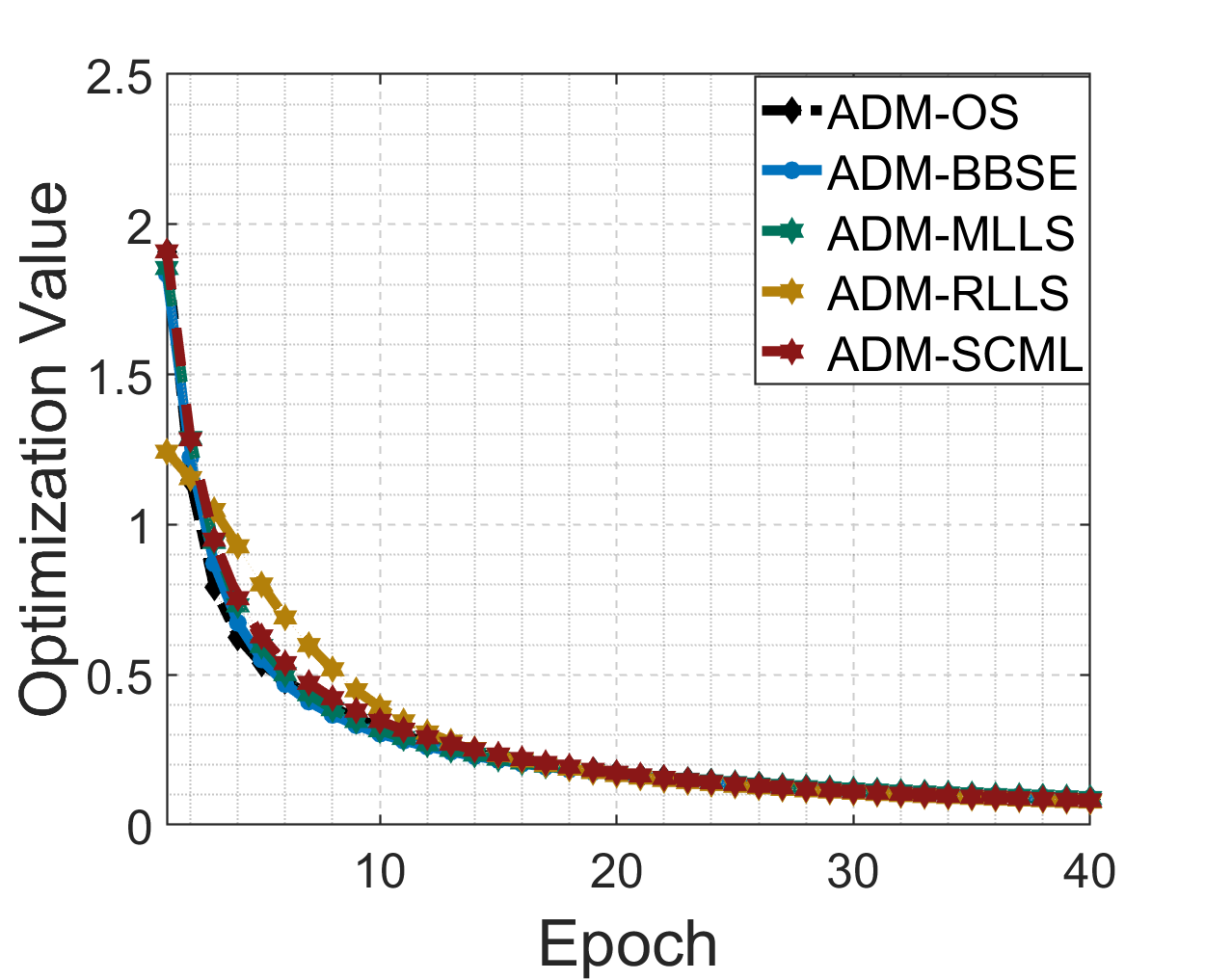}}
	\subfigure[MNIST($\rho=0.5$)]{
		\includegraphics[width=0.24\textwidth]{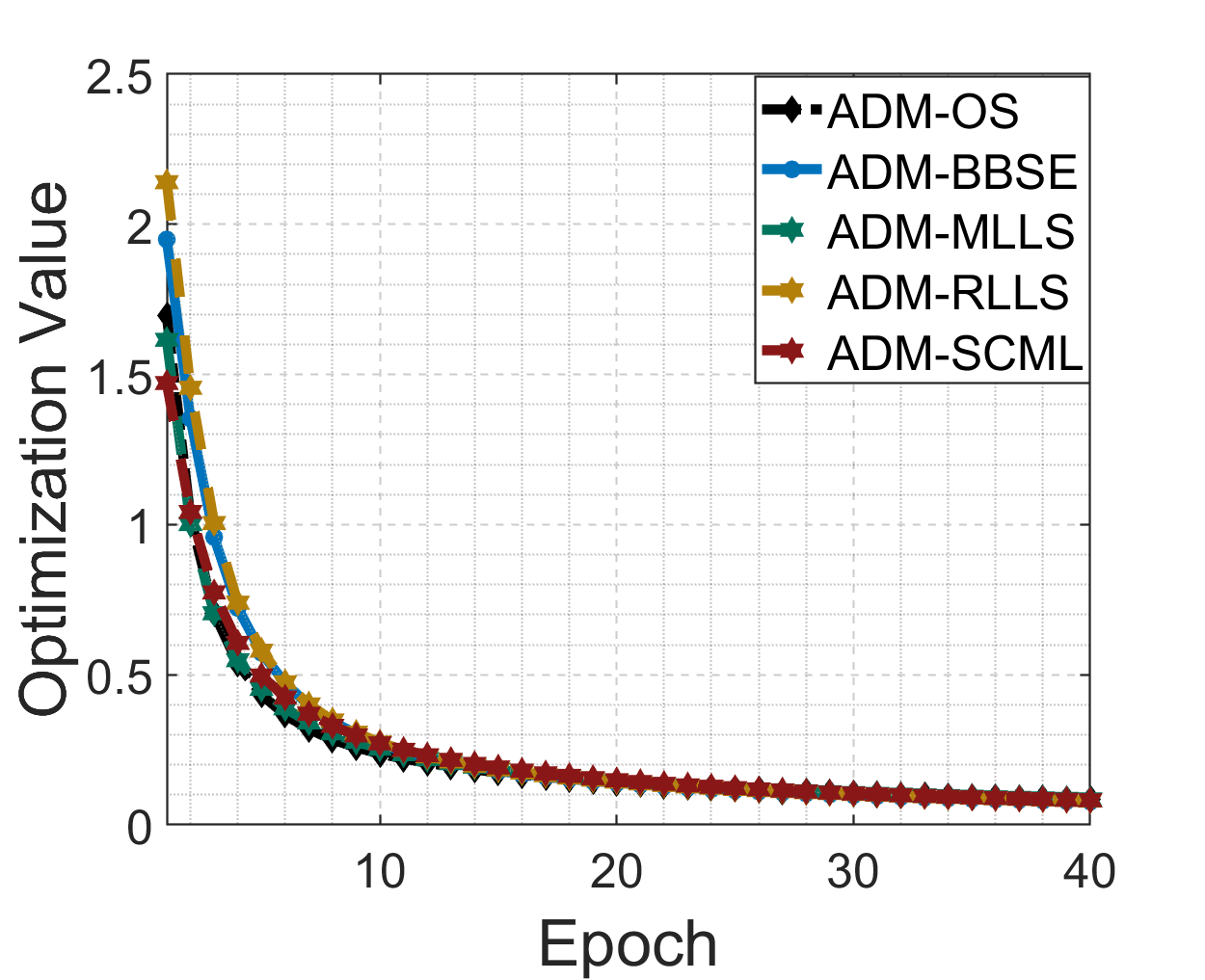}}
	\subfigure[FMNIST($\alpha=1$)]{
		\includegraphics[width=0.24\textwidth]{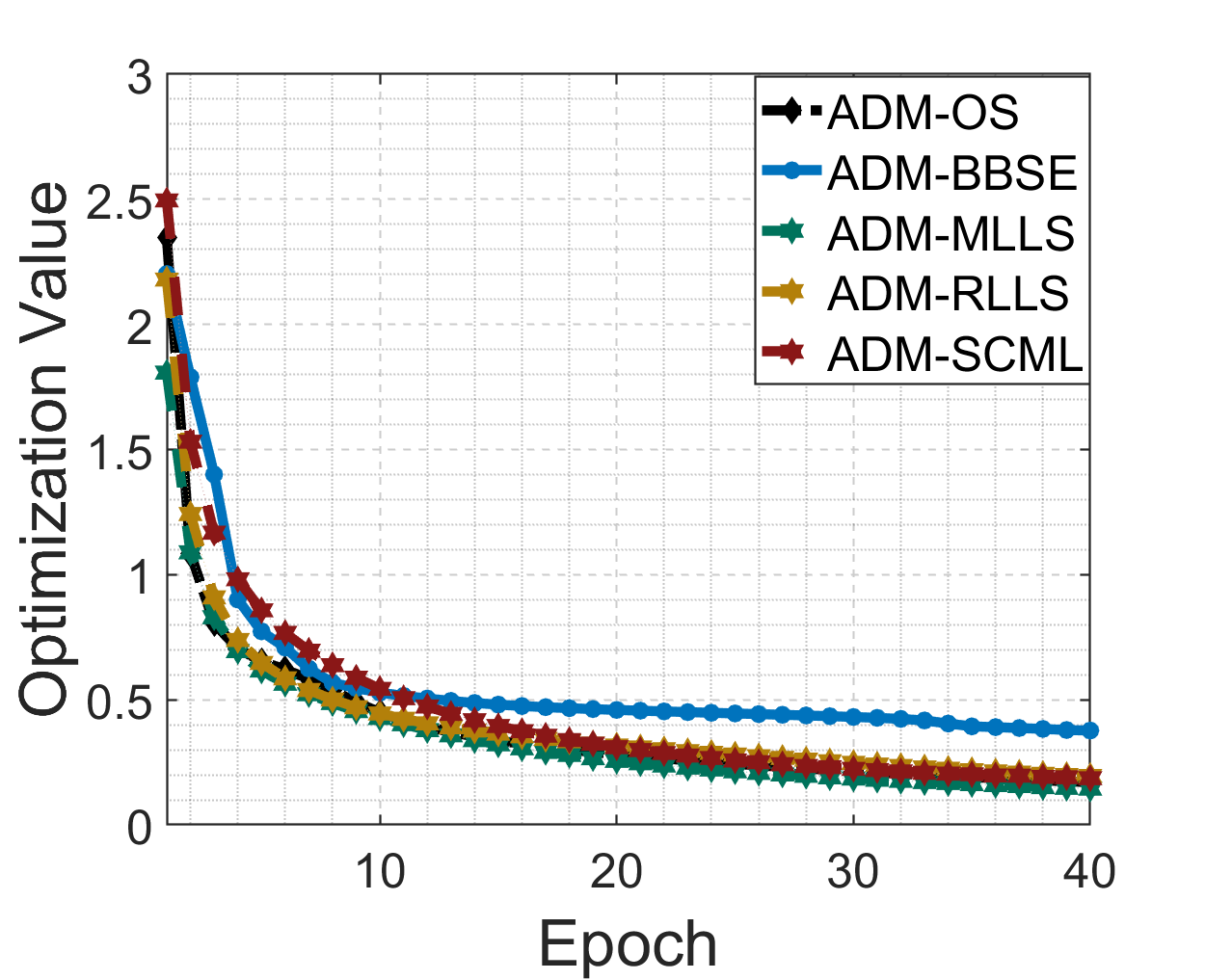}}
	\subfigure[FMNIST($\rho=0.5$)]{
		\includegraphics[width=0.24\textwidth]{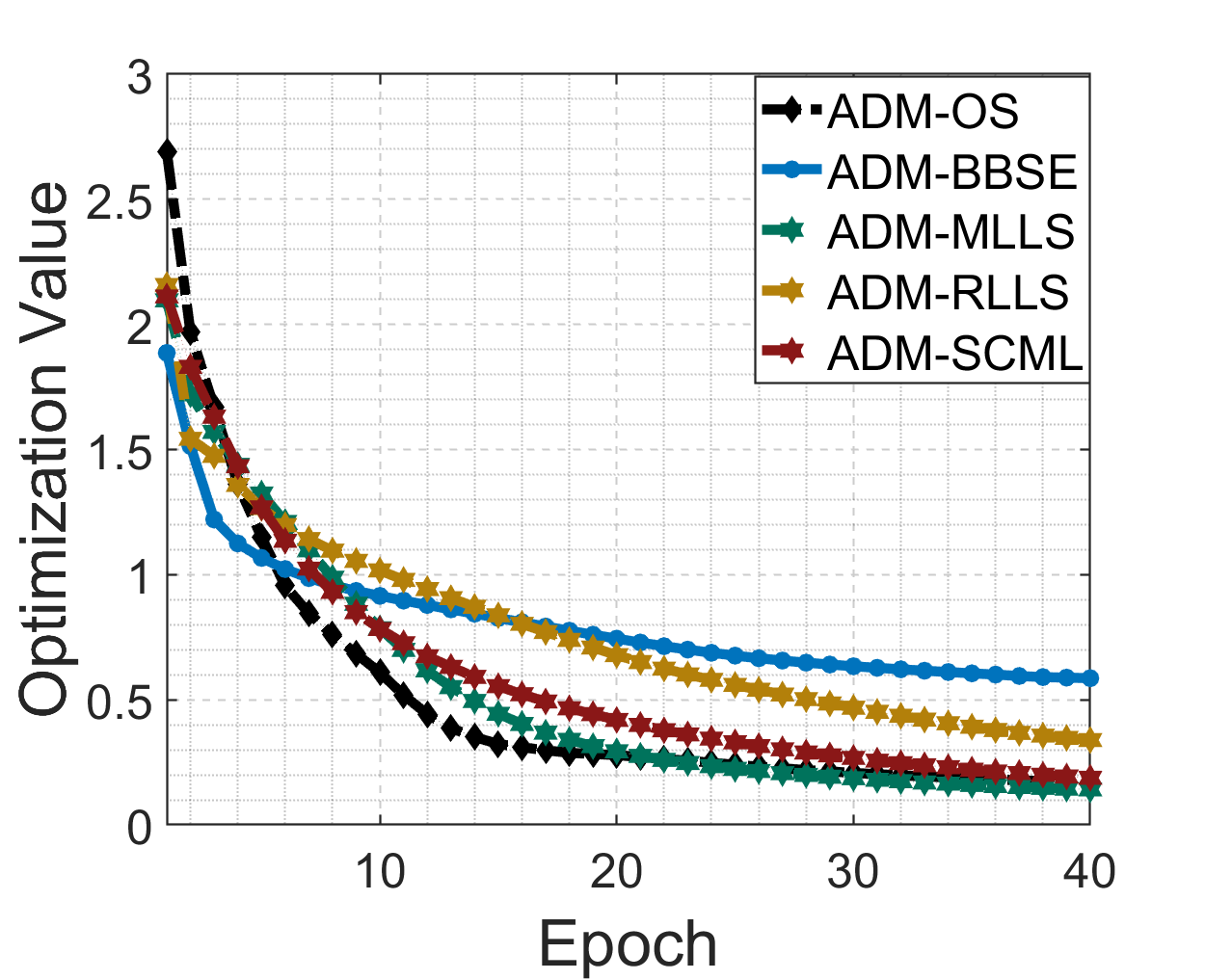}}
	\subfigure[USPS($\alpha=1$)]{
		\includegraphics[width=0.24\textwidth]{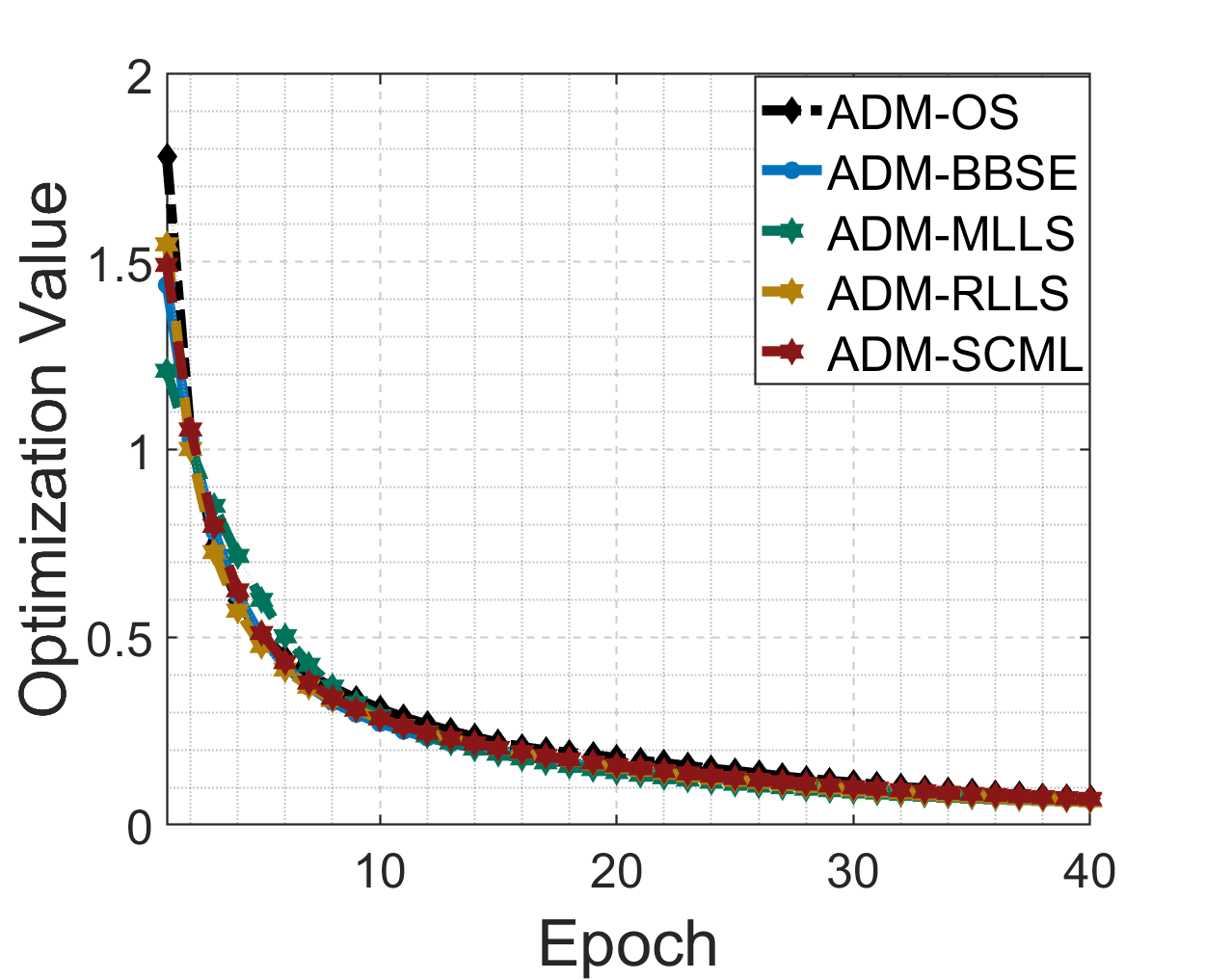}}
	\subfigure[USPS($\rho=0.5$)]{
		\includegraphics[width=0.24\textwidth]{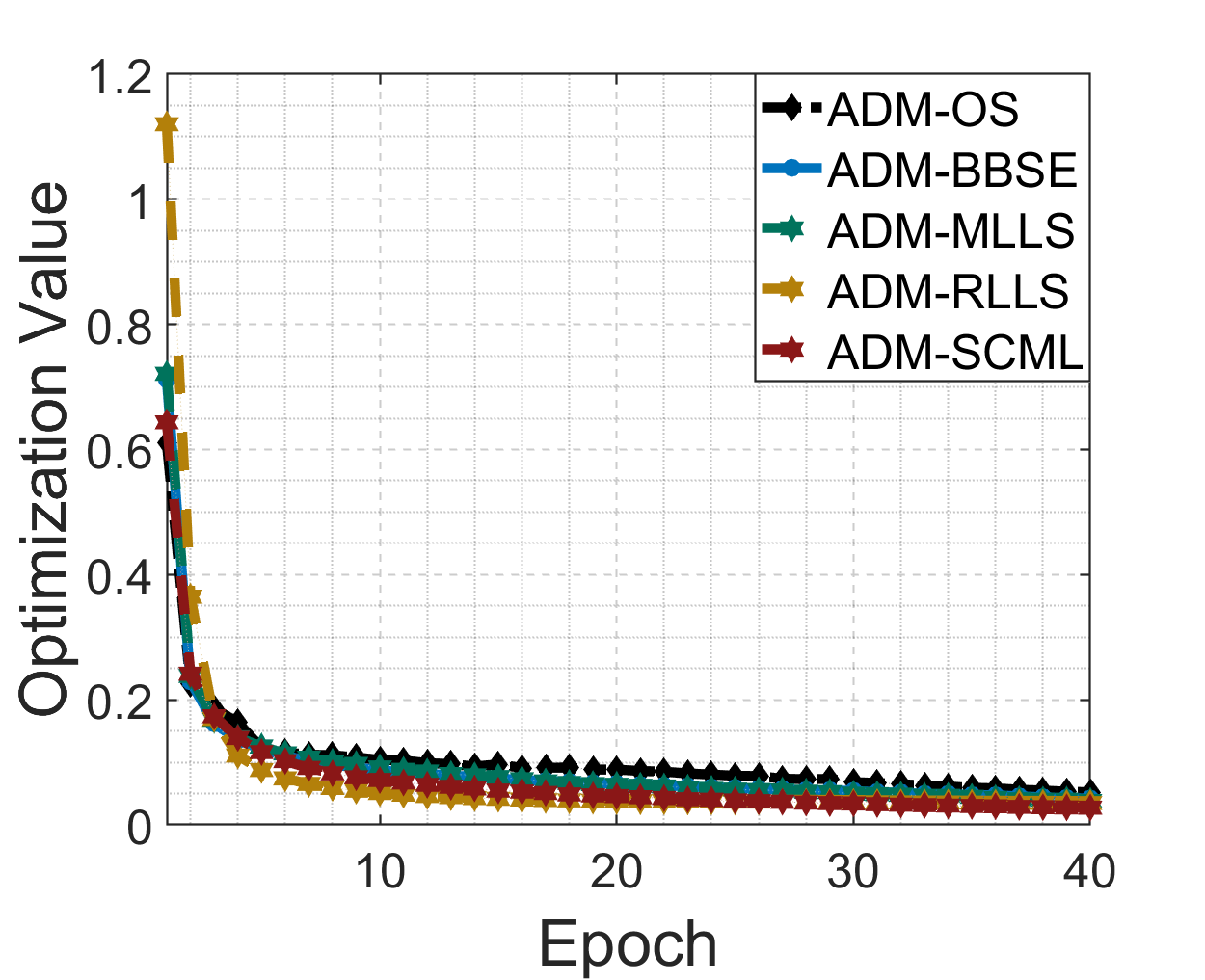}}
	\subfigure[CIFAR10($\alpha=1$)]{
		\includegraphics[width=0.24\textwidth]{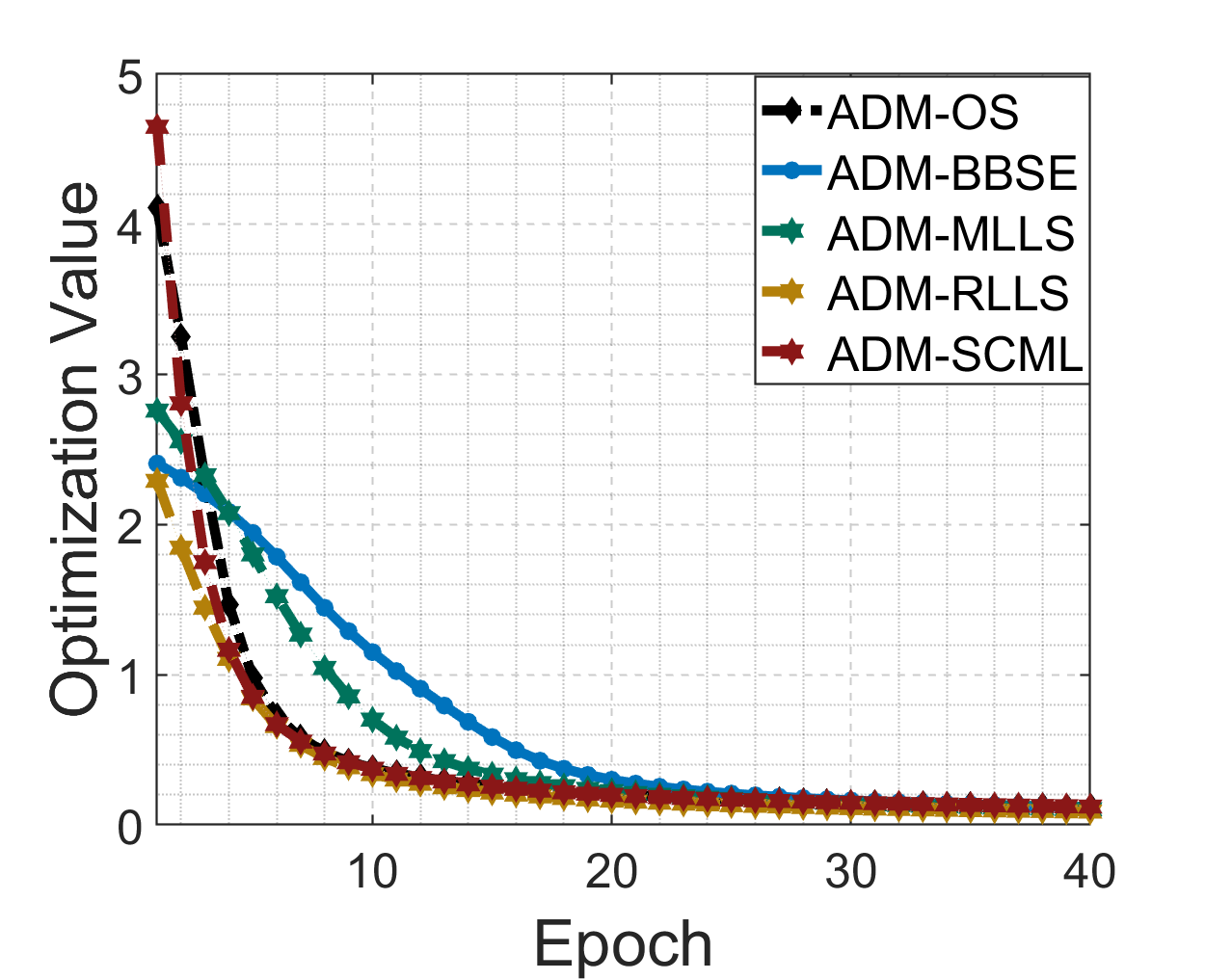}}
	\subfigure[CIFAR10($\rho=0.5$)]{
		\includegraphics[width=0.24\textwidth]{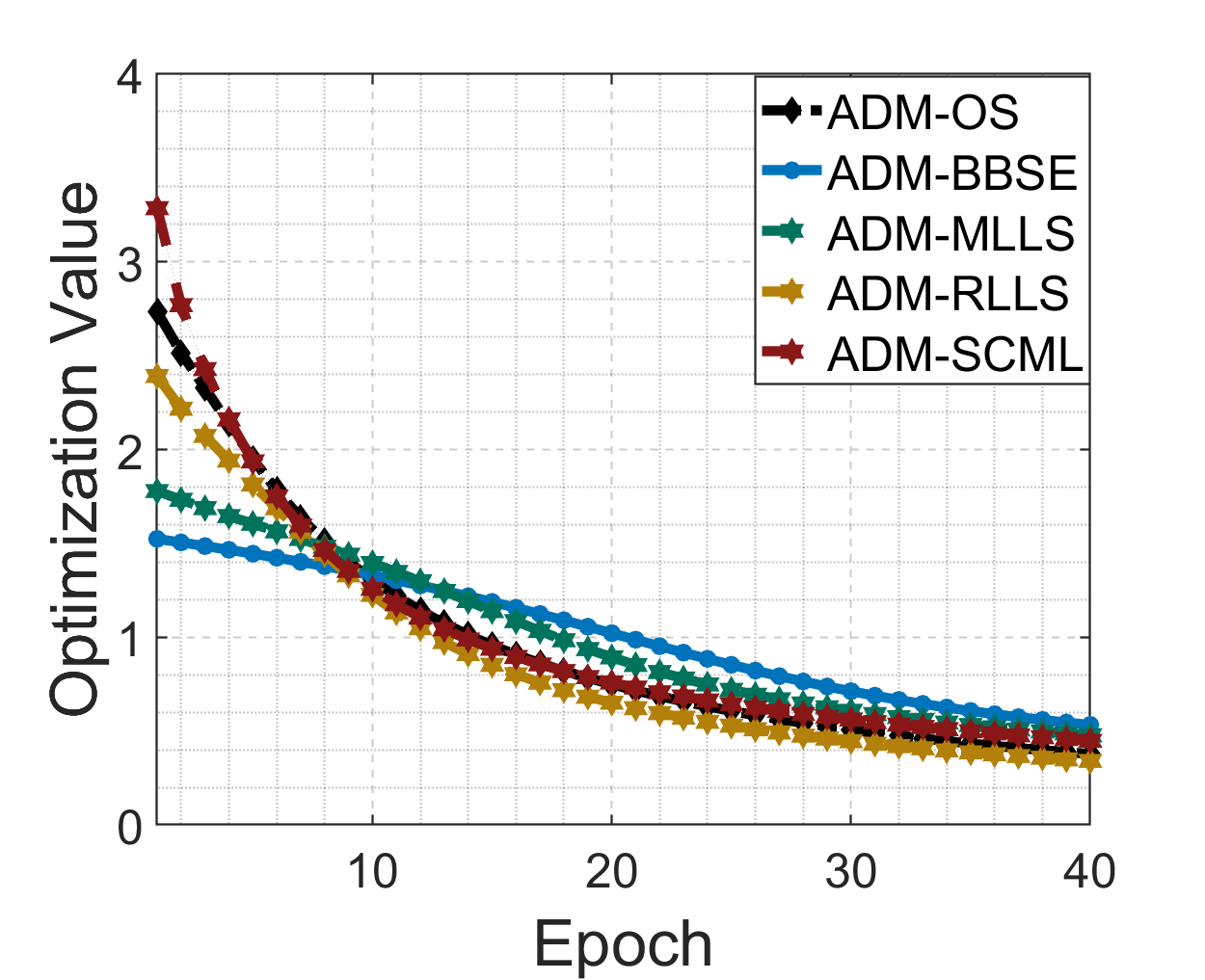}}
	\caption{Convergence behavior of multiple ADM approaches on four datasets with $\rho=0.5$.}
\end{figure*}

\subsubsection{The impact of sample quantity}
To assess the impact of sample quantity from both the source and target domains, the performance of ADM framework and comparison methods is evaluated on the MNIST and Fashion MNIST datasets with $\alpha=1$. Eight distinct scenarios are established, where the quantity of source data~($n$) spans from $[500,1000,2000,4000]$, and the quantity of target data~($m$) ranges from $[2000,8000,14000,20000]$. The results are presented in Fig. 4 and we have the following insights.
\begin{enumerate}
	\item Based on the experimental findings, it becomes apparent that in the majority of instances, the approaches encompassed within  ADM framework outperform comparison methods concerning Acc performance. This substantiates the effectiveness of ADM framework in enhancing predictive abilities. Furthermore, experimental results reveal that the influence of the quantity of source data supersedes that of target data. This reminds us that enhancing the performance of base classifiers can augment the overall model performance. 
	\item For the ADM-OS approach, it can be observed that there is a significant performance improvement when the quantity of source samples is limited. For example, on the Fasion MNIST dataset with $\alpha=1$, ADM-OS achieves an increase of more than 4\%. This demonstrates the effectiveness of incorporating target data into classifier training. Furthermore, ADM-OS exhibits a greater reliance on the quantity of target data compared to alternative methods. For instance, on the Fasion MNIST dataset with $\alpha=1$, the Acc of ADM-OS escalates by more than 5\% as the quantity of target data transitions from $m=2000$ to $m=20000$. This phenomenon can be attributed to the simultaneous impact of the target data quantity on both the weight estimation and classifier training facets within the ADM-OS approach.
\end{enumerate}
\begin{table*}[!htbp]
	\caption{Acc performance(mean(std)) comparison on Dirichlet and Tweak-One shift datasets. Improvements of ADM framework are boldfaced.}
	\label{Tab2}
	\centering
	\setlength{\tabcolsep}{3.5pt}
	\renewcommand\arraystretch{1.2}
	\begin{tabular}{c|c|cccc|cccc}
		\toprule[1.5pt]
		\midrule[0.75pt]
		Dataset &Methods &{$\alpha = 0.1$}&{$\alpha = 0.5$} &{$\alpha = 1$} &{$\alpha = 5$} &{$\rho = 0.3$}&{$\rho = 0.5$} &{$\rho = 0.7$} &{$\rho = 0.9$}	\\
		\midrule[0.75pt]
		\multirow{12}{*}{COVID-19}
		&WW       & 0.5344(.0203) & 0.7226(.0174) & 0.7598(.0061) & 0.7666(.0140) & 0.7805(.0044) & 0.7846(.0042) & 0.7625(.0068) & 0.5476(.0717) \\
		&CSSL     & 0.6713(.0226) & 0.7465(.0066) & 0.7784(.0090) & 0.7867(.0116) & 0.7928(.0083) & 0.7934(.0033) & 0.7710(.0049) & 0.6237(.0680) \\
		&BBSE     & 0.5831(.0864) & 0.7498(.0092) & 0.7736(.0063) & 0.7804(.0089) & 0.7853(.0060) & 0.7874(.0065) & 0.7695(.0055) & 0.6050(.0475) \\
		&ADM-BBSE & 0.6401(.0695) & 0.7669(.0074) & 0.7840(.0068) & 0.7959(.0038) & 0.7979(.0026) & 0.7982(.0048) & 0.7784(.0040) & 0.6453(.0578) \\
		&RLLS     & 0.6482(.0773) & 0.7499(.0092) & 0.7736(.0063) & 0.7804(.0089) & 0.7853(.0060) & 0.7874(.0065) & 0.7695(.0055) & 0.6785(.0559) \\
		&ADM-RLLS & 0.7090(.0638) & 0.7661(.0078) & 0.7830(.0057) & 0.7930(.0074) & 0.7976(.0026) & 0.7974(.0063) & 0.7795(.0038) & 0.7628(.0159) \\
		&MLLS     & 0.6971(.0380) & 0.7522(.0091) & 0.7735(.0062) & 0.7809(.0094) & 0.7857(.0057) & 0.7814(.0096) & 0.7699(.0079) & 0.7030(.0484) \\
		&ADM-MLLS & 0.7576(.0129) & 0.7618(.0039) & 0.7801(.0079) & 0.7916(.0079) & 0.7968(.0021) & 0.7978(.0057) & 0.7816(.0027) & 0.7296(.0408) \\
		&SCML     & 0.7026(.0155) & 0.7535(.0095) & 0.7728(.0054) & 0.7841(.0103) & 0.7852(.0049) & 0.7755(.0156) & 0.7706(.0048) & 0.5578(.0027) \\
		&ADM-SCML & 0.7529(.0157) & 0.7738(.0133) & 0.7827(.0053) & 0.7961(.0076) & 0.7971(.0027) & 0.7898(.0145) & 0.7798(.0022) & 0.5856(.0033) \\
		&ADM-OS   & 0.7340(.0127) & 0.7648(.0116) & 0.7817(.0056) & 0.7955(.0078) & 0.7969(.0025) & 0.7983(.0044) & 0.7792(.0030) & 0.7070(.0411) \\
		&AvgImp   & \textbf{.0572}/\textbf{.0763} & \textbf{.0158}/\textbf{.0135}  & \textbf{.0091}/\textbf{.0083} &  \textbf{.0127}/\textbf{.0141} 
		& \textbf{.0120}/\textbf{.0115} & \textbf{.0129}/\textbf{.0154}  & \textbf{.0099}/\textbf{.0093} &  \textbf{.0448}/\textbf{.0709} \\
		\midrule[0.75pt]
		\bottomrule[1.5pt]
	\end{tabular}
\end{table*}
\subsubsection{Influence of parameters}
Within the ADM framework, the primary parameters are $\lambda$ and $\beta$. This segment delves into a sensitivity analysis of these parameters, encompassing varying values of $\lambda$ and ${(1 - \beta) \mathord{\left/{\vphantom {(1 - \beta) \beta }} \right.\kern-\nulldelimiterspace} \beta }$ across three ADM approaches~(ADM-MLLS, ADM-RLLS and ADM-OS) applied on the MNIST, Fasion MNIST and USPS datasets under $\alpha=1$ and $\rho=0.5$. The selection ranges for $\lambda$ are $\{0.8,0.9,1,1.5,2\}$, while ${(1 - \beta) \mathord{\left/{\vphantom {(1 - \beta) \beta }} \right.\kern-\nulldelimiterspace} \beta }$ is varied within $\{0.01,0.05,0.1,0.5,1\}$. The results are visually presented in Fig. 5 and we have the following observations.
\begin{enumerate}
	\item The performance of our approaches exhibits variability across distinct parameter settings. Notably, within a specific range of adjustments to $\lambda$ and ${(1 - \beta) \mathord{\left/{\vphantom {(1 - \beta) \beta }} \right.\kern-\nulldelimiterspace} \beta }$, the performance of our approaches demonstrates little fluctuations. This characteristic mitigates the complexity associated with parameter tuning endeavors.
	\item In contrast to the two-step approaches ADM-MLLS and ADM-RLLS, ADM-OS displays greater sensitivity to parameter variations. This observation indicates that source samples and calibration loss have a significant impact on model performance, thereby  indirectly highlighting the effectiveness of ADM-OS.
\end{enumerate}

\subsubsection{Convergence Behavior}
To verify the convergence of our ADM procedure outlined in Algorithm 1 and 2, we present the objective function values of all ADM approaches on four datasets. As seen from the results in Fig. 6, we know that the objective function values of ADM approaches eventually converge. Furthermore, it can be observed that when ADM-OS updates weights iteratively, its convergence rate aligns closely with that of the two-stage approaches. This observation shows the efficacy of proposed bi-level optimization strategy. 

\subsection{Application to COVID-19 Screening}
The COVID-19 pandemic has rapidly spread across the globe in a condensed timeframe. While Reverse Transcription Polymerase Chain Reaction~(RT-PCR) tests are widely regarded as the diagnostic gold standard, their widespread implementation faced obstacles during the initial phases of the outbreak due to challenges in sample preparation and quality control~\cite{pami/ZengGHFHRW23}. Serving as a vital adjunctive diagnostic modality, chest Computed Tomography~(CT) scans can unveil bilateral patchy opacities in the lungs of infected individuals, exhibiting heightened sensitivity for COVID-19 detection in comparison to the initial RT-PCR testing of pharyngeal swab specimens~\cite{ChowdhuryRKMKMI20}. This insight provides crucial evidence to inform prompt clinical decision-making during the ongoing pandemic. Nevertheless, the laborious task of manually annotating and delineating lesions for diagnostic interpretation poses a formidable hurdle for radiologists, especially amid the surge in infection rates and suspected cases, compounded by the diminished sensitivity of visual inspection in detecting subtle lesions during the early infection stages. Therefore,  there exists an exigent need to develop an efficient screening model leveraging a limited amount of labeled data from the initial phases and a substantial volume of unlabeled data during the outbreak period.
\begin{figure}[!htbp] \label{fig9}
	\centering
	\subfigure[COVID($\alpha=0.5$)]{
		\includegraphics[width=0.23\textwidth]{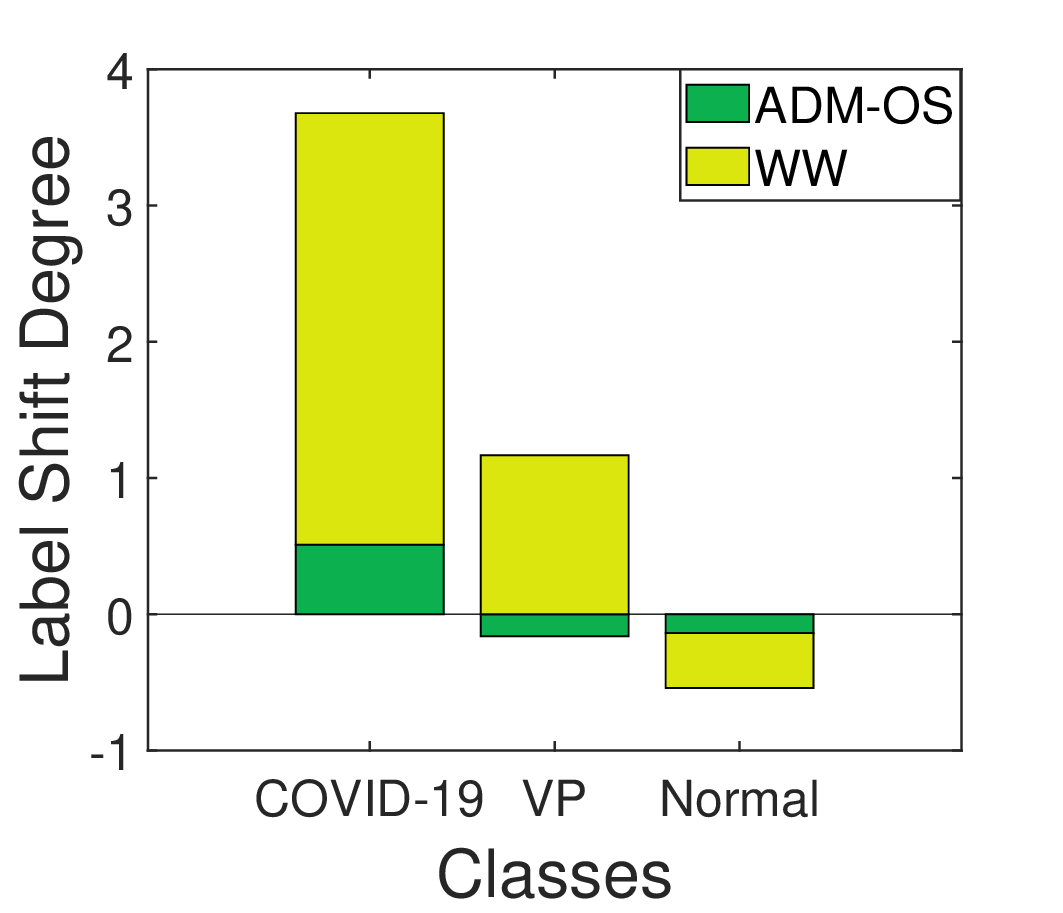}}
	\subfigure[COVID($\alpha=1$)]{
		\includegraphics[width=0.23\textwidth]{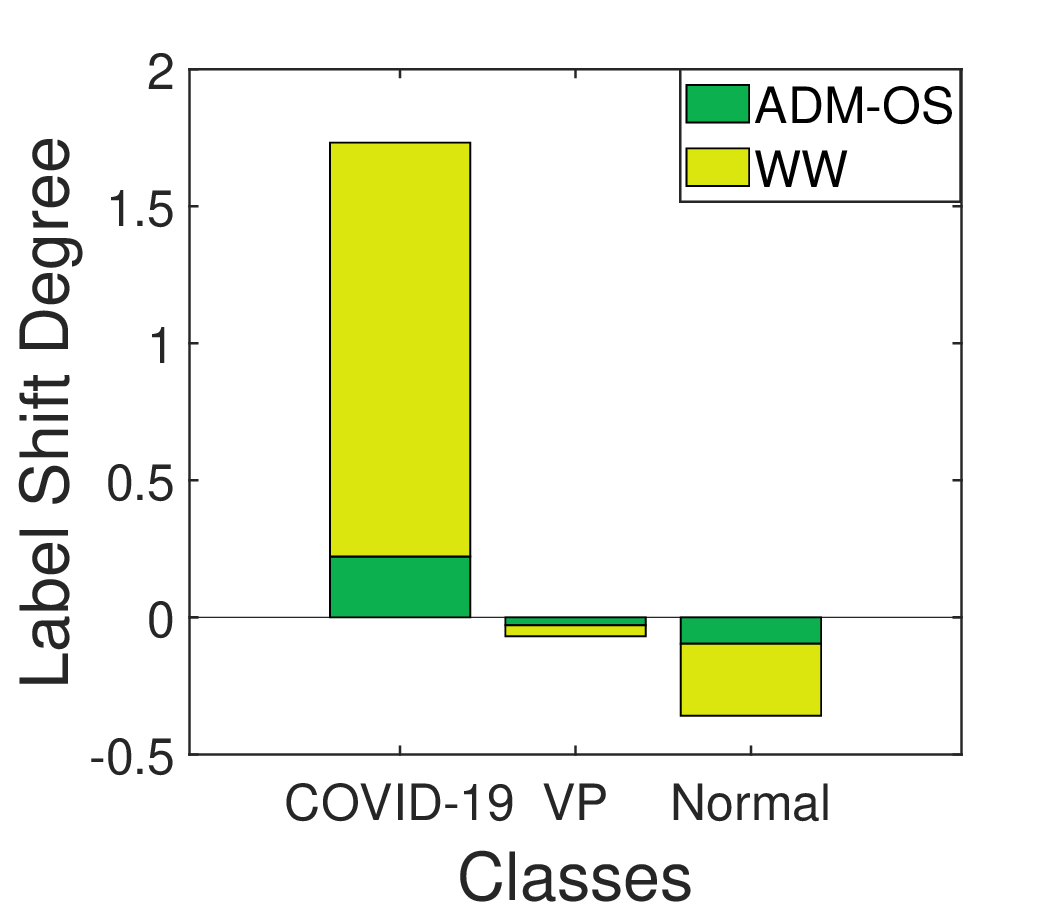}}
	\subfigure[COVID($\rho=0.7$)]{
		\includegraphics[width=0.23\textwidth]{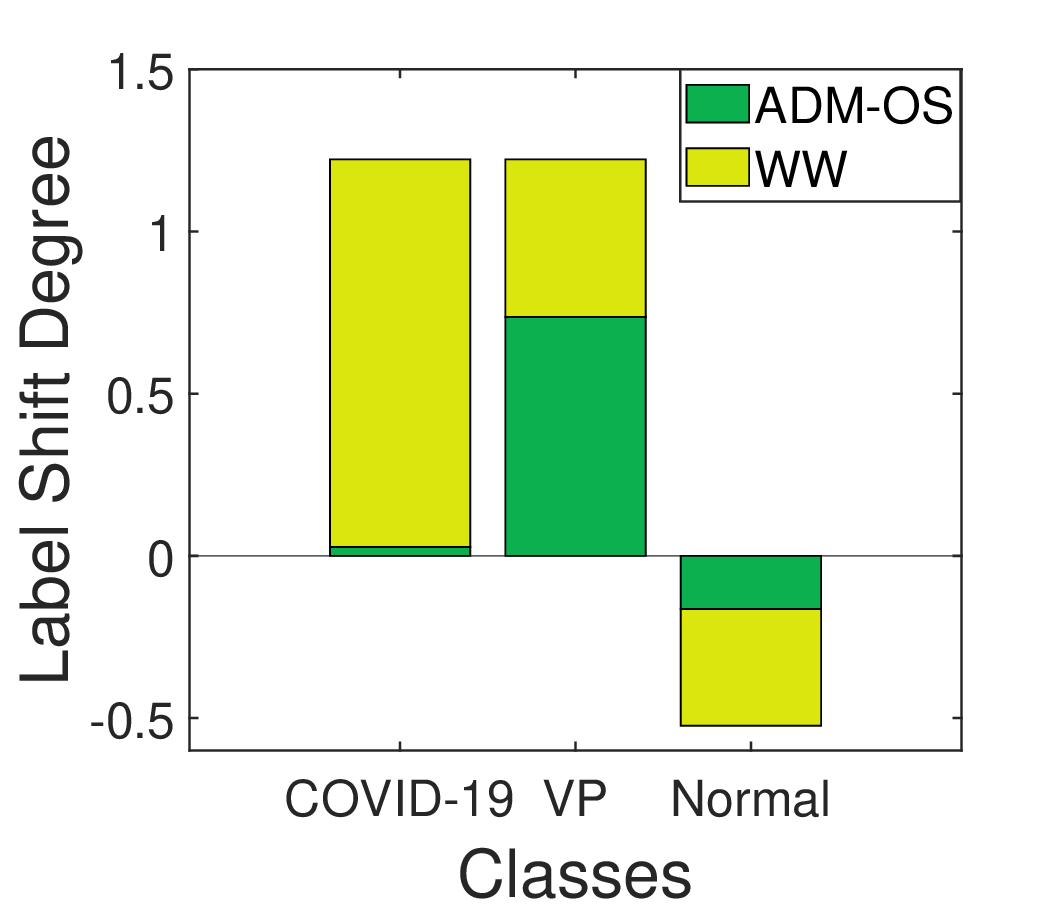}}
	\subfigure[COVID($\rho=0.9$)]{
		\includegraphics[width=0.23\textwidth]{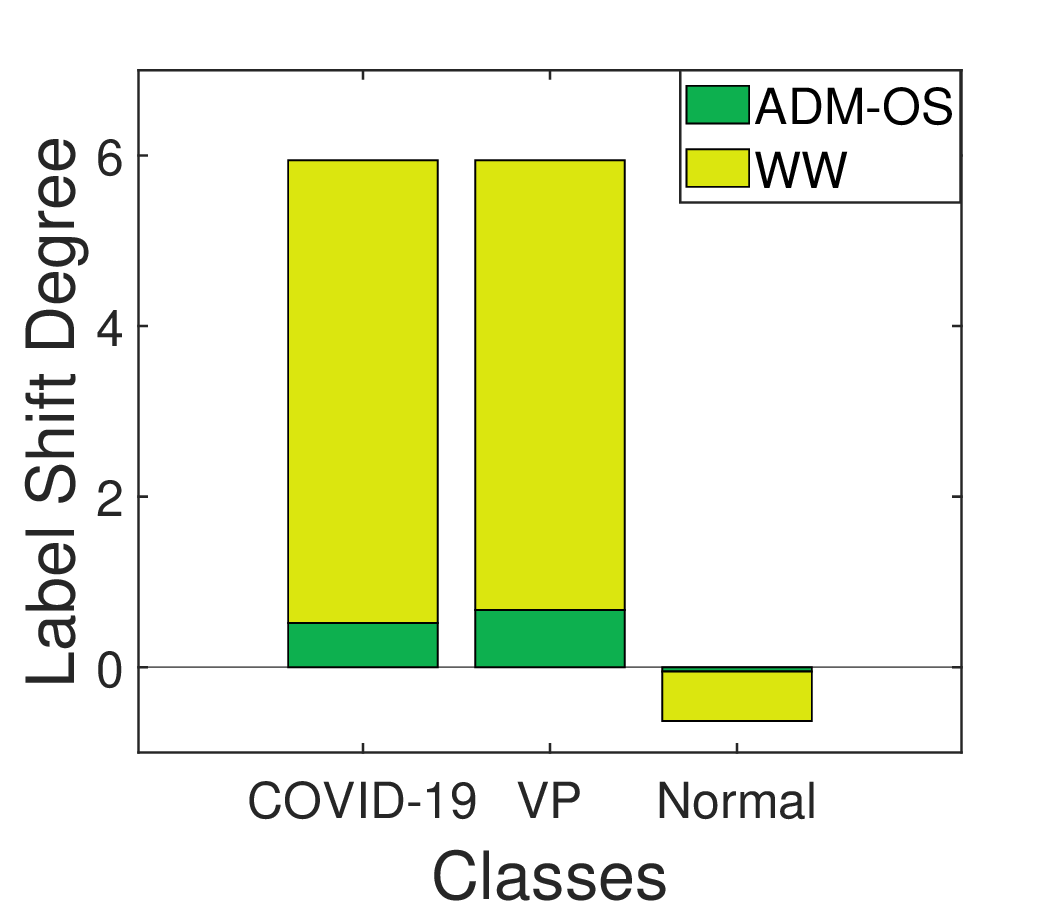}}
	\caption{The label distribution degree of ADM-OS on COVID dataset.}
\end{figure}

In order to evaluate the performance of ADM framework, we utilize one public COVID dataset\footnote{https://www.heywhale.com/mw/dataset/6027caee891f960015c863d7}, which contains a mixture of 1200 COVID-19, 1341 viral pneumonia~(VP), and 1345 normal chest X-ray images. Specifically, the dataset is partitioned into two subsets: 500 labeled samples~(including a small number of COVID-19 images) and 2000 unlabeled samples~(comprising a higher quantity of COVID-19 images). The proportion of COVID-19 in labeled images is determined by the shift parameters $\alpha\in[0.1,0.5,1,5]$ and $\rho\in[0.3,0.5,0.7,0.9]$. We randomly sample 10 times for each distribution parameter and employ a two-layer fully connected neural network as the base classifier. The experimental results are displayed in Table \ref{Tab2} and we obtain that the two-step and one-step ADM approaches have improved performance compared to the existing label shift methods in all shift cases. Notably, on the COVID dataset with $\rho =0.9$ and $\alpha =0.1$, the two-step ADM approaches demonstrate an average increase of more than 4\%, while ADM-OS achieves an average improvement exceeding 7\%. These results offer compelling validation of the stability and efficacy of ADM framework.

Furthermore, in order to validate the precision of weight estimation, a comparative analysis is conducted between ADM-OS and the non-weighted WW method. The label shift degree for ADM-OS and WW is characterized by $w^*-w_{ADM-OS}$ and $w^*-\mathbf{1}$ to elucidate  the weight bias. The outcomes, depicted in Fig. 7, manifest that the estimated weights yielded by ADM-OS closely align with the true weights, indicating the effectiveness of ADM-OS approach.

\section{Conclusion}
This paper is centered on leveraging target samples to enhance the performance of existing label shift methods. The concept of aligned distribution mixture is introduced as a theoretical mechanism to counteract the inherent theoretical deviation of direct distribution mixture within label shift scenario. Building upon this foundation, we propose a theory-inspired label shift framework with calibration loss, encompassing both two-step and one-step approaches. The experimental results highlight the superiority of ADM framework, displaying its effectiveness in addressing label shifts of varying magnitudes. We believe that our work represents a significant advancement in improving the efficacy of label shift algorithms in practical applications, owing to its remarkable performance. In future work, our focus will be on exploring strategies to optimize the performance of ADM-OS in scenarios involving a substantial number of classes.

\bibliographystyle{IEEEtran}
\bibliography{tsls}
\begin{IEEEbiography}[{\includegraphics[width=1in,height=1.25in,clip,keepaspectratio]{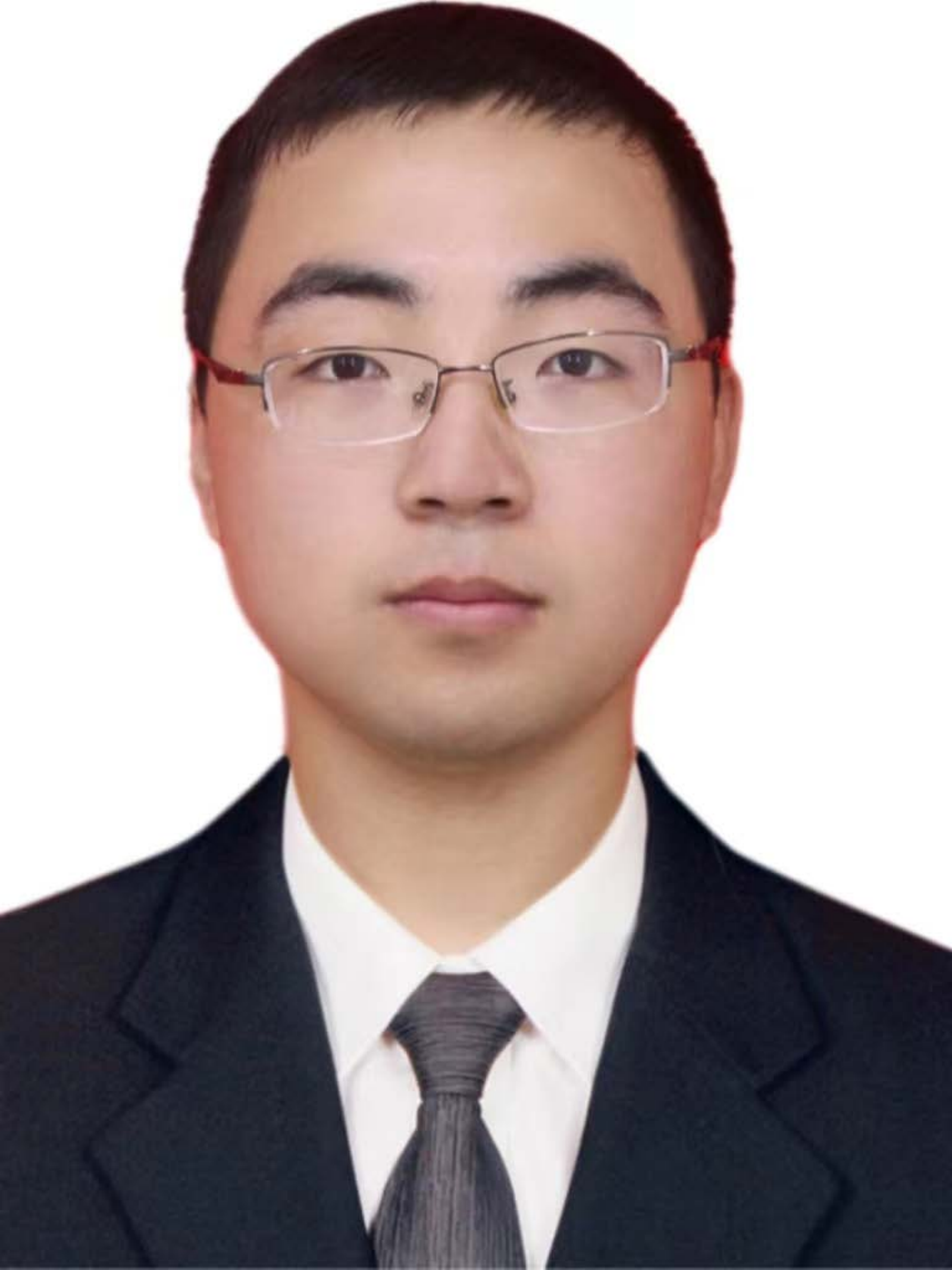}}] {Ruidong Fan}~{is a Ph.D. condidate at the National University of Defense Technology, Changsha, China. He received the B.S. degree from Lanzhou University in 2018 and the M.S. degree from the National University of Defense Technology in 2020. His research interests include data mining, optimization and transfer learning.}
\end{IEEEbiography}

\begin{IEEEbiography}[{\includegraphics[width=1in,height=1.25in,clip,keepaspectratio]{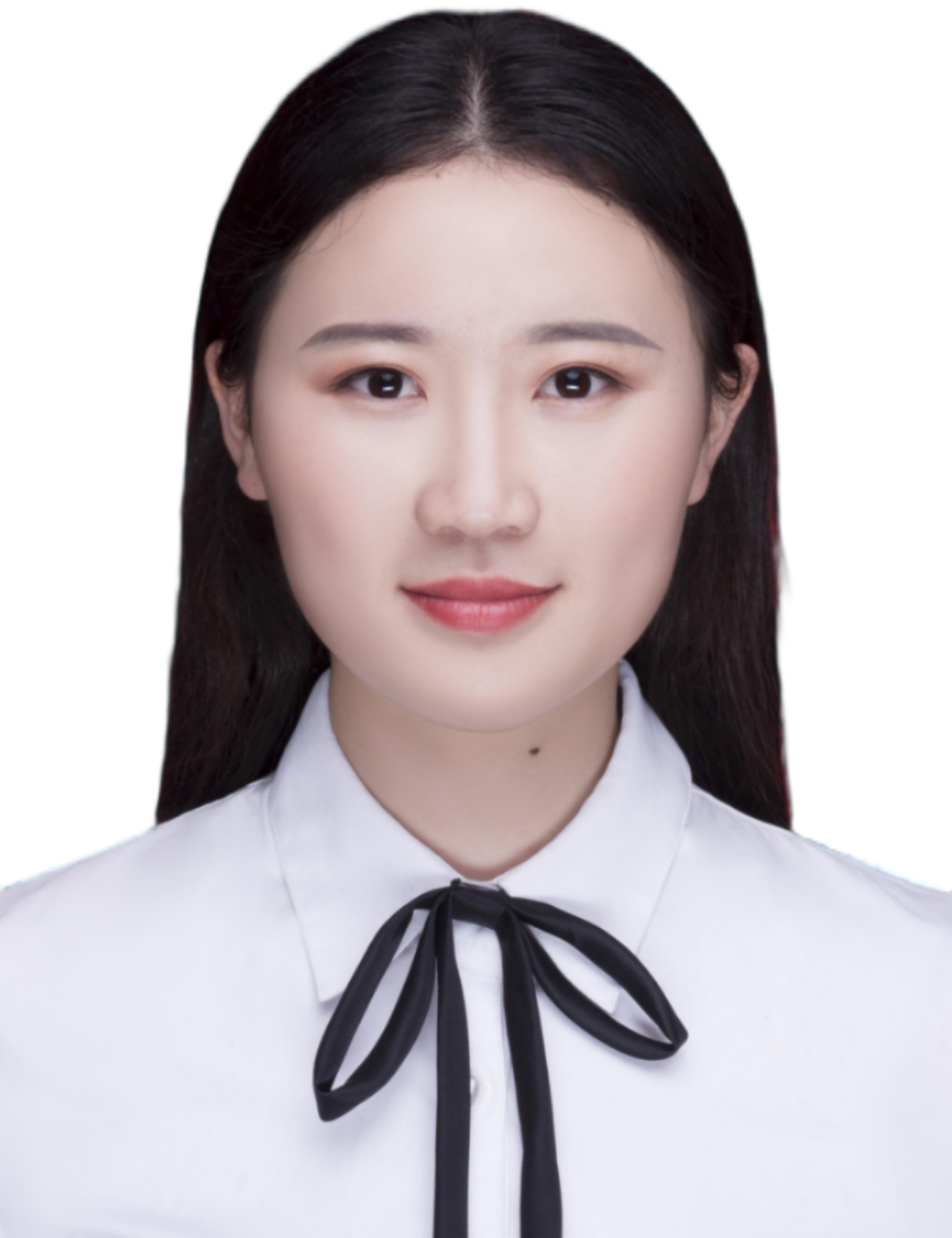}}] {Xiao Ouyang}~{is a Ph.D. condidate at the National University of Defense Technology, Changsha, China. She received the B.S. degree from Wuhan University of Technology in 2020 and the M.S. degree from the National University of Defense Technology in 2022. Her current research interests include data mining and streaming view learning.}
\end{IEEEbiography}

\begin{IEEEbiography}[{\includegraphics[width=1in,height=1.25in,clip,keepaspectratio]{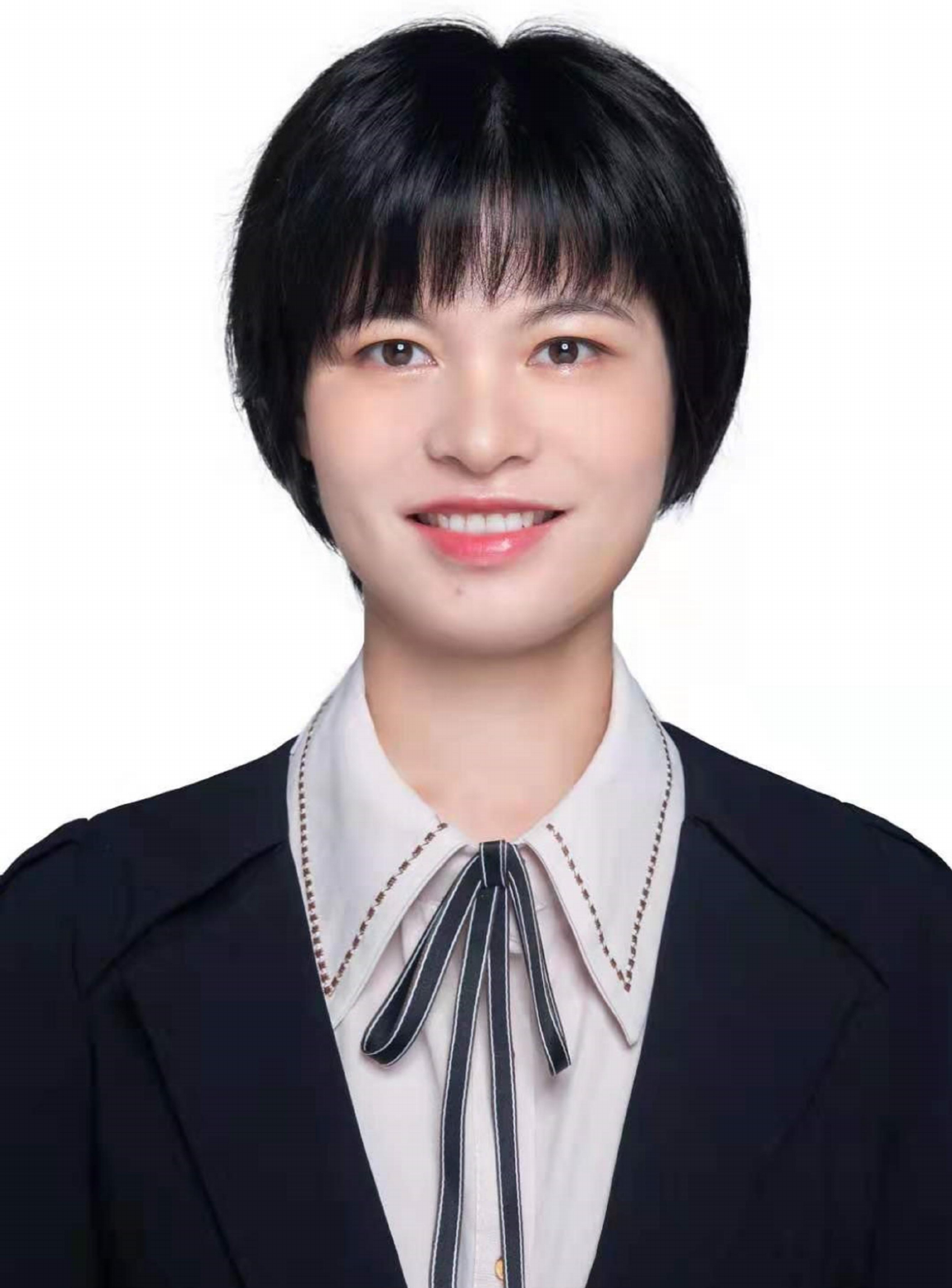}}] {Hong Tao}~{received the Ph.D. degree from the National University of Defense Technology, Chang sha, China, in 2019. She is currently an Associate Professor with the College of Science of the same university. Her research interests include machine learning, system science, and data mining.}
\end{IEEEbiography}

\begin{IEEEbiography}[{\includegraphics[width=1in,height=1.25in,clip,keepaspectratio]{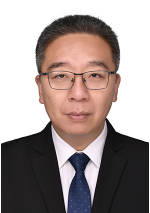}}] {Yuhua Qian}~{received the M.S. and Ph.D. degrees in computers with applications from Shanxi University, Taiyuan, China, in 2005 and 2011, respectively. He is currently a Professor at the Key Laboratory of Computational Intelligence and Chinese Information Processing, Ministry of Education, Shanxi University. He is best known for multi-granulation rough sets in learning from categorical data and granular computing. He is involved in research on machine learning, pattern recognition, feature selection, granular computing, and artificial intelligence. He has authored over 100 articles on these topics in international journals. He served on the Editorial Board of the International Journal of Knowledge-Based Organizations and Artificial Intelligence Research. He has served as the Program Chair or Special Issue Chair of the Conference on Rough Sets and Knowledge Technology, the Joint Rough Set Symposium, and the Conference on Industrial Instrumentation and Control, and a PC Member of many machine learning, and data mining conferences.}
\end{IEEEbiography}

\begin{IEEEbiography}[{\includegraphics[width=1in,height=1.25in,clip,keepaspectratio]{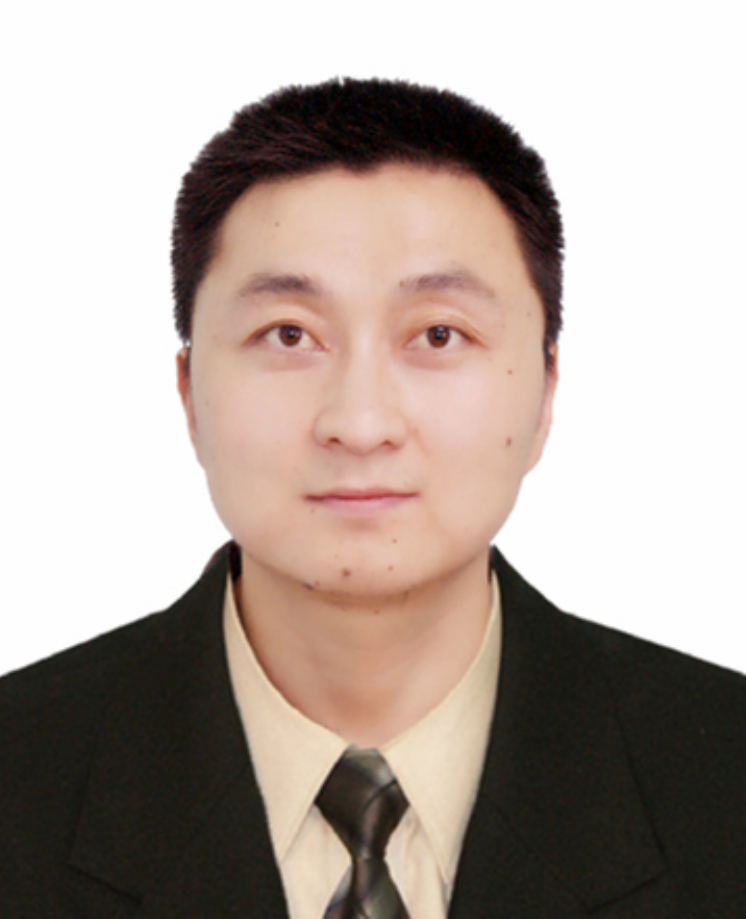}}] {Chenping Hou}~{received the B.S. and Ph.D. degrees in applied mathematics from the National University of Defense Technology, Changsha, China, in 2004 and 2009, respectively. He is currently a full Professor at the College of Science, National University of Defense Technology. He has authored more than 100 papers in journals and conferences, such as the IEEE TPAMI, IEEE TKDE, IEEE TNNLS/TNN, IEEE TCYB, IEEE TIP, PR, IJCAI, AAAI, etc. He has been a Program Committee member of several conferences including IJCAI, AAAI, etc. His current research interests include pattern recognition, machine learning, data mining, and computer vision.}
\end{IEEEbiography}

\end{document}